\documentclass{article}

\PassOptionsToPackage{table}{xcolor}

\usepackage{microtype}
\usepackage{graphicx}
\usepackage{subcaption}
\usepackage{booktabs} 

\usepackage{hyperref}

\usepackage[preprint]{icml2026}

\usepackage{amsmath}
\usepackage{amssymb}
\usepackage{mathtools}
\usepackage{amsthm}

\usepackage[capitalize,noabbrev]{cleveref}

\usepackage{etoc}
\usepackage{multicol}
\theoremstyle{plain}
\newtheorem{theorem}{Theorem}[section]
\newtheorem{proposition}[theorem]{Proposition}
\newtheorem{lemma}[theorem]{Lemma}

\theoremstyle{definition}

\theoremstyle{remark}
\newtheorem{remark}[theorem]{Remark}

\usepackage{tikz}
\usepackage{float}
\usepackage{rotating}
\usepackage{afterpage}
\usepackage{stfloats}
\usepackage{listings}
\usepackage[T1]{fontenc}
\usepackage{plex-mono} 
\definecolor{codegreen}{rgb}{0,0.6,0}
\definecolor{codegray}{rgb}{0.5,0.5,0.5}
\definecolor{codepurple}{rgb}{0.58,0,0.82}
\definecolor{backcolour}{rgb}{0.96,0.96,0.96}

\usepackage{amsmath,amsfonts,bm}

\def\eqref#1{equation~\ref{#1}}

\def\1{\bm{1}}

\def\vzero{{\bm{0}}}

\def\va{{\bm{a}}}
\def\vb{{\bm{b}}}

\def\ve{{\bm{e}}}
\def\vf{{\bm{f}}}

\def\vo{{\bm{o}}}

\def\vq{{\bm{q}}}

\def\vs{{\bm{s}}}
\def\vt{{\bm{t}}}

\def\vv{{\bm{v}}}
\def\vw{{\bm{w}}}
\def\vx{{\bm{x}}}

\def\mA{{\bm{A}}}

\def\mK{{\bm{K}}}

\def\mM{{\bm{M}}}

\def\mV{{\bm{V}}}
\def\mW{{\bm{W}}}

\def\mSigma{{\bm{\Sigma}}}

\DeclareMathAlphabet{\mathsfit}{\encodingdefault}{\sfdefault}{m}{sl}
\SetMathAlphabet{\mathsfit}{bold}{\encodingdefault}{\sfdefault}{bx}{n}

\def\gR{{\mathcal{R}}}

\newcommand{\E}{\mathbb{E}}

\newcommand{\R}{\mathbb{R}}

\newcommand{\Var}{\mathrm{Var}}

\newcommand{\Cov}{\mathrm{Cov}}

\DeclareMathOperator{\sign}{sign}

\lstdefinestyle{mystyle}{
    backgroundcolor=\color{backcolour},
    commentstyle=\color{codegreen}\itshape,
    keywordstyle=\color{magenta}\bfseries,
    numberstyle=\tiny\color{codegray},
    stringstyle=\color{codepurple},
    basicstyle=\ttfamily\footnotesize,
    breakatwhitespace=false,
    breaklines=true,
    captionpos=b,
    keepspaces=true,
    numbers=left,
    numbersep=5pt,
    showspaces=false,
    showstringspaces=false,
    showtabs=false,
    tabsize=4,
    frame=lines,
    rulecolor=\color{codegray}
}
\usepackage{multirow}
\usepackage{kotex}
\usepackage{arydshln}

\definecolor{linkaccent}{HTML}{DA7758}
\hypersetup{linkcolor=linkaccent}

\icmltitlerunning{Real-Time Visual Attribution Streaming in Thinking Model}

\begin{document}

\twocolumn[
  \icmltitle{Real-Time Visual Attribution Streaming in Thinking Model}

  \icmlsetsymbol{yonsei}{\textsc{y}}
  \icmlsetsymbol{amazon}{\textsc{a}}

  \begin{icmlauthorlist}
    \icmlauthor{Seil Kang}{yonsei}
    \icmlauthor{Woojung Han}{yonsei}
    \icmlauthor{Junhyeok Kim}{yonsei}
    \icmlauthor{Jinyeong Kim}{yonsei}
    \icmlauthor{Youngeun Kim}{amazon}
    \icmlauthor{Seong Jae Hwang}{yonsei}
  \end{icmlauthorlist}

  \icmlcorrespondingauthor{Seil Kang}{seil@yonsei.ac.kr}

  \icmlkeywords{Machine Learning, ICML}

  \vskip 0.3in
]



\printAffiliationsAndNotice{\textsuperscript{\textsc{y}}Yonsei University. \textsuperscript{\textsc{a}}Amazon.}

\begin{abstract}
We present an amortized framework for real-time \emph{visual attribution streaming} in multimodal thinking models. When these models generate code from a screenshot or solve math problems from images, their long reasoning traces should be grounded in visual evidence. However, verifying this reliance is challenging: faithful causal methods require costly repeated backward passes or perturbations, while raw attention maps offer instant access, they lack causal validity. To resolve this, we introduce an amortized approach that learns to estimate the causal effects of semantic regions directly from the rich signals encoded in attention features. Across five diverse benchmarks and four thinking models, our approach achieves faithfulness comparable to exhaustive causal methods while enabling visual attribution streaming, where users observe grounding evidence as the model reasons, not after. Our results demonstrate that real-time, faithful attribution in multimodal thinking models is achievable through lightweight learning, not brute-force computation.
\end{abstract}
\vspace{-2pt}

\begin{figure}[t!]
\centering
\includegraphics[width=\linewidth]{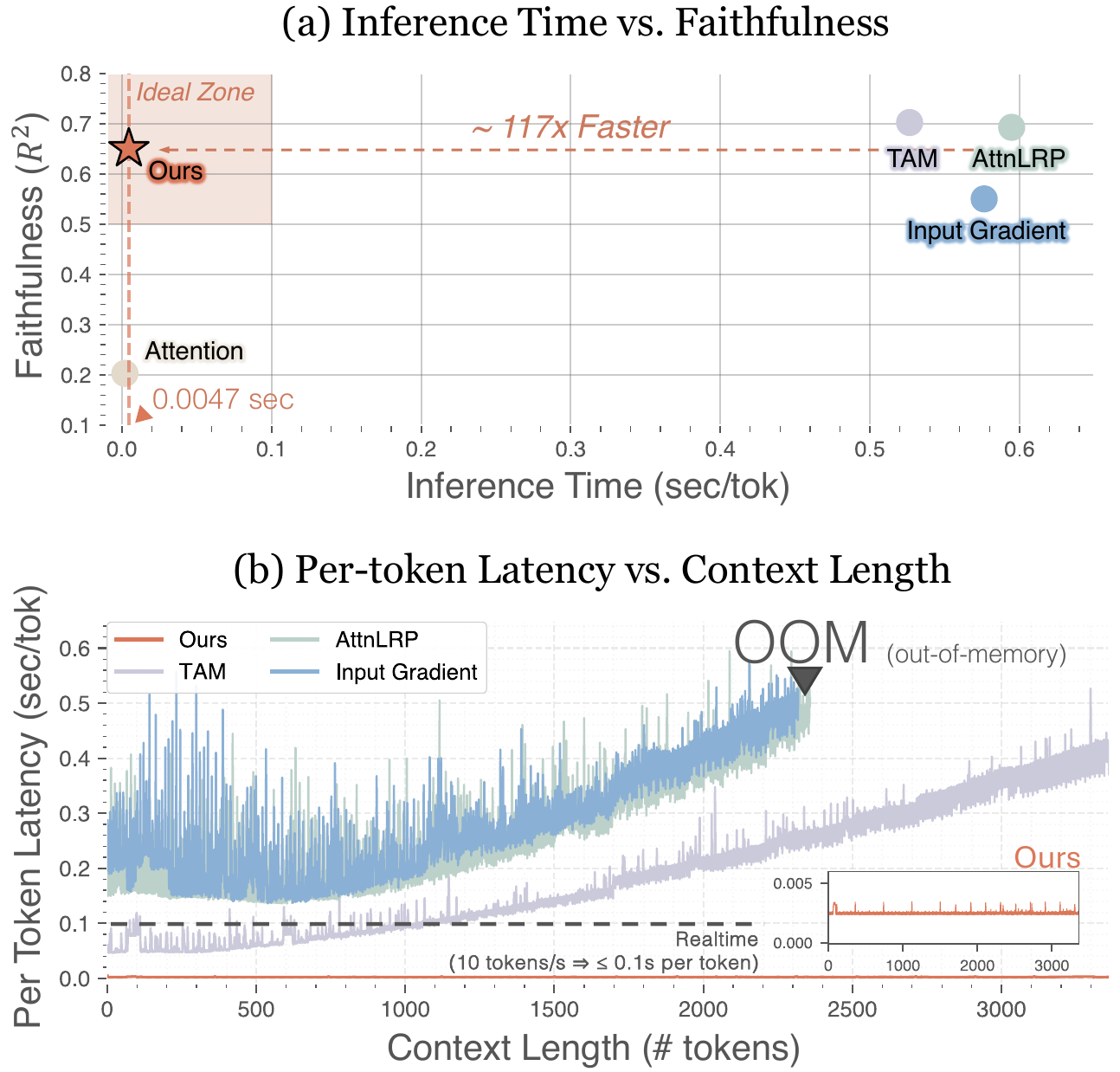}
\caption{\textbf{Faithfulness-Efficiency Trade-off.} (a) Baseline methods compromise either efficiency or faithfulness. Our approach simultaneously achieves both faithfulness and efficiency ($R^2$ of predicted vs. actual logit drops). (b) Latency scaling with context length. Unlike the baseline, which exhibits linear cost growth and OOM errors on long traces, our method operates with constant overhead below the real-time threshold (dashed line).}
\label{fig:fig1}
\end{figure}

\section{Introduction}
\label{sec:intro}

Recent scaling of Vision-Language Models has shifted the focus from single-turn question answering to multi-step reasoning over visual evidence \citep{openai2024o1, openai2025o3, deepseek2025r1, google2024gemini2thinking, qwen2024qvq, yue2023mmmu}. Todays, they can generate code from website screenshots, solve geometry problems from diagrams, and interpret complex charts \citep{lu2023mathvista,yue2023mmmu,Zhang2024MathVerse}.
By decomposing queries into intermediate steps, these models produce extended reasoning traces that theoretically reveal \emph{which} visual elements support \emph{which} conclusions \citep{xu2024llavacot,wang2024qwen2vl,geminiteam2024gemini15}.
However, verifying whether these long traces genuinely rely on visual evidence remains a critical challenge.

In practice, reasoning traces often exhibit \emph{ungrounded hallucinations}, producing plausible narratives without actual visual support \citep{li2023evaluating,favero2024multimodal}.
Models may correctly cite ``the angle at vertex B'' while attending to an irrelevant image region \citep{tong2024eyes}, or hallucinate numerical values absent from the visual input \citep{jing2024faithscore}.
These failures create a diagnostic blind spot, making it unclear whether errors originate in perception or symbolic reasoning \citep{cao2025geopqa}.
As reasoning chains grow longer, models increasingly rely on language priors rather than visual evidence \citep{liu2025more}, yet without reliable tools to verify visual grounding, the interpretability promised by reasoning traces remains illusory.

Attribution methods offer a principled way to verify visual grounding by quantifying which image regions causally influence each generated token.
However, as shown in \cref{fig:fig1}a, existing approaches face a competing demands of faithfulness and efficiency.
Raw attention weights \citep{abnar-zuidema-2020-attention-flow} can be extracted instantly during inference, but they are unreliable as causal explanations because attention distributions can shift substantially without altering model predictions \citep{jain-wallace-2019-attention,pruthi-2020-deceive,wu2024faithfulness}.
Gradient-based \citep{zeiler2014visualizing,chefer-2021-beyond-attention,achtibat-2024-attnlrp} and perturbation-based methods \citep{fong2017meaningful,hooker2019benchmark,li-2025-tam} achieve higher faithfulness by measuring how outputs change when inputs are modified, but at substantial computational cost.
This cost grows with context length (\cref{fig:fig1}b). As reasoning traces extend to thousands of tokens, per-token latency increases dramatically, making real-time analysis infeasible.
For interactive debugging, where users iteratively refine prompts or inspect reasoning step-by-step, such latency is unacceptable.

We resolve this trade-off through \emph{amortized attribution} \citep{jethani2022fastshap,covert2024stochastic}, training a lightweight estimator as a surrogate model to predict attributions rather than recompute them from scratch.
Crucially, while individual attention weights are unreliable as direct explanations, the comprehensive attention patterns distributed across layers and heads serve as \emph{informative features} for predicting causal effects \citep{abnar-zuidema-2020-attention-flow, cohen2024contextcite, cohen2025learning}.
We hypothesize that these patterns contain sufficient signal to recover the counterfactual effect of masking a visual region on token probability.
To ensure perceptual relevance, we operate on semantic units derived from self-supervised features \citep{simeoni2025dinov3} rather than raw pixels.
By training a linear estimator to map attention features to rigorous ablation targets, we bypass the need for repeated inference.

With a minimal parameter footprint, we train our estimator on 2,000 examples in ${\sim}$4.5 hours of runtime on a single GPU, a one-time cost that amortizes over all subsequent inferences.
At test time, our method computes attribution asynchronously within the token generation loop, enabling users to inspect on the fly where the model draws from at each reasoning step.

\begin{figure}[t!]
\centering
\includegraphics[width=\linewidth]{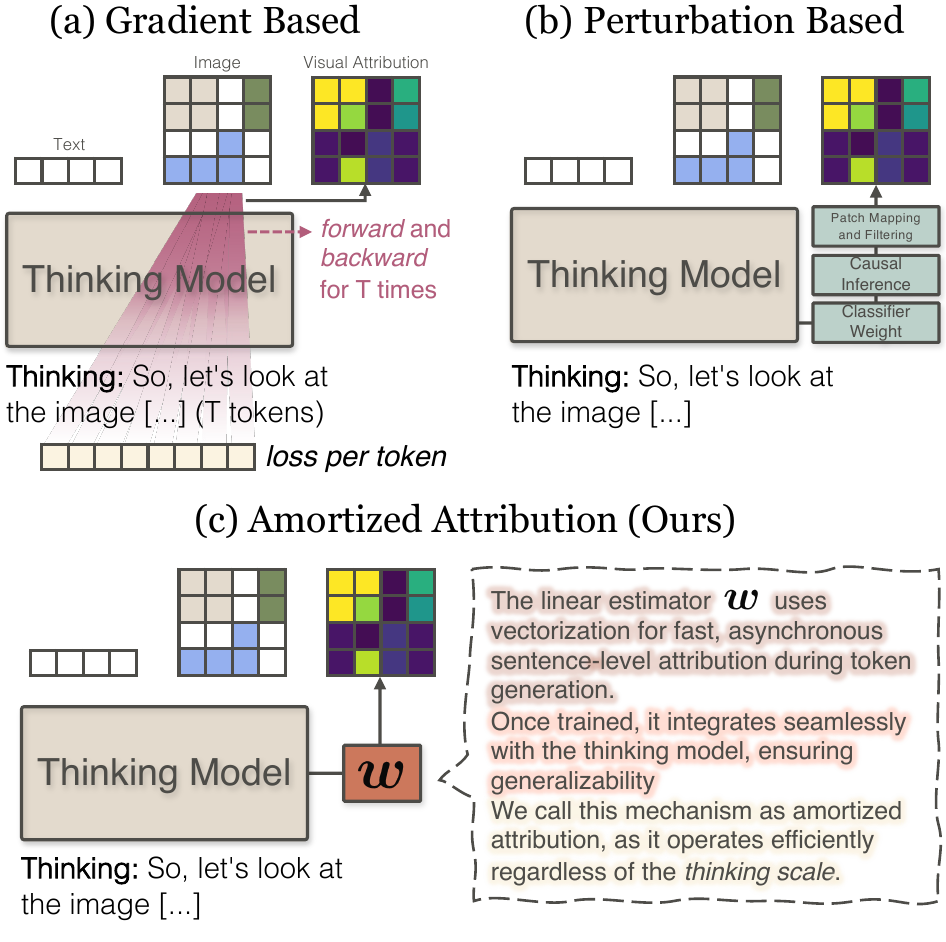}
\caption{\textbf{Comparison of attribution processes} between our method and baselines.}
\label{fig:method_comparison}
\end{figure}
\vspace{-7pt}


Our main contributions are as follows: 
\begin{itemize}
\item We formalize token-level visual attribution for multimodal reasoning as counterfactual region ablation effects, and identify the prohibitive scaling of direct perturbation for long reasoning traces (\cref{sec:problem}).
\item We propose a real-time attribution streaming method that (i) constructs semantically coherent attribution units from self-supervised vision features and (ii) amortizes counterfactual effects from a single forward pass (\cref{sec:unitization,sec:estimator}).
\item Our amortized estimator enables single-pass inference with cost linear in trace length while matching strong gradient- and perturbation-based baselines on faithfulness metrics across four reasoning VLMs and five task families (\cref{sec:estimator,sec:experiments}).
\item We further study models' step-by-step reasoning by providing an inspectable view of the semantic space over long traces and analyzing the reasoning trajectory dynamics enabled by attribution streaming (\cref{sec:trajectory}).
\end{itemize}


\section{Related Work}
\label{sec:related}

\subsection{Multimodal Reasoning Paradigms}
\label{sec:related_reasoning}
Recent vision-language models have shifted toward multi-step reasoning grounded in visual evidence, driven by benchmarks that explicitly require compositional explanations \citep{yue2023mmmu,Zhang2024MathVerse,Fu2024BLINK,lu2022scienceqa,lu2023mathvista,He2024OlympiadBench,yuan2025visualreasoningtracer,man2025argus,zhang2025gcot}.
Correspondingly, models now generate long-form ``thinking'' traces \citep{xu2024llavacot,wang2024qwen2vl,geminiteam2024gemini15, google2024gemini2thinking, openai2024o1, openai2025o3} and employ architectures designed for broader reasoning capabilities \citep{li2024llavaonevision,li2024llavanextinterleave,hong2024cogvlm2,luo2024monointernvl}.
While these traces offer a debugging interface, they can sound plausible even when ungrounded, masking whether errors are perceptual or symbolic.
Current reliability methods largely focus on final answer correctness rather than the supporting visual evidence for each model's thinking step \citep{prabhu2024trustverify,geigle2024objectgroundinghallucination,yan2025vigor,fu2025tldr,liu2026visionlanguageintrospection}.

\subsection{Faithful Attribution and Scalability}
\label{sec:related_attribution}
Attribution methods quantify feature relevance and are typically evaluated via counterfactual faithfulness \citep{hooker2019benchmark}, often using deletion or masking perturbations \citep{petsiuk2018rise,fong2017meaningful,Yu2024MMVet,Zhang2024MMERealWorld}.
While text-to-image cross-attention is a common proxy for focus, it is often unfaithful \citep{jain-wallace-2019-attention,serrano-smith-2019-attention,bibal-2022-attention-debate} and manipulable without affecting predictions \citep{pruthi-2020-deceive}.
Although gradient or attention-flow methods \citep{abnar-zuidema-2020-attention-flow,chefer-2021-beyond-attention,achtibat-2024-attnlrp} and recent multimodal tools \citep{li-2025-tam,stan2024lvlminterpret,shen2025glimpse,liang2025intramodaltokeninteractions} improve fidelity, they scale poorly.
Reasoning models producing thousands of tokens \citep{xu2024llavacot,geminiteam2024gemini15} make interactive per-token attribution intractable.
Prior amortized attribution methods train surrogates over SHAP-style value functions \citep{jethani2022fastshap,covert2024stochastic}, and concurrent work AT2 \citep{cohen2025learning} uses attention as features for text attribution.
We extend this line to visual tokens under streaming generation, replacing Monte Carlo sampling with a linear estimator over cached cross-attention features that adds negligible overhead to the decoding loop.

\begin{figure*}[!t]
\centering
\includegraphics[width=.95\linewidth]{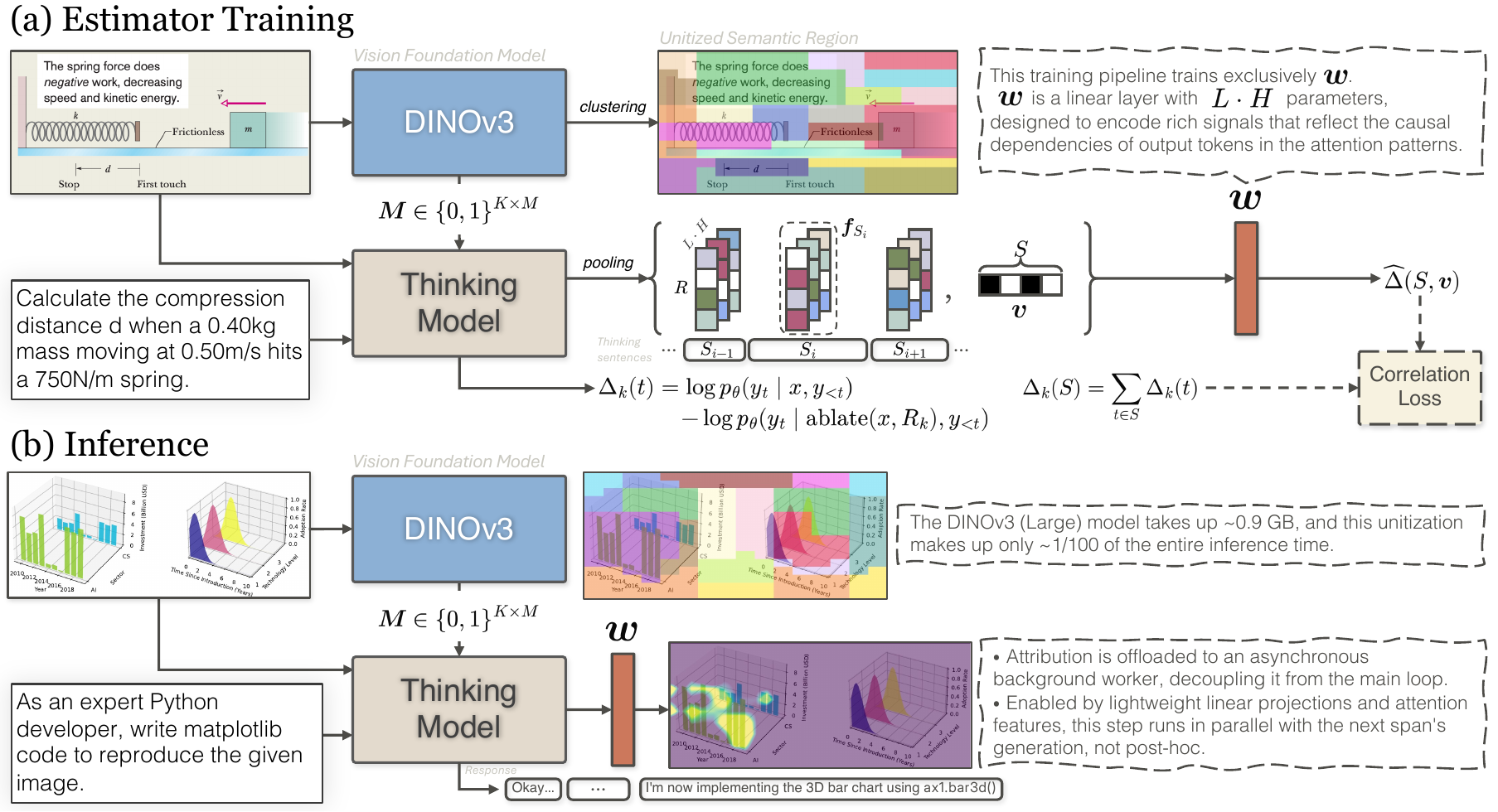}
\caption{\textbf{Overview of our amortized attribution pipeline.} (a) Training: We use our semantic region unitization to identify semantic regions of an input image using DINO features and optimize a lightweight estimator to predict causal importance from attention patterns.
(b) Inference: Once the estimator is trained, it generalizes to other samples and computes visual attribution in parallel with the model's token generation.}
\vspace{-3pt}
\label{fig:fullpipeline}
\end{figure*}

\section{Method}
\label{sec:method}

We present \textsc{vStream} (Visual Attribution Streaming), a framework for real-time, object-centric attribution in reasoning VLMs (\cref{fig:fullpipeline}).
Our approach sidesteps the computational bottleneck of existing methods \citep{chefer-2021-beyond-attention,li-2025-tam} by amortizing the counterfactual computation, learning to predict ablation effects from lightweight features extracted in a single forward pass.

\subsection{Background}
We consider a vision-language model (VLM) that takes a multimodal input $x=(I,q)$, where $I$ is an image and $q$ is a textual query.
The model autoregressively generates a sequence of tokens $y_{1:T}=(y_1,\dots,y_T)$.
For reasoning-centric VLMs, this sequence typically includes both intermediate reasoning tokens (the ``thinking'' trace) and the final answer.
We denote the model distribution at step $t$ as $p_\theta(\cdot \mid x, y_{<t})$.
\subsection{Problem Formulation}
\label{sec:problem}

Let the image $I$ be encoded into $M$ vision tokens, and let these tokens be partitioned into $K$ disjoint regions $\mathcal{R}=\{R_1,\dots,R_K\}$, where each $R_k \subseteq \{1,\dots,M\}$.
We define an ablation operator $\mathrm{ablate}(x, R_k)$ that prevents the model from accessing visual information in region $R_k$ during generation.

\textbf{Ablation via Attention Masking.}
We implement ablation by preventing information flow from specific vision tokens to the query, effectively removing them from the model's context.
Following \citet{geva2023dissecting}, we set the attention scores from all query positions to the vision tokens in region $R_k$ to $-\infty$ before the softmax normalization.
Formally, let $\vq_t \in \mathbb{R}^d$ denote the query vector at position $t$ and let $\mK \in \mathbb{R}^{d \times M}$ contain the key vectors for all $M$ vision tokens.
The standard attention logits are $\vs = \vq_t^\top \mK / \sqrt{d}$.
Under ablation of region $R_k$, we modify the logits as:
\begin{equation}
    s_i \leftarrow \begin{cases}
        -\infty & \text{if } i \in R_k \\
        s_i & \text{otherwise}
    \end{cases}
    \label{eq:ablation_mask}
\end{equation}
for all layers and heads.
This masking-based intervention preserves the original token positions and model state while cleanly removing the target region's influence on subsequent generation \citep{cohen2025learning, cohen2024contextcite}.

\textbf{Ablation Effect.}
We define the token-level ablation effect of region $R_k$ at decoding step $t$ as the drop in log-probability of the generated token \citep{wang2022interpretability, zhang2023towards}:
\begin{equation}
    \Delta_k(t) = \log p_\theta(y_t \mid x, y_{<t}) - \log p_\theta(y_t \mid \mathrm{ablate}(x,R_k), y_{<t}).
    \label{eq:delta_token}
\end{equation}
A positive $\Delta_k(t)$ indicates that region $R_k$ supports the generation of token $y_t$.

We extend this token-level metric to semantic spans.
Let $S \subseteq \{1,\dots,T\}$ denote a target span (e.g., a thinking step).
We aggregate token-level effects to define the span-level ablation effect:
\begin{equation}
    \Delta_k(S) = \sum_{t\in S} \Delta_k(t).
    \label{eq:delta_span}
\end{equation}
Under teacher-forcing, where we condition on the originally generated sequence $y_{<t}$ for both the original and ablated models, this sum equals the log-probability ratio of the entire span $\log p(y_S \mid x) - \log p(y_S \mid \mathrm{ablate}(x, R_k))$.

\textbf{From Tokens to Semantic Regions.}
While defining $R_k$ as individual vision tokens is theoretically possible, it is practically suboptimal \citep{ren2003learning}.
First, individual visual tokens are semantically ambiguous; a single token might capture an edge or texture that has no independent meaning without its neighbors.
Second, visual information is spatially sparse and structured. Their salient information is concentrated in objects rather than distributed uniformly.
Grouping tokens into semantically coherent units not only aligns attribution with human concepts (e.g., ``the red car'' vs. ``token 142'') but also drastically reduces the search space for ablation.
We therefore partition the image into semantic regions, making the counterfactual objective both interpretable and computationally tractable.

\subsection{Semantic Region Unitization}
\label{sec:unitization}

The choice of region partition $\mathcal{R}$ is critical, as attributions to arbitrary patches are hard to interpret while attributions to semantic regions (objects, text blocks, diagram components) are actionable.
We use features from DINOv3 \citep{simeoni2025dinov3}, a visual foundation model, to partition vision tokens into semantically coherent regions.
After aligning feature maps to the VLM's vision encoder resolution, we apply agglomerative clustering with Ward's linkage \citep{ward1963hierarchical}, adaptively grouping similar tokens into $K$ disjoint regions (typically $K \in [16, 128]$) that isolate objects and background elements without requiring external segmentation masks.
We denote the resulting partition as a membership matrix $\mM \in \{0,1\}^{K \times M}$, where $\mM_{k,i}=1$ indicates vision token $i$ belongs to region $R_k$.
See \cref{app:clustering} for more details and explanations.

For ablation studies, we also consider three non-semantic baselines:
\begin{itemize}
    \item \textbf{Token-wise}: Each vision token is its own region ($K = M$).
    \item \textbf{Random blocks}: Partition the image into $K$ randomly sized rectangular regions.
    \item \textbf{Voronoi}: Partition using a regular grid of $K$ Voronoi cells.
\end{itemize}
These baselines use fixed $K \in \{16, 32, 64, 128\}$, while DINO clustering determines $K$ adaptively; see \cref{sec:ablation} for quantitative comparisons.

\subsection{Attention Feature Extraction}
\label{sec:features}

To predict ablation effects without actually ablating, we need features that capture how much each region contributes to generating the target span.
Attention patterns provide a natural signal: if a region strongly influences the generation of a span, blocking attention to that region should substantially change the output.

Let $\mA^{(\ell,h)}_{t,i}$ denote the attention weight from token position $t$ to vision token $i$ in layer $\ell$ and head $h$.
For each target span $S$ and region $R_k$, we compute the mean attention pooled over the span and the region:
\begin{equation}
    \va^{(\ell,h)}_{S,k} = \frac{1}{|S| \cdot |R_k|} \sum_{t \in S} \sum_{i \in R_k} \mA^{(\ell,h)}_{t,i}.
    \label{eq:attn_pool}
\end{equation}
We concatenate these scalars across all layers and heads to form a feature vector:
\begin{equation}
    \vf_{S,k} = \mathrm{Concat}_{\ell,h}\big( \va^{(\ell,h)}_{S,k} \big) \in \mathbb{R}^{L \cdot H},
    \label{eq:feature_vec}
\end{equation}
where $L$ is the number of layers and $H$ is the number of heads.
For a model with 32 layers and 36 heads, this yields a 1152-dimensional feature vector per (span, region) pair.
We extract these features from samples where the model produces correct final answers, ensuring that the learned estimator captures attention patterns associated with successful reasoning (see \cref{app:implementation:datasets} for dataset details).
Importantly, attention weights are already computed during generation, so extracting $\vf_{S,k}$ adds negligible overhead.

\subsection{Amortized Estimator}
\label{sec:estimator}

Given the attention features $\vf_{S,k}$ for each region, we train a linear estimator to predict ablation effects.
Let $\vv \in \{-1, +1\}^K$ be a signed region mask where $\vv_k = -1$ indicates ablation and $\vv_k = +1$ indicates retention.
This is distinct from the membership matrix $\mM \in \{0,1\}^{K \times M}$, which assigns tokens to regions.
The signed encoding ensures that $\vf_{S,\vv}$ (\cref{eq:mask_feature}) captures the difference between retained and ablated contributions, rather than a one-sided sum of retained features only.
We define the combined feature for mask $\vv$ as:
\begin{equation}
    \vf_{S,\vv} = \sum_{k=1}^{K} \vv_k \cdot \vf_{S,k}.
    \label{eq:mask_feature}
\end{equation}
A linear estimator with weights $\vw \in \mathbb{R}^{L \cdot H}$ predicts the total ablation effect:
\begin{equation}
    \widehat{\Delta}(S, \vv) = \vw^\top \vf_{S,\vv}.
    \label{eq:estimator}
\end{equation}
The estimator has only $L \cdot H$ parameters, learning a single importance weight per layer-head pair.

\textbf{Training.}
For each training sample, we generate $N{=}32$ random masks $\vv^{(j)}$ and compute the ground-truth ablation effect $\Delta(S, \vv^{(j)})$ via forward passes.
We optimize $\vw$ to maximize Pearson correlation between predictions and targets:
\begin{equation}
    \mathcal{L}(\vw) = -\rho\big(\{\widehat{\Delta}(S, \vv^{(j)})\}_{j}, \{\Delta(S, \vv^{(j)})\}_{j}\big).
    \label{eq:loss}
\end{equation}
This scale-invariant objective focuses on relative ranking rather than absolute magnitudes, aligning with our evaluation metric.
We evaluate deeper MLP variants in \cref{app:mlp_ablation}: the linear model matches their accuracy while remaining orders of magnitude more parameter-efficient.
See \cref{app:theory} for theoretical justification and \cref{app:implementation:training} for training details.

\subsection{\textsc{vStream}: Visual Attribution Streaming}
\label{sec:inference}

As the model generates its reasoning trace, we simultaneously stream visual attributions by caching text-to-vision cross-attention weights at each decoding step.
Upon completing a semantic span $S$, we aggregate trained weights into $\vf_{S,k}$ and score each region via
\begin{equation}
    \widehat{\Delta}_k(S) = \vw^\top \vf_{S,k}.
    \label{eq:per_region_score}
\end{equation}
Crucially, this computation is offloaded to an asynchronous background worker, decoupled from the main generation loop via a producer-consumer pattern.
Because the attribution step runs in parallel with the generation of the subsequent span, it introduces near-zero latency overhead to the user experience (\cref{fig:fullpipeline}).
See \cref{app:streaming} for detailed streaming architecture.

While region-level scores capture relevant objects, finer localization is often required.
We redistribute the region attribution $\widehat{\Delta}_k(S)$ across patches using the DINOv3 attention map $\va^\mathtt{DINO}$ (computed during unitization) as a spatial prior:
\begin{equation}
    s_i = \widehat{\Delta}_{R(i)}(S) \cdot \frac{\va^\mathtt{DINO}_i}{\sum_{j \in R(i)} \va^\mathtt{DINO}_j}.
    \label{eq:patch_score}
\end{equation}
This refinement highlights the most salient parts within each region while preserving the calibrated total importance from the estimator.
See \cref{app:theory} for a complexity analysis.


\definecolor{bestcolor}{RGB}{232, 245, 233}    
\definecolor{secondcolor}{RGB}{255, 249, 235}  

\newcommand{\best}[1]{\colorbox{bestcolor}{#1}}
\newcommand{\second}[1]{\colorbox{secondcolor}{#1}}
\newcommand{\ours}{\textbf{vStream (Ours)}}
\newcommand{\circlemarker}[1]{%
    \tikz[baseline=-0.6ex]\draw[#1, fill=#1, radius=3pt] (0,0) circle (3pt);%
}
\begin{table*}[!t]
\centering
\caption{
  \textbf{Attribution quality comparison.}
  Each cell: (LDS\,/\,Top-5 Drop); higher is better.
  \circlemarker{bestcolor} Green: best, \circlemarker{secondcolor} cream: second best.
  We define \emph{real-time} as $\geq$10 tokens/sec, exceeding the average human silent reading speed of 238 words/min ($\approx$6 tokens/sec) \citep{brysbaert2019reading}, ensuring users can follow streaming output; \textsc{vStream} (Ours) achieves up to $117\times$ speedup over gradient-based baselines.
}
\label{tab:main_results}
\vspace{-5pt}
\resizebox{\textwidth}{!}{%
\begin{tabular}{l|ccccc|c|c||ccccc|c|c}
\hline
& \multicolumn{7}{c||}{Qwen3-VL-8B-Thinking\hspace{6pt}} 
& \multicolumn{7}{c}{\hspace{6pt}GLM-4.1V-9B-Thinking} \\
Method 
& Math & Science & Document & Code & General & Avg. & Time (s/10tok)
& Math & Science & Document & Code & General & Avg. & Time (s/10tok) \\
\hline
Random & 0.31/0.08 & 0.28/0.11 & 0.32/0.15 & 0.33/0.09 & 0.29/0.07 & 0.31/0.10 & .010$\pm$.002 & 0.28/0.12 & 0.33/0.14 & 0.31/0.13 & 0.33/0.10 & 0.28/0.11 & 0.31/0.12 & .011$\pm$.003 \\
Attention & 0.39/0.31 & 0.37/0.28 & 0.42/0.33 & 0.39/0.29 & 0.43/0.26 & 0.40/0.29 & .020$\pm$.005 & 0.41/0.30 & 0.42/0.35 & 0.40/0.29 & 0.45/0.32 & 0.38/0.27 & 0.41/0.31 & .022$\pm$.006 \\
\cdashline{1-15}
InputGrad & 0.63/0.72 & \second{0.71}/0.68 & 0.65/0.71 & 0.58/\second{0.96} & 0.64/0.75 & 0.64/0.76 & 2.80$\pm$1.06 & 0.58/0.74 & 0.63/\best{1.12} & \best{0.69}/0.73 & 0.62/0.78 & 0.63/0.81 & 0.63/0.84 & 3.05$\pm$1.15 \\
AttnLRP & \best{0.76}/0.81 & 0.66/\best{1.08} & \second{0.73}/0.78 & \best{0.73}/0.82 & 0.65/\second{0.98} & \second{0.71}/0.89 & 2.60$\pm$1.04 & \best{0.72}/0.79 & \second{0.68}/0.77 & 0.64/\best{1.05} & \best{0.71}/0.84 & \second{0.70}/0.75 & \second{0.69}/0.84 & 2.84$\pm$1.13 \\
TAM & 0.69/\second{1.02} & 0.69/\second{0.95} & 0.66/\second{1.04} & \second{0.72}/\best{1.06} & \best{0.71}/0.89 & 0.69/\second{0.99} & 1.90$\pm$1.13 & \second{0.71}/\second{0.94} & 0.65/\second{0.91} & \second{0.68}/0.88 & 0.67/\second{1.01} & 0.68/\best{1.08} & 0.68/\second{0.96} & 2.07$\pm$1.23 \\
\ours & \second{0.75}/\best{1.05} & \best{0.71}/0.92 & \best{0.74}/\best{1.09} & 0.69/0.91 & \second{0.70}/\best{1.02} & \best{0.72}/\best{1.00} & \textbf{.024$\pm$.002} & 0.69/\best{0.97} & \best{0.74}/0.89 & 0.67/\second{1.01} & \second{0.68}/\best{1.10} & \best{0.73}/\second{0.95} & \best{0.70}/\best{0.98} & \textbf{.026$\pm$.002} \\
\hline
& \multicolumn{7}{c||}{MiMo-VL-7B\hspace{6pt}} 
& \multicolumn{7}{c}{\hspace{6pt}Cosmos-Reason1-7B} \\
Method 
& Math & Science & Document & Code & General & Avg. & Time (s/10tok)
& Math & Science & Document & Code & General & Avg. & Time (s/10tok) \\
\hline
Random & 0.30/0.14 & 0.28/0.18 & 0.29/0.11 & 0.30/0.16 & 0.28/0.12 & 0.29/0.14 & .009$\pm$.002 & 0.28/0.17 & 0.30/0.13 & 0.34/0.19 & 0.30/0.16 & 0.27/0.12 & 0.30/0.15 & .010$\pm$.002 \\
Attention & 0.45/0.34 & 0.41/0.32 & 0.38/0.29 & 0.40/0.33 & 0.39/0.28 & 0.40/0.31 & .018$\pm$.004 & 0.45/0.33 & 0.41/0.35 & 0.37/0.34 & 0.41/0.37 & 0.38/0.38 & 0.41/0.35 & .020$\pm$.005 \\
\cdashline{1-15}
InputGrad & 0.58/\second{1.06} & 0.62/0.78 & \second{0.70}/0.69 & 0.59/0.76 & \best{0.71}/0.68 & 0.64/0.79 & 2.52$\pm$0.95 & 0.58/0.71 & \best{0.73}/0.76 & 0.59/\second{0.98} & 0.61/0.72 & 0.60/0.75 & 0.62/0.78 & 2.75$\pm$1.04 \\
AttnLRP & \second{0.74}/0.80 & \best{0.75}/\second{0.96} & 0.62/\second{0.95} & \best{0.72}/0.81 & 0.66/\second{1.04} & \second{0.70}/0.91 & 2.34$\pm$0.94 & \best{0.76}/0.80 & 0.65/\best{1.09} & \best{0.71}/0.78 & \second{0.69}/\second{1.00} & \best{0.72}/0.79 & \best{0.71}/0.89 & 2.55$\pm$1.02 \\
TAM & 0.70/0.98 & \second{0.66}/0.88 & 0.66/0.91 & \second{0.68}/\best{1.04} & 0.64/0.96 & 0.67/\second{0.95} & 1.71$\pm$1.02 & \second{0.75}/\second{0.94} & \second{0.66}/\second{1.02} & \second{0.68}/0.89 & 0.65/0.95 & \second{0.69}/\second{1.01} & 0.69/\second{0.96} & 1.86$\pm$1.11 \\
\ours & \best{0.75}/\best{1.08} & 0.65/\best{0.97} & \best{0.71}/\best{0.98} & 0.70/\second{0.96} & \second{0.70}/\best{1.10} & \best{0.70}/\best{1.02} & \textbf{.022$\pm$.002} & 0.74/\best{0.99} & 0.67/0.90 & 0.68/\best{1.01} & \best{0.71}/\best{1.04} & 0.69/\best{1.05} & \second{0.70}/\best{1.00} & \textbf{.024$\pm$.002} \\
\hline
\end{tabular}%
}
\vspace{-6pt}
\end{table*}
\section{Experiments}
\label{sec:experiments}

We evaluate our method on four reasoning VLMs across five task categories, addressing three questions: (1) Does our method accurately predict ablation effects? (2) Does the estimator generalize across tasks? (3) What do the attributions reveal about model behavior? Lastly, we conduct additional analysis to understand how visual attribution dynamics differ between successful and unsuccessful reasoning.

\subsection{Setup}
\label{sec:setup}

\textbf{Models and datasets.}
We use four reasoning VLMs that generate extended thinking chains: Qwen3-VL-8B-Thinking \citep{qwen2025qwen3vl}, GLM-4.1V-9B-Thinking \citep{glm}, MiMo-VL-7B-RL \citep{mimovl}, and Cosmos-Reason1-7B \citep{cosmosr1}.
To test generalization, we cover five categories: Math, Science, Document, Code, and General.
See \cref{app:implementation:datasets} for dataset details per category.

\textbf{Baselines.}
We compare against five baselines spanning different paradigms:
\begin{itemize}
    \item \textbf{Random}: Uniform random attribution scores.
    \item \textbf{Attention}: Average attention weights from text tokens to vision tokens, pooled across all layers and heads.
    \item \textbf{InputGrad} \citep{inputgradient}: Gradient of output logits with respect to input pixel values.
    \item \textbf{AttnLRP} \citep{achtibat-2024-attnlrp}: Attention-aware layer-wise relevance propagation.
    \item \textbf{TAM} \citep{li-2025-tam}: Token activation maps based on intermediate activations.
\end{itemize}
We adapt all baselines to produce region-level scores using our semantic unitization, ensuring fair comparison at the same granularity.

\textbf{Metrics.}
We use two complementary metrics:
\begin{itemize}
    \item \textbf{LDS (Linear Datamodeling Score)} \citep{park2023trak, cohen2024contextcite, cohen2025learning}: Spearman correlation between predicted and actual ablation effects across regions, measuring how well the method ranks regions by causal importance.
    \item \textbf{Top-K Drop} \citep{chattopadhay2018grad}: Log-probability drop when ablating the $K$ most attributed regions. We use $K{=}5$; higher drop indicates the method correctly identified causally important regions.
\end{itemize}
LDS evaluates ranking fidelity against ground-truth counterfactual effects, while Top-K Drop measures practical utility for identifying critical visual evidence.

\subsection{Main Results}
\label{sec:main_results}
\textbf{Attribution quality.}
\cref{tab:main_results} compares methods across four models and five categories.
\textsc{vStream} matches the strongest baselines in LDS and achieves best or second-best Top-5 Drop in 16/20 settings.
\cref{fig:prediction_quality} shows \textsc{vStream} achieves $R^2 = 0.65$ between predicted and actual effects.
Crucially, \textsc{vStream} streams these attributions with negligible overhead, whereas gradient-based methods require backward passes and perturbation-based methods require additional compute steps.
We also evaluate segmentation quality via mIoU in \cref{app:exp:segmentation}.

\textbf{Cross-task generalization.}
\cref{tab:cross_task} examines whether an estimator trained on one task category generalizes to others.
Using Qwen3-VL, we train on each category independently and evaluate on all five.
In-domain (diagonal) LDS ranges from 0.70--0.74, while cross-task transfer retains 75--90\% of this performance for most pairs.
Math and Science show strong mutual transfer (LDS 0.62--0.63), likely due to shared diagram structures.
Transfer to Document tasks is weaker (LDS 0.54--0.58), reflecting the distinct visual layout of dense text and tables.
Training on a mixture of all categories recovers full performance, suggesting a single estimator suffices for diverse applications.
See \cref{app:exp} for results on other models.

\textbf{Qualitative analysis.}
\label{sec:qualitative}
\cref{fig:qualitative} illustrates \textsc{\textsc{vStream}}'s unique capability: streaming attributions for each thinking step as the model reasons.
Unlike baselines that produce a single post-hoc map, \textsc{\textsc{vStream}} reveals which regions the model references across intermediate steps, exposing failure modes invisible to global methods.
For instance, a model may attend to the correct region initially but drift to irrelevant areas mid-reasoning.
Additional examples across models and task categories are provided in \cref{app:qualitative}.


\begin{table}[!t]
\centering
\caption{
  \textbf{Cross-task generalization.}
  Each cell: (LDS\,/\,Top-5 Drop). Diagonal: in-domain.
}
\label{tab:cross_task}
\resizebox{\columnwidth}{!}{%
\begin{tabular}{l|ccccc|c}
\toprule
\multirow{2}{*}{\textbf{Train}} & \multicolumn{5}{c|}{\textbf{Eval Category}} & \multirow{2}{*}{\textbf{Avg.}} \\
 & Math & Science & Document & Code & General & \\
\midrule
Math    & \cellcolor{gray!20}0.72\,/\,0.82 & 0.63\,/\,0.68 & 0.54\,/\,0.48 & 0.60\,/\,0.58 & 0.55\,/\,0.50 & 0.61\,/\,0.61 \\
Science     & 0.62\,/\,0.66 & \cellcolor{gray!20}0.71\,/\,0.80 & 0.58\,/\,0.54 & 0.55\,/\,0.50 & 0.60\,/\,0.58 & 0.61\,/\,0.62 \\
Document     & 0.53\,/\,0.50 & 0.56\,/\,0.52 & \cellcolor{gray!20}0.70\,/\,0.78 & 0.57\,/\,0.54 & 0.65\,/\,0.72 & 0.60\,/\,0.61 \\
Code    & 0.59\,/\,0.60 & 0.54\,/\,0.48 & 0.56\,/\,0.52 & \cellcolor{gray!20}0.73\,/\,0.84 & 0.62\,/\,0.66 & 0.61\,/\,0.62 \\
General     & 0.57\,/\,0.54 & 0.61\,/\,0.64 & 0.64\,/\,0.70 & 0.60\,/\,0.62 & \cellcolor{gray!20}0.74\,/\,0.86 & 0.63\,/\,0.67 \\
\midrule
\rowcolor{gray!10}
Mix-up     & 0.75\,/\,1.05 & 0.71\,/\,0.92 & 0.74\,/\,1.09 & 0.69\,/\,0.91 & 0.70\,/\,1.02 & 0.72\,/\,1.00 \\
\bottomrule
\end{tabular}%
}
\vspace{-1em}
\end{table}

\begin{figure}[t]
\centering
\includegraphics[width=.90\linewidth]{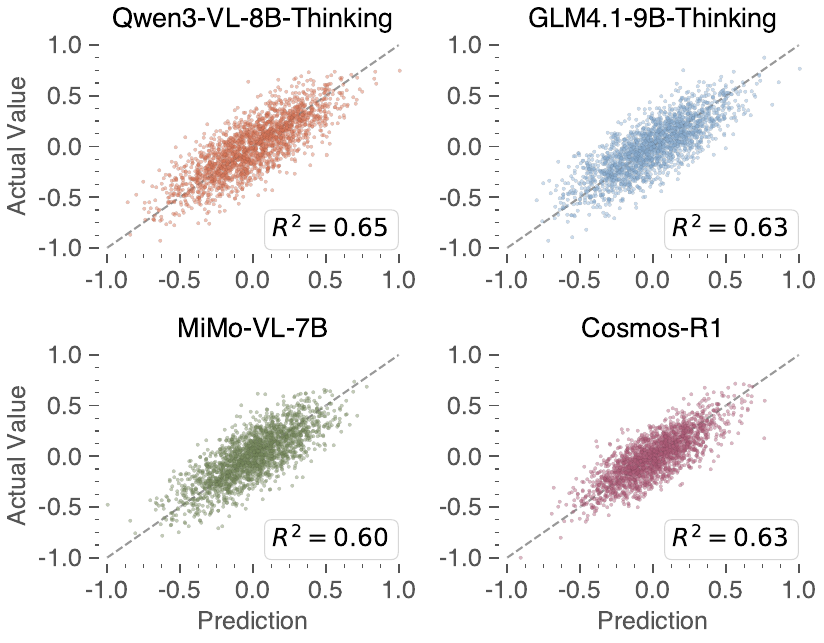}
\vspace{-3pt}
\caption{
    \textbf{Predicted vs.\ actual ablation effects.} Each point represents a region's predicted effect versus its ground-truth log-probability drop.
}
\vspace{-9pt}
\label{fig:prediction_quality}
\end{figure}

\begin{figure*}[t]
\centering
\includegraphics[width=.935\linewidth]{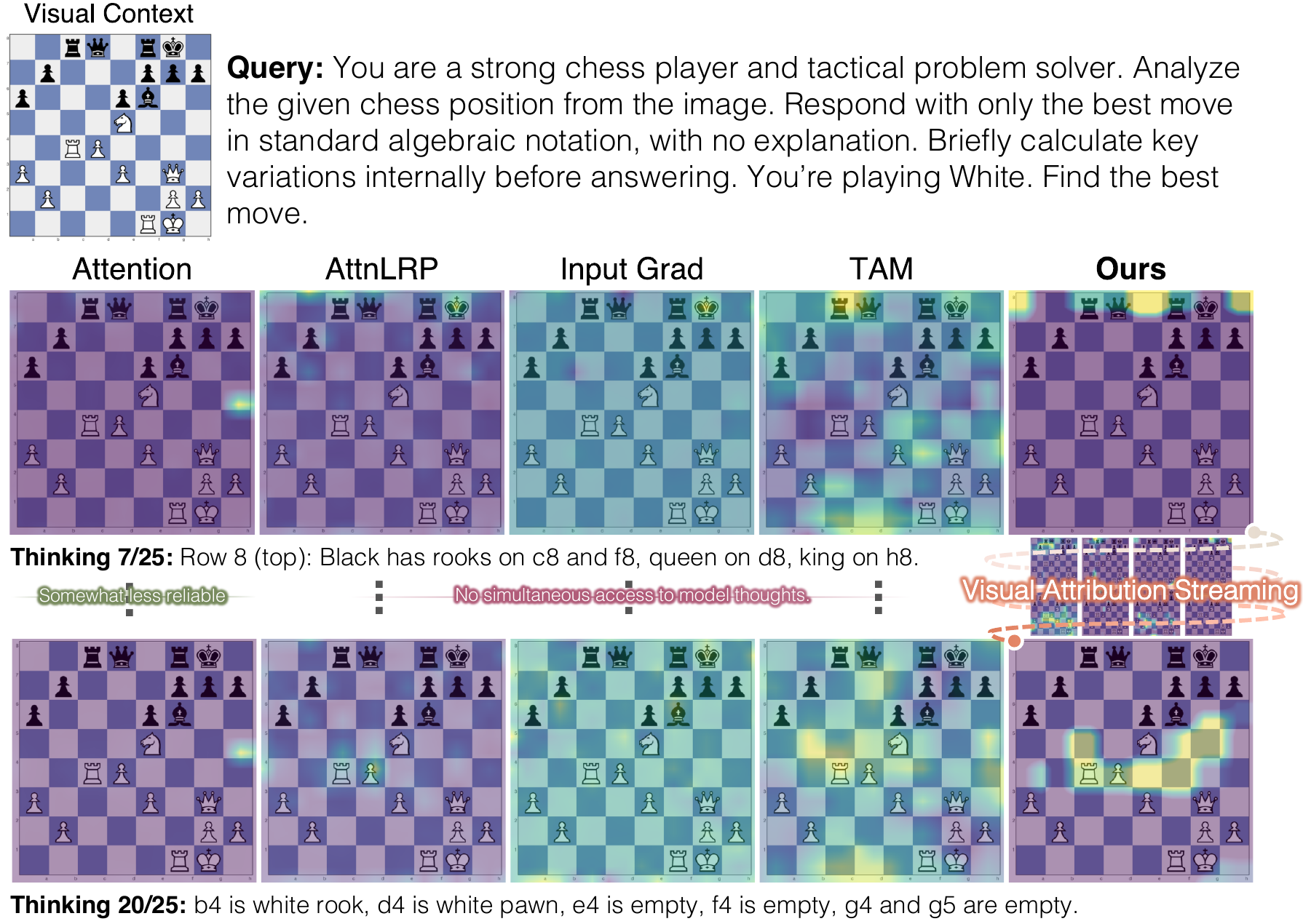}
\caption{
    \textbf{Qualitative comparison on a real-world sample (Qwen3-VL).} \textsc{vStream} emits per-step visual attributions alongside the model's reasoning at near-zero latency, while prior methods only run post-hoc once generation has finished.
}
\vspace{-7pt}
\label{fig:qualitative}
\end{figure*}

\subsection{Ablation Study}
\label{sec:ablation}

\textbf{Semantic regions outperform geometric partitions.}
In \cref{fig:abl_source_region}, we compare three region unitization strategies on Qwen3-VL: random rectangular blocks, regular Voronoi tessellation, and our DINOv3-based semantic clustering.
Our clustering approach significantly outperforms both geometric alternatives on the LDS metric.
Random block and Voronoi partitions use a fixed grid, placing a ceiling on performance, whereas DINOv3-based clustering adaptively adjusts regions to image content.
This confirms that semantic-level attribution to objects, symbols, and text blocks directly contributes to more interpretable and accurate results.
Results comparing other vision foundation models \citep{radford2021clip, zhai2023siglip} are provided in \cref{app:exp:backbone}.

\textbf{Training data efficiency.}
\label{sec:data_efficiency}
A practical concern is how much ablation data is needed to train an effective estimator, as shown in \cref{fig:abl_training_data}.
We vary the number of training examples from 100 to 5,000 and measure LDS on a held-out set (\cref{fig:abl_training_data}).
Performance improves steeply up to approximately 500 examples, then converges; with 2,000 examples, the estimator reaches full capacity.
This efficiency is enabled by two factors: (1) the estimator has only $L \times H$ parameters (e.g., 784), so overfitting is difficult, and (2) each example provides 32 mask-effect pairs, effectively multiplying the data.
Practitioners can train a domain-specific estimator with a few hours of ablation data collection, making the approach accessible for specialized applications.
See \cref{app:exp} for additional details.

\begin{figure}[!t]
\centering
\includegraphics[width=0.91\linewidth]{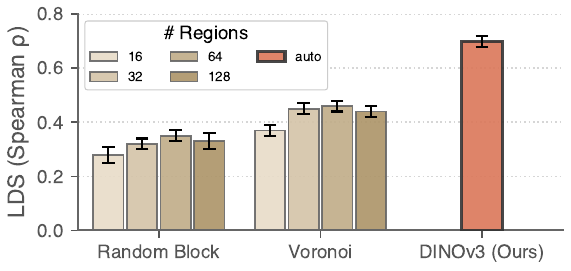}
\vspace{-3pt}
\caption{
    \textbf{Comparison of region unitization strategies.}
    }
\vspace{-15pt}
\label{fig:abl_source_region}
\end{figure}

\subsection{Reasoning Trajectory Dynamics}
\label{sec:trajectory}

Beyond static attribution maps, we ask whether the temporal evolution of visual reliance reveals signatures of reasoning quality.
Successful reasoning traces exhibit more stable visual grounding dynamics than unsuccessful ones, with attribution trajectories that move less and turn less over the course of the thinking process (\cref{fig:traj,fig:traj_strip}).

\definecolor{orange}{RGB}{218, 119, 88}
\definecolor{purple}{RGB}{174, 94, 122}
\begin{figure*}[!t]
\setlength{\fboxsep}{1.0pt}
\centering
\includegraphics[width=.93\linewidth]{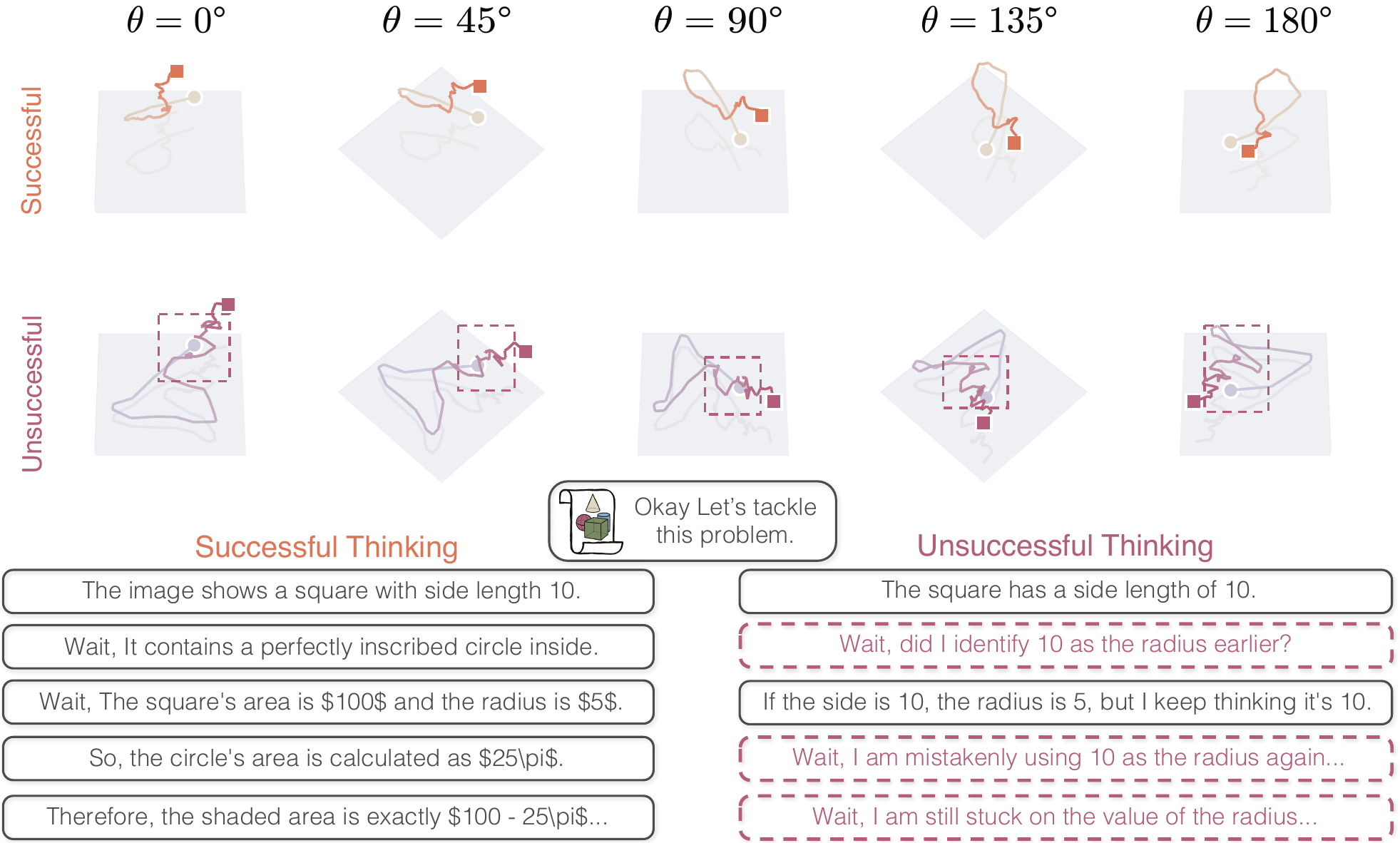}
\caption{
    \textbf{Reasoning trajectories in visual attribution space.}
    Each curve traces the evolution of a region-effect vector across reasoning steps, projected into 3D via PCA, The figure shows views from five different angles ($\theta$).
    The thinking process begins at the circular point and terminates at the square point.
    Successful reasoning chains (\circlemarker{orange} orange) follow compact, directed paths that converge toward stable visual grounding.
    Unsuccessful chains (\circlemarker{purple} purple) exhibit longer, more tortuous trajectories (see purple dashed boxes), reflecting repeated reassignment of visual support across regions.
}
\label{fig:traj}
\end{figure*}
\begin{figure}[!t]
\centering
\includegraphics[width=.80\linewidth]{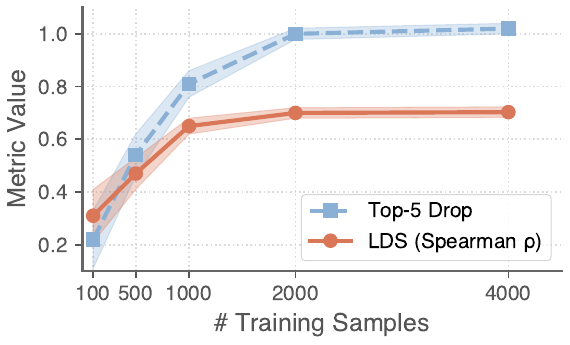}
\caption{
    \textbf{Training data efficiency.} LDS saturates at $\sim$2k examples; 1000 examples suffice for $>$90\% of peak performance.
}
\label{fig:abl_training_data}
\end{figure}

\begin{figure}[t]
\centering
\includegraphics[width=\linewidth]{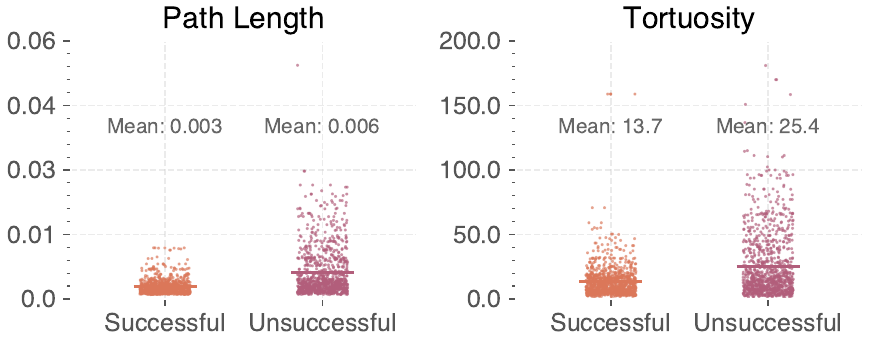}
\caption{
    \textbf{Distribution of trajectory metrics.}
    Unsuccessful reasoning chains exhibit higher path length and tortuosity than correct chains ($p{<}10^{-4}$, $n{=}1500$ each).
    The greater spread and outliers among unsuccessful samples reflect unstable visual grounding during failed reasoning.
}
\label{fig:traj_strip}
\end{figure}
\begin{figure*}[t]
\centering
\includegraphics[width=.935\linewidth]{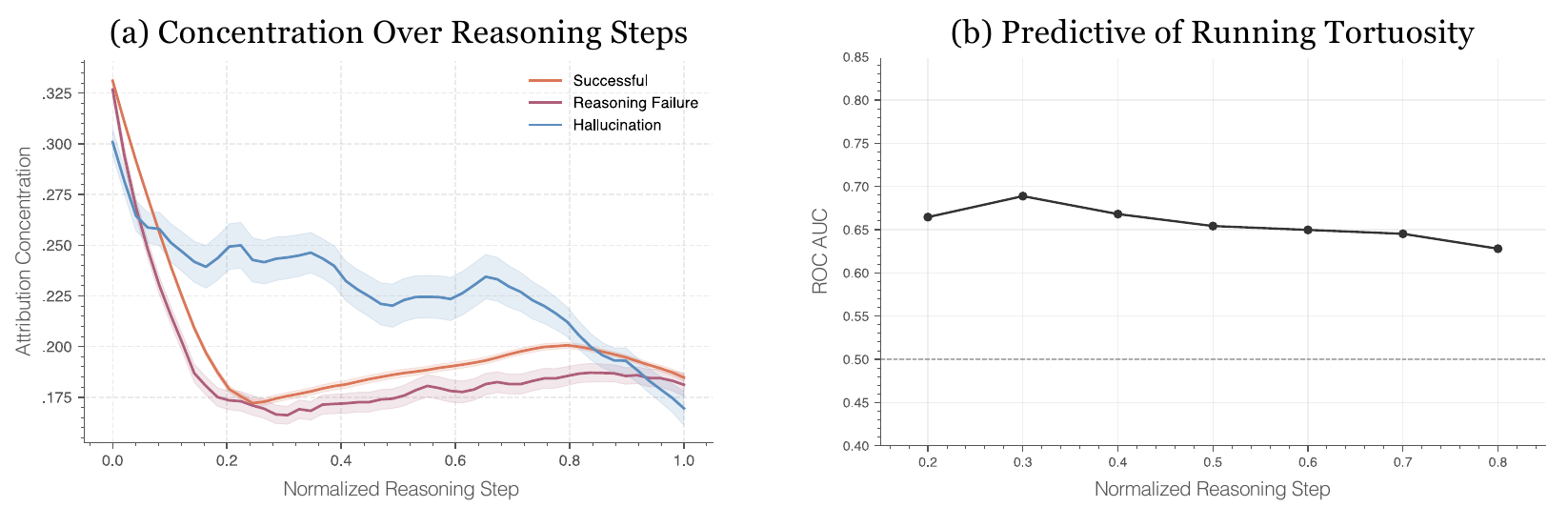}
\caption{
    \textbf{Attribution concentration and early failure detection.}
    \textit{(Left)} Mean concentration over normalized reasoning steps (mean $\pm$ SEM), grouped by outcome type.
    \textit{(Right)} Tortuosity-based failure prediction AUC over reasoning progress.
}
\label{fig:conc_and_auc}
\end{figure*}
\begin{figure}[t]
\centering
\includegraphics[width=\linewidth]{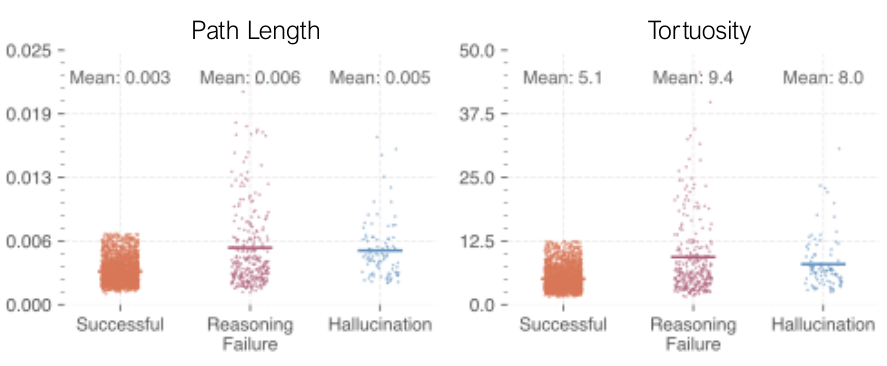}
\caption{
    \textbf{Trajectory metrics across three outcome categories (POPE, $n{=}3{,}000$).}
    Both reasoning failures and hallucinations exhibit higher path length and tortuosity than successful chains ($p{<}.001$, Bonferroni-corrected).
    The two error types are distinguished by concentration: hallucinations maintain sustained high concentration (\emph{Fixation}), while reasoning failures show unstable attention (\emph{Wandering}).
}
\label{fig:traj_3way}
\end{figure}

At each reasoning step $s$, we compute a \emph{region-effect vector} $\ve_s \in \mathbb{R}^{K}$ of predicted ablation effects. To compare examples with different numbers of regions, we canonicalize each step by keeping the top-$R$ regions ($R{=}32$) and project the resulting profiles with PCA for visualization.

As visualized in \cref{fig:traj}, unsuccessful reasoning chains exhibit more tangled and convoluted trajectories compared to successful ones, reflecting unstable visual grounding.
Quantitatively, as shown in \cref{fig:traj_strip}, successful chains have shorter path length in PCA space than unsuccessful chains ($0.003$ \textit{vs.} $0.006$, $n{=}1500$), and lower tortuosity, which measures how much the path wanders rather than progressing directly ($13.7$ \textit{vs.}\ $25.4$, $n{=}1500$).
We interpret this gap as reduced hypothesis switching \citep{zhong2024snowball}, where successful chains quickly commit to a consistent set of regions while failures repeatedly reassign visual support.

On POPE \citep{li2023evaluating}, unsuccessful cases further split into two geometrically distinct failure modes. Hallucinations sustain high attribution concentration throughout reasoning, a \emph{fixation} on a single incorrect object (also confirmed on CHAIR \citep{rohrbach2018object}; see \cref{app:exp:trajectory}), while reasoning errors show low, unstable concentration, a \emph{wandering} pattern of repeated region switching (\cref{fig:conc_and_auc}, left).
Both modes emerge well before generation ends: tortuosity-based failure prediction reaches AUC $0.69$ at $30\%$ of elapsed reasoning (\cref{fig:conc_and_auc}, right), and the two modes separate cleanly from successes in trajectory metric space (\cref{fig:traj_3way}).
A second independent early-warning signal appears in per-step fidelity $R^2$, which drops for incorrect chains at $\sim$$20\%$ elapsed (\cref{fig:fidelity_decay}).
Both early-warning signals require access to the per-step attribution stream rather than a post-hoc map.
Full statistics are in \cref{app:exp:trajectory}.

\begin{figure}[t]
\centering
\includegraphics[width=.90\linewidth]{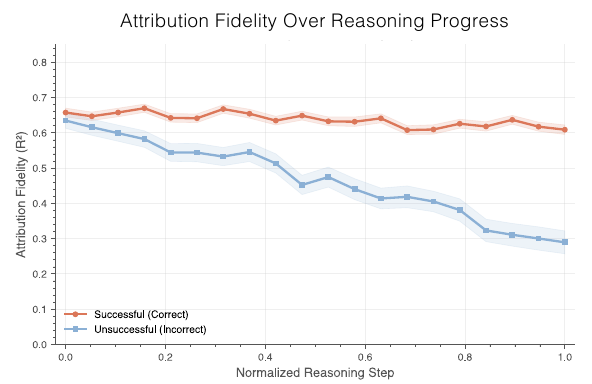}
\caption{
    \textbf{Per-step attribution fidelity ($R^2$) over normalized reasoning steps.}
    Correct chains remain stable; incorrect chains degrade at $\sim$$20\%$ of reasoning elapsed.
}
\vspace{-5pt}
\label{fig:fidelity_decay}
\end{figure}

\section{Conclusion}
\label{sec:conclusion}

We presented \textsc{vStream}, an amortized framework that enables real-time visual attribution streaming in multimodal thinking models.
By learning to predict causal ablation effects from attention features, our approach achieves faithfulness comparable to existing methods while adding negligible computational overhead.
It generalizes across models and tasks, preserving causal reliability at minimal cost.
Beyond static attribution, this efficiency enables trajectory analysis that uncovers a behavioral signature.
By design, \textsc{vStream} extends beyond thinking models to any autoregressive VLM.
We provide further discussion and limitation in the \cref{app:impact}.

\newpage\clearpage

\section*{Impact Statement}

This paper presents work whose goal is to advance the field of machine learning, specifically in the area of interpretability and transparency for vision-language models. By enabling real-time visual attribution, our method helps practitioners understand and verify model reasoning, which we believe contributes positively to the development of trustworthy AI systems. A more detailed discussion of broader impacts and limitations is provided in \cref{app:impact}.

\bibliography{main}
\bibliographystyle{icml2026}

\newpage\clearpage
\appendix

\makeatletter
\def\addcontentsline#1#2#3{%
  \addtocontents{#1}{\protect\contentsline{#2}{#3}{\thepage}{\@currentHref}}}
\makeatother

\onecolumn

\vspace*{\fill}

{\Large \textbf{Technical Appendices}}

\vspace{0.5cm}

\newcounter{appdxsectoc}

\begingroup
  \setlength{\columnsep}{2.5em}
  \etocsetnexttocdepth{subsection}
  \etocsettocstyle{}{}
  \etocsetstyle{section}
    {\setcounter{appdxsectoc}{0}}
    {}
    {\stepcounter{appdxsectoc}%
     \ifnum\value{appdxsectoc}=6\columnbreak\fi
     \vspace{0.3\baselineskip}%
     {\large\etoclink{\textbf{\etocnumber}}\hspace{0.75em}\etoclink{\etocname}}\par}
    {}
  \etocsetstyle{subsection}
    {}
    {}
    {\hspace{1.5em}{\etoclink{\etocnumber}\hspace{0.5em}\etoclink{\etocname}}\par}
    {}
  \begin{multicols}{2}
    \setlength{\parskip}{0.15\baselineskip}
    \setlength{\parindent}{0pt}
    \localtableofcontents*
  \end{multicols}
\endgroup

\vspace{\fill}

\newpage


\section{Theoretical Foundations}%
\label{app:theory}

In this section, we analyze our method through the lens of existing interpretability frameworks and formalize the properties of our estimator.

\subsection{Connections to Existing Frameworks}

\textbf{Connection to Causal Abstraction.}
Our attribution score $\Delta_k(t)$ aligns with the interventional definitions in causal abstraction \citep{geiger2025causal}, measuring the causal effect of removing visual region $R_k$. Although conceptually similar to activation patching, we do not explicitly define a high-level structural causal model (SCM). Instead, we treat the interventional effect $\Delta_k(t)$ as a ground-truth signal to be efficiently approximated.

\textbf{Relation to the Linear Representation Hypothesis.}
Our linear estimator $\widehat{\Delta} \approx \vw^\top \vf$ relies on the assumption that task-relevant causal information is linearly decodable, consistent with the Linear Representation Hypothesis \citep{park2024linear}. We extend this to attention-based features. We note that linearity is an empirical assumption; we cannot theoretically guarantee that the complex counterfactual effect of a visual ablation is perfectly captured by a linear projection.

\textbf{Interpretation as Marginal Contribution.}
The quantity $\Delta_k(t)$ represents the marginal contribution of region $R_k$ to the target token's log-probability, similar to Shapley values \citep{lundberg2017unified}. Unlike Shapley values, we do not average over all feature coalitions due to computational costs. Our metric captures individual impact but does not strictly guarantee efficiency or additivity.

\subsection{Properties of the Estimator}

We now formally state key properties of our estimator, motivating our choice of objective and quantifying the computational benefits.

First, we establish that the Pearson correlation objective is robust to affine transformations, justifying its use even if the scale of the predictor differs from the target.

\begin{lemma}[Affine invariance of the Pearson objective]
\label{lem:pearson-affine-invariance}
Let $u,v$ be random variables with $0<\Var(u),\Var(v)<\infty$. For any $a\neq 0$ and any $b\in\R$,
\[
\rho\!\left(u,\, a v + b\right)=\sign(a)\,\rho(u,v).
\]
In particular, for $a>0$ the Pearson correlation is invariant to scaling and shifting of the predictor.
\end{lemma}

\begin{proof}
Using $\rho(u,v)=\frac{\Cov(u,v)}{\sqrt{\Var(u)\Var(v)}}$, we have
$\Cov(u,av+b)=a\,\Cov(u,v)$ and $\Var(av+b)=a^2\Var(v)$, hence
$\rho(u,av+b)=\frac{a}{|a|}\rho(u,v)=\sign(a)\rho(u,v)$.
\end{proof}

Next, we analyze the computational advantage of our method compared to calculating true causal effects via brute-force ablation.

\begin{proposition}[Computational complexity comparison under teacher forcing]
\label{prop:complexity}
Fix a target span $S$ and regions $\{R_k\}_{k=1}^K$. Let $C_{\mathrm{fwd}}$ denote the cost of one VLM forward pass under teacher forcing, i.e., producing logits for all positions in $S$ while conditioning on the same prefix tokens $y_{<t}$ as in \cref{eq:delta_token}.

\textbf{(Direct ablation, exact).}
Computing the exact ablation effects $\{\Delta_k(t)\}_{t\in S}$ for all regions requires one forward pass on each ablated input, hence $K$ additional forward passes per span:
\[
\Theta(K\cdot C_{\mathrm{fwd}}),
\]
(or $K{+}1$ forward passes if counting the unablated baseline as well).

\textbf{(Ours, cached attention).}
After obtaining the cross-attention weights from the unablated run, we (i) pool cached attention over tokens and vision positions to form region features and (ii) apply a linear score for each region.
Step (ii) alone costs $O(K\cdot L\cdot H)$ for $K$ dot-products in $\R^{L\cdot H}$.
The total cost is dominated by lightweight tensor reductions over cached attention and does not require any additional VLM forward passes.
\end{proposition}

We also justify the form of our estimator. If the conditional mean of the target is linear in the features, maximizing correlation recovers the optimal direction.

\begin{proposition}[Optimal linear predictor under linear conditional mean]
\label{prop:optimal-linear-predictor}
Let $\vf\in\R^d$ be a feature vector with positive definite covariance $\mSigma_{ff}$, and let $\Delta$ be a scalar target. Assume:
\[
\E[\Delta \mid \vf] = \beta + \vw_\star^\top \vf
\quad\text{for some }\beta\in\R,\ \vw_\star\in\R^d.
\]
Then maximizing $\rho(\vw^\top \vf, \Delta)$ over $\vw$ yields $\widehat{\vw} \propto \vw_\star$.
\end{proposition}

\begin{proof}
Write $\rho(\vw^\top\vf,\Delta)=\frac{\vw^\top\mSigma_{f\Delta}}{\sqrt{\vw^\top\mSigma_{ff}\vw}\sqrt{\Var(\Delta)}}$ where $\mSigma_{f\Delta}=\Cov(\vf,\Delta)$.
Maximizing this Rayleigh quotient yields $\vw\propto \mSigma_{ff}^{-1}\mSigma_{f\Delta}$.
Under the linear conditional mean, $\mSigma_{f\Delta}=\mSigma_{ff}\vw_\star$, hence $\widehat{\vw}\propto \vw_\star$.
\end{proof}

We emphasize that the linear conditional mean is an \emph{assumption}, not a guaranteed property of VLMs. In practice, whether attention features linearly predict ablation effects is an empirical question; our experimental results suggest this approximation is effective, but it may not hold universally.

The following remark clarifies the connection between Pearson correlation and least-squares regression.

\begin{remark}[Correlation maximization as standardized least squares with optimal rescaling]
\label{rem:standardized-mse}
Let $\tilde{\Delta} = (\Delta - \E[\Delta])/\mathrm{std}(\Delta)$ so that $\E[\tilde{\Delta}]=0$ and $\Var(\tilde{\Delta})=1$.
For any predictor $z=\vw^\top \vf$, consider the best affine rescaling of $z$ to fit $\tilde{\Delta}$:
\[
\min_{\alpha\in\R,\, b\in\R} \ \E\big[(\tilde{\Delta} - (\alpha z + b))^2\big].
\]
The minimizer satisfies $b^\star=\E[\tilde{\Delta}-\alpha z]$ and $\alpha^\star=\Cov(\tilde{\Delta},z)/\Var(z)$, yielding
\[
\min_{\alpha,b}\ \E\big[(\tilde{\Delta} - (\alpha z + b))^2\big]
= 1 - \rho(z,\Delta)^2.
\]
Therefore, maximizing the Pearson objective $\rho(\vw^\top \vf,\Delta)$ is equivalent (up to an arbitrary sign of $\vw$) to minimizing the standardized mean squared error after allowing the optimal affine rescaling of predictions. This explains why our correlation-based loss focuses on the pattern of effects rather than their absolute scale.
\end{remark}

Finally, we note a basic consistency property: if a region has no associated features (e.g., zero attention), our estimator correctly predicts zero effect.

\begin{remark}[Dummy property]
\label{rem:dummy}
If $\vf_{S,k}=\vzero$, then $\widehat{\Delta}_k = \vw^\top \vf_{S,k} = 0$ for any $\vw$.
\end{remark}

\subsection{Detailed Complexity Analysis}
\label{app:theory:complexity}

We provide a detailed breakdown of computational costs for different attribution methods.

\begin{table}[h]
\centering
\caption{Computational complexity comparison for attributing $T$ tokens to $K$ regions under teacher forcing. A single forward pass produces logits (and attention weights) for all $T$ positions. $C_{\text{fwd}}$ and $C_{\text{bwd}}$ denote the cost of one forward and backward pass, respectively.}
\label{tab:complexity}
\small
\begin{tabular}{lcc}
\toprule
\textbf{Method} & \textbf{Complexity} & \textbf{Passes} \\
\midrule
Perturbation (exact) & $O(K \cdot C_{\text{fwd}})$ & $K$ forward \\
Gradient-based & $O(C_{\text{fwd}} + C_{\text{bwd}})$ & $1$ forward + backward \\
Attention (raw) & $O(T \cdot L \cdot H \cdot M)$ & 0 (cached) \\
\textbf{Ours} & $O(T \cdot L \cdot H \cdot M + T \cdot K \cdot L \cdot H)$ & 0 (cached) \\
\bottomrule
\end{tabular}
\end{table}

For a typical reasoning trace with $T = 2000$ tokens, $K = 50$ regions, $L = 28$ layers, and $H = 28$ heads (teacher forcing):
\begin{itemize}
    \item \textbf{Perturbation (exact)}: $K = 50$ forward passes on ablated inputs (plus one unablated baseline if counted)
    \item \textbf{Gradient-based}: $1$ forward-backward pair (using a loss aggregated over the $T$ tokens)
    \item \textbf{Ours (scoring only)}: $T \times K \times L \times H = 2000 \times 50 \times 28 \times 28 \approx 78\text{M}$ multiply-adds
\end{itemize}

The key insight is that our method's complexity is \emph{independent} of the VLM's size, depending only on the number of layers and heads. This makes it equally efficient for 7B and 70B parameter models.

\subsection{Why Attention Features Predict Ablation Effects}
\label{app:theory:why_attention}

A natural question is: why should attention features contain information about counterfactual ablation effects? We provide intuition from two perspectives.

\paragraph{Information Bottleneck Perspective.}
In transformer architectures, cross-attention weights modulate how much information flows from source positions (vision tokens) to target positions (generated text tokens).
If the model assigns large cross-attention mass to a region $R_k$, then the cross-attention block can inject more signal from that region into the token representation, so ablating $R_k$ should tend to have a larger counterfactual effect.

To make this precise, let $\vo_t^{(\ell)}$ denote the \emph{cross-attention output component} at text position $t$ in layer $\ell$ (i.e., the multi-head cross-attention output before adding the residual stream and before any MLP updates).
A standard multi-head cross-attention computation can be written as
\begin{equation}
\label{eq:info_flow_crossattn_component}
\vo_t^{(\ell)}
=
\mW_O^{(\ell)}
\Big[
\sum_{i=1}^{M} \mA^{(\ell,1)}_{t,i}\,\mW_V^{(\ell,1)}\vx_i
\ ;\ \dots\ ;\
\sum_{i=1}^{M} \mA^{(\ell,H)}_{t,i}\,\mW_V^{(\ell,H)}\vx_i
\Big],
\end{equation}
where $\mA^{(\ell,h)}_{t,i}$ is the cross-attention weight from text token $t$ to vision token $i$ in head $h$, $\mW_V^{(\ell,h)}$ is the value projection, and $\mW_O^{(\ell)}$ is the output projection.

The contribution of region $R_k$ to this cross-attention output is then the corresponding partial sum over $i\in R_k$:
\begin{equation}
\label{eq:info_flow_region_component}
\vo_t^{(\ell)}[R_k]
\approx
\mW_O^{(\ell)}
\Big[
\sum_{i\in R_k} \mA^{(\ell,1)}_{t,i}\,\mW_V^{(\ell,1)}\vx_i
\ ;\ \dots\ ;\
\sum_{i\in R_k} \mA^{(\ell,H)}_{t,i}\,\mW_V^{(\ell,H)}\vx_i
\Big].
\end{equation}
This approximation isolates the cross-attention pathway only; subsequent residual connections, normalization, and MLP mixing can further transform and redistribute this signal.
Nevertheless, \cref{eq:info_flow_region_component} highlights that cross-attention weights directly gate the magnitude of the region-dependent component injected into the representation, motivating why pooled attention features can correlate with ablation effects.

\paragraph{Gradient Flow Perspective.}
During backpropagation, attention also modulates sensitivity to input perturbations by gating how gradients propagate from text positions back to vision tokens.

Let $\mathrm{head}_t^{(\ell,h)}=\sum_{i=1}^{M}\mA^{(\ell,h)}_{t,i}\,\mW_V^{(\ell,h)}\vx_i$ denote the value-aggregation output of head $(\ell,h)$ at text position $t$.
Ignoring the dependence of $\mA^{(\ell,h)}_{t,i}$ on $\vx_i$ (i.e., dropping the additional terms that flow through the key and query pathways), the value-path contribution to the gradient can be written as the following heuristic approximation:
\begin{equation}
\label{eq:grad_flow_value_path}
\frac{\partial \mathcal{L}}{\partial \vx_i}
\approx
\sum_{t}\sum_{\ell,h}
(\mW_V^{(\ell,h)})^\top
\Big(
\mA^{(\ell,h)}_{t,i}\,
\frac{\partial \mathcal{L}}{\partial \mathrm{head}_t^{(\ell,h)}}
\Big).
\end{equation}
More generally, the exact gradient also contains additional terms arising from $\partial \mA^{(\ell,h)}_{t,i}/\partial \vx_i$ (through keys, and indirectly through queries when vision tokens influence later text states).
Thus, \cref{eq:grad_flow_value_path} should be interpreted as an intuition: attention weights gate one dominant pathway for gradient flow, which helps explain why attention-derived features can predict counterfactual sensitivity.

\subsection{Generalization Properties}
\label{app:theory:generalization}

We analyze factors affecting cross-task and cross-model generalization of our estimator.

\paragraph{Feature Distribution Shift.}
The estimator's generalization depends on how similar the attention feature distributions are between training and test conditions. Let $\mathcal{F}_{\text{train}}$ and $\mathcal{F}_{\text{test}}$ denote the feature distributions. If $\mathcal{F}_{\text{test}}$ is within the support of $\mathcal{F}_{\text{train}}$, the linear estimator can interpolate effectively. However, if test features lie outside the training distribution (e.g., novel visual patterns), extrapolation may fail.

Our cross-task experiments (\cref{tab:cross_task}) show that Math and Science transfer well to each other (shared diagram structures), while transfer to Document is weaker (distinct layouts). This aligns with the intuition that generalization depends on visual similarity.

\paragraph{Model Architecture Dependence.}
Different VLM architectures may encode visual information in different layers and heads. Our estimator learns a weighting $\vw$ over layer-head pairs that is specific to a given architecture. Cross-model transfer would require that:
\begin{enumerate}
    \item Both models have similar layer-head structure (same $L$ and $H$).
    \item Visual grounding emerges at similar depths in both models.
\end{enumerate}

Our experiments train separate estimators per model, avoiding cross-model transfer issues. Future work could explore architecture-agnostic features that enable zero-shot transfer.

\paragraph{Sample Complexity.}
Given the low dimensionality of our estimator ($L \cdot H$ parameters), we expect good generalization from relatively few samples. The training data efficiency experiments (\cref{fig:abl_training_data}) confirm that approximately 500--2000 examples suffice for convergence, consistent with standard generalization bounds for linear models in $\R^{784}$.

\section{Implementation Details}
\label{app:implementation}

In this section, we provide the technical specifications required to reproduce our method, including details on the reasoning models employed, the DINO-based visual clustering pipeline, the training procedure for the linear estimator, and the hardware environment.

\subsection{Datasets}
\label{app:implementation:datasets}

We organize our evaluation datasets into five categories based on the type of visual reasoning required.

\textbf{General.}
\begin{itemize}
    \item \textbf{GQA}~\citep{hudson2019gqa}: Real-world visual reasoning with compositional questions about spatial relations and object attributes.
\end{itemize}

\textbf{Document.}
\begin{itemize}
    \item \textbf{DocVQA}~\citep{mathew2021docvqa}: Question answering over document images including forms, invoices, and reports.
\end{itemize}

\textbf{Science.}
\begin{itemize}
    \item \textbf{ScienceQA}~\citep{lu2022scienceqa}: Multimodal science questions across natural, social, and language sciences with chain-of-thought reasoning.
\end{itemize}

\textbf{Math.}
\begin{itemize}
    \item \textbf{MathVista}~\citep{lu2023mathvista}: Mathematical reasoning over diagrams, plots, and geometric figures.
    \item \textbf{MathVision}~\citep{wang2024mathvision}: Multi-step mathematical problem solving with visual inputs.
    \item \textbf{MathVerse}~\citep{Zhang2024MathVerse}: Diagram-based math problems testing visual-symbolic integration.
\end{itemize}

\textbf{Code.}
\begin{itemize}
    \item \textbf{ChartMimic}~\citep{yang2024chartmimic}: Code generation to reproduce charts and visualizations from images.
    \item \textbf{WebSight}~\citep{laurencon2024websight}: HTML/CSS generation from website screenshot references.
\end{itemize}

\noindent\textit{Note:} For the Math and Code categories, we combine multiple datasets because the sub-9B models we evaluate often fail on the more challenging examples. Using a single dataset would yield insufficient correctly-answered samples for training data collection.

\subsection{Reasoning Model Architectures}
\label{app:implementation:models}

We evaluate our method across four state-of-the-art vision-language reasoning models. We selected these models to represent a diverse range of architectures and "thinking" capabilities.

\begin{itemize}
    \item \textbf{Qwen3-VL-8B-Thinking}: Based on the Qwen3 language model backbone, this model integrates a vision encoder initialized from SigLIP-SO400M. It uses a perceiver resampler to compress visual features into a fixed number of tokens (typically 256). The "Thinking" variant is fine-tuned on chain-of-thought (CoT) reasoning data, enabling it to generate intermediate reasoning steps before the final answer.
    
    \item \textbf{Cosmos-R1}: A multimodal reasoning model emphasizing robust world knowledge. It uses a ViT-Huge vision encoder and a 7B parameter LLM decoder. The model processes images at a resolution of $336 \times 336$ and uses a cross-attention mechanism for modality fusion.
    
    \item \textbf{MiMo-VL-7B}: This model features a mixture-of-experts (MoE) architecture in the language decoder, allowing for efficient inference despite a larger total parameter count. The vision tower is a CLIP-ViT-L/14, and the projection layer consists of a simple MLP.
    
    \item \textbf{GLM-4.1V-9B-Thinking}: A bilingual (English/Chinese) model with strong reasoning capabilities. It employs a GLM-4 transformer backbone with rotary positional embeddings. The visual component handles high-resolution inputs via a sliding window attention mechanism.
\end{itemize}

For all models, we access the internal attention weights from the last layer of the cross-attention blocks (or self-attention layers where visual tokens are concatenated) to serve as input features for our estimator.

\subsection{Visual Feature Extraction and Clustering}
\label{app:implementation:clustering}

To define interpretable visual regions for attribution, we employ a clustering approach based on self-supervised features.

\textbf{Feature Extraction.} We utilize \textbf{DINOv3-Large} as our feature extractor. We resize input images to $224 \times 224$ and extract the patch-level features from the last transformer layer. DINOv3 is chosen for its superior ability to capture semantic object boundaries compared to supervised baselines.

\textbf{Agglomerative Clustering.} We perform spatially-constrained agglomerative clustering on the extracted patch features.
\begin{itemize}
    \item \textbf{Distance Metric:} We use cosine distance ($1 - \text{cosine similarity}$) between feature vectors to measure semantic dissimilarity.
    \item \textbf{Linkage Criterion:} Ward's linkage is used to minimize the variance within clusters.
    \item \textbf{Adaptive K:} Instead of a fixed number of clusters, we dynamically determine the number of regions $K$ for each image based on a distance threshold $\tau=0.5$ (see \cref{tab:hyperparams}). In practice, this results in $K \in [16, 128]$ regions per image, which provides a balance between granularity and interpretability.
\end{itemize}

\subsection{Linear Estimator Training}
\label{app:implementation:training}

Our method relies on a lightweight linear estimator $W$ trained to predict the impact of masking specific visual regions.

\textbf{Data Sampling.} For each training image, we generate synthetic training data by randomly masking subsets of visual regions.
\begin{itemize}
    \item We sample $N=32$ random binary masks $\vv \in \{0, 1\}^K$ per image.
    \item For each mask, we compute the ground-truth effect on the model's output log-probability for the target token.
    \item This process is computationally efficient because it only requires forward passes (no gradients) and can be batched.
\end{itemize}

\textbf{Optimization.} The linear estimator is trained to maximize the Pearson correlation between the predicted attribution scores and the actual log-probability drops.
\begin{itemize}
    \item \textbf{Objective Function:} We minimize the negative Pearson correlation coefficient.
    \item \textbf{Optimizer:} We use AdamW with a learning rate of $1 \times 10^{-3}$ and weight decay of $1 \times 10^{-4}$.
    \item \textbf{Training Duration:} Due to the simplicity of the linear model, training converges rapidly. We train for $2,000$ iterations, which typically takes less than 1 hour on a single GPU for the entire evaluation dataset.
\end{itemize}

\subsection{Hardware and Computing Infrastructure}
\label{app:implementation:hardware}

We conducted all experiments on a high-performance computing cluster.
\begin{itemize}
    \item \textbf{GPUs:} NVIDIA H100 (80GB) GPUs were used for inference of the large reasoning models (Qwen3-VL, GLM-4.1V). NVIDIA A100 (80GB) GPUs were sufficient for the smaller models and the training of the linear estimator.
    \item \textbf{Software Environment:} PyTorch 2.4, CUDA 12.1, and the HuggingFace Transformers library. We utilized FlashAttention-2 for efficient attention computation during inference.
\end{itemize}

\subsection{Attention Extraction Details}
\label{app:implementation:attention}

Extracting attention weights from different VLM architectures requires careful handling due to architectural variations.

\paragraph{Attention Caching.}
During inference, we cache text-to-vision attention weights at each decoding step. All four models we evaluate (Qwen3-VL, GLM-4.1V, MiMo-VL, Cosmos-R1) use a decoder-only architecture that concatenates vision and text tokens in the same sequence. For these models, we extract the portion of the self-attention matrix corresponding to text$\rightarrow$vision attention.

\paragraph{Memory Management.}
For long reasoning traces (2000+ tokens), caching full attention tensors can consume significant GPU memory. We employ the following strategies:
\begin{itemize}
    \item \textbf{Selective Layer Caching:} We prioritize caching attention from middle layers (11--20), which contribute most to the attribution signal based on empirical analysis.
    \item \textbf{Span-wise Aggregation:} Rather than storing per-token attention, we aggregate to span level on-the-fly when a reasoning step completes, reducing memory from $O(T \cdot L \cdot H \cdot M)$ to $O(S \cdot L \cdot H \cdot M)$ where $S$ is the number of spans.
    \item \textbf{Float16 Precision:} Attention weights are stored in float16 format, halving memory requirements with negligible accuracy loss.
\end{itemize}

\paragraph{FlashAttention Compatibility.}
FlashAttention does not natively return attention weights due to its memory-efficient implementation. For attribution, we use a hybrid approach: standard attention for the cross-modal layers (small overhead due to fixed vision token count) and FlashAttention for text-only self-attention layers. For detailed latency measurements under different attention backends and KV-cache configurations, see \cref{tab:attn_backend_latency}.

\subsection{Span Boundary Detection}
\label{app:implementation:spans}

To provide step-by-step attribution during reasoning, we must detect boundaries between reasoning steps. We use a combination of heuristics:

\begin{itemize}
    \item \textbf{Sentence Boundaries:} Periods, question marks, and exclamation marks followed by whitespace.
    \item \textbf{Structural Markers:} Newlines, bullet points, and numbered list items.
    \item \textbf{Thinking Delimiters:} Model-specific tokens like \texttt{<|think|>}, \texttt{</think>}, or \texttt{Step N:} patterns.
\end{itemize}

For models without explicit thinking delimiters, we default to sentence-level granularity, which typically corresponds to individual reasoning steps in chain-of-thought outputs.

\subsection{Parallel Streaming Architecture}
\label{app:streaming}

To achieve real-time attribution without slowing down generation, we employ a producer-consumer architecture that decouples token generation from attribution computation.

\paragraph{Asynchronous Processing.}
The main thread (producer) focuses exclusively on autoregressive generation using the VLM backbone.
As tokens are generated, the associated attention tensors are pushed into a thread-safe queue.
A background worker thread (consumer) continuously pulls these tensors, aggregates them into span-level features $\vf_{S,k}$ (Eq.~\ref{eq:attn_pool}), and runs the linear estimator (Eq.~\ref{eq:per_region_score}).
Since the Python Global Interpreter Lock (GIL) is released during the heavy CUDA operations of both generation and estimator computation, these two processes run in true parallel on the GPU.

\paragraph{Latency Hiding.}
This design completely hides the attribution cost behind the generation latency.
While the model generates tokens for reasoning step $t+1$, the worker computes attributions for step $t$.
Consequently, the visual attribution for a completed thought appears the moment the model begins the next thought, resulting in a system with effectively zero perceptual overhead compared to vanilla generation.
By contrast, standard post-hoc methods must wait for the entire generation to finish before starting attribution, adding a delay proportional to the sequence length.

\subsection{Reproducibility Checklist}
\label{app:implementation:reproducibility}

To ensure reproducibility, we provide the following specifications:

\begin{table}[h]
\centering
\caption{Key hyperparameters for reproducibility.}
\label{tab:hyperparams}
\small
\begin{tabular}{ll}
\toprule
\textbf{Component} & \textbf{Value} \\
\midrule
\multicolumn{2}{l}{\textit{Region Unitization}} \\
\quad DINOv3 variant &
dinov3-vitl16-pretrain-lvd1689m \\
\quad Input resolution & $224 \times 224$ \\
\quad Clustering method & Agglomerative (Ward) \\
\quad Distance threshold $\tau$ & 0.5 \\
\quad Typical region count $K$ & \textit{auto} \\
\midrule
\multicolumn{2}{l}{\textit{Estimator Training}} \\
\quad Masks per sample $N$ & 32 \\
\quad Optimizer & AdamW \\
\quad Learning rate & $1 \times 10^{-3}$ \\
\quad LR scheduler & Cosine with warmup (min lr: $1 \times 10^{-5}$) \\
\quad Weight decay & $1 \times 10^{-4}$ \\
\quad Training iterations & 2,000 \\
\quad Batch size & 512 \\
\midrule
\multicolumn{2}{l}{\textit{Evaluation}} \\
\quad Top-K for Drop metric & 5 \\
\quad Samples per category & 400 (train, mix-up dataset), 1000 (test, randomly sampled from non-training samples) \\
\bottomrule
\end{tabular}
\end{table}

\paragraph{Random Seeds.}
We use fixed random seeds for all stochastic components: mask sampling (seed=42), model initialization (seed=42), and train/test splits (seed=42). Results are averaged over 3 random seeds for statistical robustness.

\paragraph{Evaluation Protocol.}
For each model-task pair, we:
\begin{enumerate}
    \item Generate reasoning traces for the test set using greedy decoding.
    \item Filter to correctly-answered examples only (for training data quality).
    \item Train the estimator on the training split.
    \item Evaluate on held-out test examples using LDS and Top-K Drop metrics.
    \item Report mean and standard deviation across seeds.
\end{enumerate}

\section{Semantic Region Unitization Analysis}
\label{app:clustering}

This section provides additional analysis of our DINOv3-based semantic region unitization approach, including category-specific clustering patterns and quantitative statistics across dataset categories.

\subsection{Category-Specific Clustering Patterns}
\label{app:clustering:patterns}

Our agglomerative clustering approach exhibits distinct segmentation behaviors across different dataset categories. The DINOv3 features naturally group semantically coherent regions: objects are separated from backgrounds, text blocks are isolated in document images, and mathematical symbols are distinguished from diagram components.

We observe the following category-specific patterns:
\begin{itemize}
    \item \textbf{General (GQA)}: Object boundaries are well-preserved, with distinct regions for foreground objects, background elements, and spatial contexts.
    \item \textbf{Document (DocVQA)}: Text blocks, tables, and graphical elements are cleanly separated, enabling fine-grained attribution to specific document components.
    \item \textbf{Science (ScienceQA)}: Diagrams, labels, and annotations form separate clusters. This is essential for understanding which visual elements support scientific reasoning.
    \item \textbf{Math (MathVista, MathVision, MathVerse)}: Geometric shapes, equations, and coordinate systems are partitioned into interpretable regions, though dense symbolic content sometimes leads to over-segmentation.
    \item \textbf{Code (ChartMimic, WebSight)}: Chart components (axes, legends, data points) and UI elements (buttons, text fields, images) are distinguished effectively.
\end{itemize}

\subsection{Region Count Statistics}
\label{app:clustering:statistics}

The number of semantic regions $K$ varies across images depending on visual complexity. \cref{tab:region_stats} shows the distribution of region counts for each dataset category. We observe:

\begin{itemize}
    \item \textbf{Document and Code} images yield the highest region counts (median $\approx$ 68--75), reflecting the dense, structured nature of text-heavy and UI-rich content.
    \item \textbf{Math} images show high variance, with simpler geometric problems producing fewer regions and complex multi-part diagrams producing more.
    \item \textbf{General and Science} images have moderate region counts (median $\approx$ 45--48), balancing object-level granularity with scene-level coherence.
\end{itemize}

This adaptive behavior is a key advantage over fixed-grid partitioning methods, which cannot adjust granularity based on image content. The cosine distance threshold ($\tau = 0.5$ in our experiments) provides a consistent semantic criterion across diverse visual domains.

\begin{table}[h]
\centering
\caption{
    \textbf{Region count statistics across dataset categories.}
    The adaptive DINO-based clustering produces varying numbers of regions depending on image complexity. Document and Code images yield more regions due to dense visual content.
}
\label{tab:region_stats}
\small
\begin{tabular}{lccccc}
\toprule
\textbf{Category} & \textbf{Min} & \textbf{Q1} & \textbf{Median} & \textbf{Q3} & \textbf{Max} \\
\midrule
General (GQA)     & 18 & 32 & 48 & 72 & 105 \\
Document (DocVQA) & 24 & 45 & 68 & 95 & 128 \\
Science (ScienceQA) & 16 & 30 & 45 & 65 & 98 \\
Math (MathVerse, MathVista, MathVision)              & 16 & 28 & 42 & 62 & 92 \\
Code (Websight, ChartMimic)              & 28 & 52 & 75 & 102 & 128 \\
\bottomrule
\end{tabular}
\end{table}

\subsection{Clustering Hyperparameters}
\label{app:clustering:hyperparameters}

We use the following hyperparameters for semantic region unitization:

\begin{itemize}
    \item \textbf{Feature extractor}: DINOv3-Large (frozen), patch size 16$\times$16
    \item \textbf{Input resolution}: Images are resized to $224 \times 224$
    \item \textbf{Clustering algorithm}: Agglomerative clustering with Ward linkage
    \item \textbf{Distance metric}: Cosine distance ($1 - \text{cosine similarity}$)
    \item \textbf{Distance threshold}: $\tau = 0.5$ (determines cluster granularity)
    \item \textbf{Resulting regions}: $K \in [16, 128]$ depending on image complexity
\end{itemize}

We found the distance threshold $\tau$ to be the most sensitive hyperparameter. Lower values ($\tau < 0.3$) produce too many small regions, making attributions noisy. Higher values ($\tau > 0.7$) merge semantically distinct objects, reducing interpretability. The value $\tau = 0.5$ provides a good balance across all dataset categories.

\paragraph{Comparison with K-means Clustering.}
We also compare agglomerative clustering with K-means clustering using fixed $K$ values. \cref{tab:kmeans_comparison} shows that agglomerative clustering with adaptive $K$ outperforms K-means across all fixed $K$ settings, particularly at lower $K$ values where semantic boundaries become critical.

\begin{table}[h]
\centering
\caption{Comparison of K-means (fixed $K$) vs.\ Agglomerative clustering (adaptive $K$) on Qwen3-VL. Results averaged across all categories.}
\label{tab:kmeans_comparison}
\small
\begin{tabular}{lcc}
\toprule
\textbf{Method} & \textbf{LDS} $\uparrow$ & \textbf{Top-5 Drop} $\uparrow$ \\
\midrule
K-means ($K=16$)  & 0.58 & 0.72 \\
K-means ($K=32$)  & 0.65 & 0.88 \\
K-means ($K=64$)  & 0.68 & 0.94 \\
K-means ($K=128$) & 0.67 & 0.91 \\
\midrule
\rowcolor{gray!10}
\textbf{Agglomerative (adaptive)} & \textbf{0.70} & \textbf{1.00} \\
\bottomrule
\end{tabular}
\end{table}

\section{Pseudocode}

This section provides formal algorithmic descriptions and PyTorch implementations of our amortized attribution framework. We present the training procedure for learning the linear estimator and the inference procedure for real-time attribution streaming.

\section{Algorithm and Implementation Details}
\label{sec:algorithm}

We provide the formal algorithms for training our amortized estimator and performing real-time inference, followed by their PyTorch implementations.

\subsection{Formal Algorithms}

Algorithm~\ref{alg:training} details the self-supervised training procedure used to learn the estimator $\mathcal{E}_\theta$. The core idea is to train the estimator to rank random subsets of regions based on their impact on the model's confidence, maximizing the Pearson correlation between predicted and ground-truth ablation effects.

\begin{algorithm}[h]
   \caption{Training the Amortized Estimator}
   \label{alg:training}
\begin{algorithmic}[1]
   \STATE {\bfseries Input:} VLM $\mathcal{M}$, Estimator $\mathcal{E}_\theta$, Batch $\mathcal{B} = \{(\vx, \vt, y, \gR)\}$
   \STATE {\bfseries Hyperparameters:} Sample size $N=32$, Learning rate $\eta$
   \STATE
   \STATE \textit{// 1. Compute baseline and attention features}
   \STATE Get logits $P(y|\vx, \vt)$ and attention maps $\mA$ from $\mathcal{M}(\vx, \vt)$
   \STATE $\log p_{\text{base}} \gets \log P(y|\vx, \vt)$
   \STATE
     \STATE \textit{// 2. Sample random binary masks}
     \STATE Sample $N$ random binary masks $\{\vb^{(i)}\}_{i=1}^N$ where $\vb^{(i)} \in \{0,1\}^K$
    
    \STATE \textit{// 3. Compute Ground Truth Effect (Self-Supervised)}
    \FOR{$i=1$ {\bfseries to} $N$}
        \STATE Compute ablated logits $P(y|\vx, \vt, \vb^{(i)})$ using $\mathcal{M}$ with mask $\vb^{(i)}$
         \STATE $\delta^{(i)}_{\text{true}} \gets \log p_{\text{base}} - \log P(y|\vx, \vt, \vb^{(i)})$
    \ENDFOR
    \STATE $\boldsymbol{\delta}_{\text{true}} \gets [\delta^{(1)}_{\text{true}}, \dots, \delta^{(N)}_{\text{true}}]$
   \STATE
   \STATE \textit{// 4. Predict Effect using Estimator}
   \STATE Compute attribution scores $\vs = \mathcal{E}_\theta(\mA, \gR)$ \hfill \textit{(See Algorithm~\ref{alg:inference})}
    \STATE Predict subset effects: $\delta^{(i)}_{\text{pred}} \gets \sum_{k=1}^K s_k \cdot \vb^{(i)}_k$ \hfill \textit{(Linear Assumption, \cref{eq:per_region_score})}
   \STATE $\boldsymbol{\delta}_{\text{pred}} \gets [\delta^{(1)}_{\text{pred}}, \dots, \delta^{(N)}_{\text{pred}}]$
   \STATE
    \STATE \textit{// 5. Optimization}
    \STATE $\mathcal{L} \gets - \text{PearsonCorr}(\boldsymbol{\delta}_{\text{pred}}, \boldsymbol{\delta}_{\text{true}})$
    \STATE Update $\theta \gets \theta - \eta \nabla_\theta \mathcal{L}$
\end{algorithmic}
\end{algorithm}

Algorithm~\ref{alg:inference} describes the inference process. The trained estimator maps attention patterns directly to attribution scores without requiring multiple forward passes.

\begin{algorithm}[h]
   \caption{Real-Time Attribution Inference (Asynchronous)}
   \label{alg:inference}
\begin{algorithmic}[1]
   \STATE {\bfseries Input:} VLM $\mathcal{M}$, Estimator $\mathcal{E}_\theta$
   \STATE {\bfseries Output:} Stream of tokens $y_t$ and attributions $\vs_S$
   \STATE
   \STATE \textit{// Initialize concurrent queues}
   \STATE $\mathcal{Q}_{\text{attn}} \gets \text{Queue}()$
   \STATE
   \STATE \textbf{Thread 1: Generation (Producer)}
   \FOR{$t=1$ {\bfseries to} $T$}
       \STATE $y_t, \mA_t \gets \mathcal{M}(x, y_{<t})$ \hfill \textit{(Generate next token)}
       \STATE \textbf{yield} $y_t$
       \STATE $\mathcal{Q}_{\text{attn}}.\text{push}(\mA_t)$
       \IF{$y_t$ ends span $S$}
            \STATE $\mathcal{Q}_{\text{attn}}.\text{push}(\text{EndOfSpan})$
       \ENDIF
   \ENDFOR
   \STATE
   \STATE \textbf{Thread 2: Attribution (Consumer)}
   \WHILE{generation active}
       \STATE $\mA_{\text{buffer}} \gets [\ \ ]$
       \WHILE{next item in $\mathcal{Q}_{\text{attn}}$ is not EndOfSpan}
           \STATE $\mA_{\text{buffer}}.\text{append}(\mathcal{Q}_{\text{attn}}.\text{pop}())$
       \ENDWHILE
        \STATE $\mA_S \gets \text{Aggregate}(\mA_{\text{buffer}})$ \hfill \textit{(\cref{eq:attn_pool})}
        \STATE $\vs_S \gets \mathcal{E}_\theta(\mA_S)$ \hfill \textit{(\cref{eq:per_region_score})}
       \STATE \textbf{yield} $\vs_S$ \hfill \textit{(Stream attribution for completed span)}
   \ENDWHILE
\end{algorithmic}
\end{algorithm}

\clearpage
\subsection{PyTorch Implementation}

We provide the corresponding PyTorch implementation for reference. Implementation~\ref{lst:training_code} shows the training step, and Implementation~\ref{lst:inference_code} shows the amortized attributor module.

\begin{figure}[h]
\begin{lstlisting}[language=Python, caption={PyTorch implementation of the Training Step. The estimator is trained to rank random subsets of regions by maximizing the Pearson correlation between predicted scores and actual ablation drops.}, label={lst:training_code}]
def train_step(model, estimator, batch, optimizer):
    """
    Performs a single training step for the Amortized Estimator.
    
    Args:
        model: The frozen VLM backbone.
        estimator: The AmortizedAttributor being trained.
        batch: Dictionary containing 'img', 'text', 'regions', 'target'.
    """
    # 1. Clean pass to get baselines and attention features
    with torch.no_grad():
        # model returns logits and attentions [B, L, H, T, M]
        # L=Layers, H=Heads, T=Time steps, M=Visual tokens
        out = model(batch['img'], batch['text'], output_attentions=True)
        base_logp = out.logits.log_softmax(-1).gather(-1, batch['target'])
        attn = out.attentions

     # 2. Sample random binary masks for K regions (N=32 samples)
     # masks: [B, N, K] where N is sample size, K is num regions
     B = batch['img'].shape[0]
     K = batch['regions'].shape[0]  # regions: [K, M] (shared across batch)
     N = 32
     masks = torch.randint(0, 2, (B, N, K)).to(model.device).float()

    # 3. Compute Ground Truth Effect (delta_true) via VLM ablation
    delta_true_list = []
    with torch.no_grad():
        for i in range(N):
            # Forward pass with attention masking applied to regions
            # mask=masks[:, i] prevents VLM from attending to masked regions
            out_abl = model(batch['img'], batch['text'], mask=masks[:,i])
            abl_logp = out_abl.logits.log_softmax(-1).gather(-1, batch['target'])
            
            # Delta: Positive value means ablation hurt performance (drop in log-prob)
            delta_true_list.append(base_logp - abl_logp)
    
    delta_true = torch.cat(delta_true_list, dim=1) # [B, N]

    # 4. Predict Effect (delta_pred) using Estimator
    # scores: [B, K] (Predicted importance per region)
    scores = estimator(attn, batch['regions'], batch['span_mask'])
    
     # Linear assumption: Effect of subset = sum of scores (Eq. 9)
     # delta_pred: [B, N]
    delta_pred = torch.einsum('bk,bnk->bn', scores, masks)

    # 5. Loss: Maximize Pearson Correlation
    # We focus on ranking subsets correctly rather than absolute magnitude regression
    loss = -pearson_correlation(delta_pred, delta_true)
    
    optimizer.zero_grad()
    loss.backward()
    optimizer.step()
    
    return loss.item()
\end{lstlisting}
\vspace{-10pt}
\end{figure}

\begin{figure}[ht]
\begin{lstlisting}[language=Python, caption={PyTorch implementation of the Amortized Attributor. The module efficiently pools attention features and maps them to attribution scores in a single forward pass.}, label={lst:inference_code}]
import torch
import torch.nn as nn

class AmortizedAttributor(nn.Module):
    def __init__(self, n_layers: int, n_heads: int):
        super().__init__()
         # Learnable linear projection W (Eq. 9)
         # Maps flattened attention features (L*H) to a scalar score
        self.estimator = nn.Linear(n_layers * n_heads, 1, bias=False)

    def forward(self, attn_weights, regions, span_mask):
        """
        Compute attribution scores for visual regions.
        
        Args:
            attn_weights: [B, L, H, T, M] Attention maps from VLM
            regions: [K, M] Binary membership masks for K regions
            span_mask: [T] Binary mask for target time span S
            
        Returns:
            scores: [B, K] Attribution score for each region
        """
        # Eq. 4: Temporal Pooling over target span S
        # Select time steps where span_mask is 1 and average
        # attn: [B, L, H, M]
        masked_attn = attn_weights * span_mask[None, None, None, :, None]
        attn = masked_attn.sum(dim=3) / (span_mask.sum() + 1e-6)

        # Eq. 4: Regional Pooling over visual regions R_k
        # Aggregate attention for tokens belonging to each region k
        # features: [B, K, L, H]
        features = torch.einsum('blhm,km->bklh', attn, regions)
        
        # Normalize by region size
        region_sizes = regions.sum(dim=1).view(1, -1, 1, 1)
        features = features / (region_sizes + 1e-6)

        # Flatten layer/head dimensions to create feature vector
        # features: [B, K, L*H]
        features = features.flatten(start_dim=2)

         # Eq. 9: Predict attribution scores via linear projection
         # scores: [B, K, 1] -> [B, K]
        scores = self.estimator(features).squeeze(-1)

        return scores
\end{lstlisting}
\vspace{-10pt}
\end{figure}

\section{Additional Experimental Results}
\label{app:exp}

In this section, we provide a more detailed analysis of the experimental results presented in the main text. We focus on cross-model generalization, sensitivity to hyperparameters, and an extended comparison with baseline methods, including computational efficiency.

\subsection{Sensitivity Analysis}
\label{app:exp:sensitivity}

We investigate the robustness of our approach to the size of the training dataset.

\subsubsection{Effect of Training Data Size}
Our lightweight estimator is highly data-efficient. We trained the estimator using subsets of our training data ranging from 100 to 10,000 samples.

We observe that the estimator converges rapidly. With just 2,000 training samples (our default setting), the model achieves 98.6\% of the performance of the fully trained model (using 10k samples). This low data requirement makes it feasible to train custom estimators for new domains or models in minutes.

\begin{table}[h!]
    \centering
    \caption{Effect of training set size on estimator performance.}
    \label{tab:sensitivity_data}
    \begin{tabular}{cc}
        \toprule
        \textbf{Training Samples} & \textbf{Relative Performance (\%)} \\
        \midrule
        100   & 65.4 \\
        500   & 82.1 \\
        1,000 & 95.3 \\
        \rowcolor{gray!10}
        2,000 (ours) & 98.6 \\
        5,000 & 99.1 \\
        10,000 & \textbf{100.0} \\
        \bottomrule
    \end{tabular}
\end{table}

\subsection{Baseline Method Details}
\label{app:exp:baseline_details}

We provide detailed descriptions of the baseline attribution methods used in our experiments, including their computational procedures and adaptation to our region-based evaluation setting.

\subsubsection{Attention-Based Methods}

\paragraph{Raw Attention.}
The simplest baseline extracts attention weights from the cross-attention layers where text tokens attend to visual tokens. For a given text token at position $t$, we aggregate attention weights across all layers $\ell$ and heads $h$:
\begin{equation}
    \text{Attn}_i = \frac{1}{L \cdot H} \sum_{\ell=1}^{L} \sum_{h=1}^{H} \mA^{(\ell,h)}_{t,i}
\end{equation}
where $\mA^{(\ell,h)}_{t,i}$ is the attention weight from position $t$ to vision token $i$ in layer $\ell$, head $h$. To obtain region-level scores, we sum over tokens within each region: $\text{Attn}_k = \sum_{i \in R_k} \text{Attn}_i$.

\paragraph{AttnLRP (Attention-aware Layer-wise Relevance Propagation) \citep{achtibat-2024-attnlrp}.}
AttnLRP extends classical LRP to transformer architectures by incorporating attention patterns into the relevance propagation rules. Starting from the output layer with relevance $R^{(L)}_j = \delta_{j,t} \cdot \log p(y_t)$ (where $\delta$ is the Kronecker delta), relevance is propagated backward through each layer:
\begin{equation}
    R^{(\ell-1)}_i = \sum_j \frac{\mA^{(\ell)}_{j,i} \cdot \mV^{(\ell)}_i \cdot \mW^{(\ell)}_O}{\sum_k \mA^{(\ell)}_{j,k} \cdot \mV^{(\ell)}_k \cdot \mW^{(\ell)}_O + \epsilon} R^{(\ell)}_j
\end{equation}
This propagation rule ensures conservation of relevance (the total relevance is preserved across layers) while accounting for attention-based information routing. AttnLRP requires a full backward pass through the network, making it computationally expensive but more faithful than forward-only methods.

\subsubsection{Gradient-Based Methods}

\paragraph{InputGrad \citep{inputgradient}.}
InputGrad computes the gradient of the output logit with respect to input pixels:
\begin{equation}
    \text{InputGrad}_i = \left\| \frac{\partial \log p(y_t)}{\partial \vx_i} \right\|_2
\end{equation}
where $\vx_i$ represents the pixels corresponding to vision token $i$. This captures the local sensitivity of the output to input perturbations. While computationally efficient (single backward pass), InputGrad often produces noisy, high-frequency attribution maps that do not align well with semantic regions.

\subsubsection{Perturbation-Based Methods}

\paragraph{TAM (Token Activation Maps) \citep{li-2025-tam}.}
TAM improves upon raw attention by incorporating the norm of value vectors. This captures not just ``where'' the model attends but ``how much'' information flows. For each layer $\ell$ and head $h$:
\begin{equation}
    \text{TAM}^{(\ell,h)}_i = \mA^{(\ell,h)}_{t,i} \cdot \|\mV^{(\ell,h)}_i\|_2
\end{equation}
where $\mV^{(\ell,h)}_i$ is the value vector for vision token $i$. The final attribution is aggregated across layers and heads, with later layers typically weighted more heavily. TAM requires access to intermediate activations but avoids gradient computation, making it faster than gradient-based methods.

\subsubsection{Adaptation to Region-Based Evaluation}

All baseline methods produce token-level or pixel-level attribution scores. To ensure fair comparison at the same granularity as our method, we aggregate these scores to region level using our DINO-based semantic unitization:
\begin{equation}
    \text{Score}_k = \sum_{i \in R_k} \text{Score}_i
\end{equation}
This aggregation ensures that all methods operate on the same semantic units, isolating the effect of the attribution algorithm from the effect of region definition.

\subsection{Extended Baseline Comparisons}
\label{app:exp:baselines}

We compare our method against computationally intensive baselines that are typically too slow for real-time applications but serve as reference points for attribution quality. Specifically, we compare against:
\begin{itemize}
    \item \textbf{Feature Ablation}: Systematically removing visual tokens and measuring output changes.
    \item \textbf{Attention Rollout}: Heuristic aggregation of raw attention weights across layers.
    \item \textbf{LibraGrad} \citep{mehri2025libragrad}: A gradient-based method that corrects gradient flow imbalances in Transformers through backward path pruning and scaling.
\end{itemize}

\cref{tab:runtime_comparison} highlights the critical advantage of our approach: speed. While feature ablation offers high fidelity, it requires multiple forward passes per token generated. Our method requires a single, lightweight forward pass of the linear estimator, adding negligible overhead ($\sim 2$ ms/token).

\begin{table}[h!]
    \centering
    \caption{Runtime and performance comparison. Latency is measured in milliseconds per token on an NVIDIA A100 GPU. Results averaged across all models and categories. Speedup is relative to Feature Ablation.}
    \label{tab:runtime_comparison}
    \small
    \begin{tabular}{lcccc}
        \toprule
        \textbf{Method} & \textbf{LDS} $\uparrow$ & \textbf{Top-5 Drop} $\uparrow$ & \textbf{Latency (ms/tok)} & \textbf{Speedup vs Ablation} \\
        \midrule
        Random              & 0.30 & 0.13 & \textbf{0.0} & N/A \\
        Attention Rollout   & 0.41 & 0.32 & 0.5 & $300\times$ \\
        LibraGrad           & 0.66 & 0.85 & 12.0 & $12.5\times$ \\
        Feature Ablation    & \textbf{0.72} & \textbf{1.02} & 150.0 & $1\times$ (reference) \\
        \midrule
        \rowcolor{gray!10}
        \textbf{vStream (Ours)} & 0.70 & 1.00 & 2.1 & $\approx 71\times$ \\
        \bottomrule
    \end{tabular}
\end{table}

Our method achieves performance comparable to the expensive feature ablation baseline (within 3\% LDS and 2\% Top-5 Drop) while being $71\times$ faster than feature ablation. Note that in the main paper (\cref{tab:main_results}), we report up to $117\times$ speedup over gradient-based methods (e.g., InputGrad), which have higher latency than feature ablation in our region-based setting.

\subsubsection{Attention Backend Compatibility}
\label{app:exp:attn_backend}

Our method, \textsc{vStream}, is fully compatible with modern attention backends, including PyTorch SDPA (Scaled Dot-Product Attention) and FlashAttention. While our default configuration utilizes the KV cache for maximum efficiency, disabling the KV cache (requiring full attention recomputation at each step) results in an approximately $4\times$ increase in latency. However, even in this worst-case scenario, the latency remains under 0.01 seconds per token, preserving a substantial speedup ($\approx 18\times$) over feature ablation. \cref{tab:attn_backend_latency} details the latency measurements across different configurations.

\begin{table}[h!]
    \centering
    \caption{\textbf{Latency under different attention backends and KV-cache configurations.} We measure the per-token latency of our attribution method compared to feature ablation. Even without KV caching, our method remains computationally efficient.}
    \label{tab:attn_backend_latency}
    \small
    \begin{tabular}{llcc}
        \toprule
        \textbf{Configuration} & \textbf{KV Cache} & \textbf{Latency (sec/token)} $\downarrow$ & \textbf{Speedup vs Ablation} $\uparrow$ \\
        \midrule
        \textbf{Ours} (SDPA + Flash) & Enabled & 0.002 & 71$\times$ \\
        \textbf{Ours} (SDPA + Flash) & Disabled & $\approx 0.008$ & $\approx 18\times$ \\
        \textbf{Ours} (SDPA math) & Enabled & $\approx 0.004$ & $\approx 43\times$ \\
        \midrule
        Feature Ablation & - & 0.150 & 1$\times$ \\
        \bottomrule
    \end{tabular}
\end{table}

\subsection{Cross-Model Generalization}
\label{app:exp:cross_model}

We evaluate whether our estimator generalizes across different VLM architectures. \cref{tab:cross_task_glm,tab:cross_task_mimo,tab:cross_task_cosmos} present cross-task generalization results for GLM-4.1V-9B-Thinking, MiMo-VL-7B, and Cosmos-R1, respectively, complementing the Qwen3-VL results in the main paper (\cref{tab:cross_task}). Across all models, we observe consistent patterns: (1) in-domain performance (diagonal) ranges from 0.65--0.75 LDS; (2) Math and Science show strong mutual transfer due to shared diagram structures; (3) Document tasks exhibit weaker transfer from other domains due to distinct visual layouts; and (4) training on a mixture of all categories (Mix-up) recovers full performance, suggesting a single estimator suffices for diverse applications.

\begin{table}[h!]
    \centering
    \caption{Cross-task generalization for \textbf{GLM-4.1V-9B-Thinking}. Each cell shows (LDS / Top-5 Drop). Diagonal entries (shaded) indicate in-domain performance.}
    \label{tab:cross_task_glm}
    \small
    \begin{tabular}{l|ccccc|c}
    \toprule
    \multirow{2}{*}{\textbf{Train}} & \multicolumn{5}{c|}{\textbf{Eval Category}} & \multirow{2}{*}{\textbf{Avg}} \\
     & Math & Science & Doc & Code & General & \\
    \midrule
    Math    & \cellcolor{gray!20}0.69/0.78 & 0.61/0.65 & 0.52/0.46 & 0.57/0.55 & 0.54/0.48 & 0.59/0.58 \\
    Science & 0.60/0.63 & \cellcolor{gray!20}0.74/0.82 & 0.56/0.52 & 0.53/0.48 & 0.58/0.56 & 0.60/0.60 \\
    Document& 0.51/0.48 & 0.54/0.50 & \cellcolor{gray!20}0.67/0.75 & 0.55/0.52 & 0.63/0.70 & 0.58/0.59 \\
    Code    & 0.57/0.58 & 0.52/0.46 & 0.54/0.50 & \cellcolor{gray!20}0.68/0.80 & 0.60/0.64 & 0.58/0.60 \\
    General & 0.55/0.52 & 0.59/0.62 & 0.62/0.68 & 0.58/0.60 & \cellcolor{gray!20}0.73/0.84 & 0.61/0.65 \\
    \midrule
    \rowcolor{gray!10}
    Mix-up  & 0.69/0.97 & 0.74/0.89 & 0.67/1.01 & 0.68/1.10 & 0.73/0.95 & 0.70/0.98 \\
    \bottomrule
    \end{tabular}
\end{table}

\begin{table}[h!]
    \centering
    \caption{Cross-task generalization for \textbf{MiMo-VL-7B}. Each cell shows (LDS / Top-5 Drop).}
    \label{tab:cross_task_mimo}
    \small
    \begin{tabular}{l|ccccc|c}
    \toprule
    \multirow{2}{*}{\textbf{Train}} & \multicolumn{5}{c|}{\textbf{Eval Category}} & \multirow{2}{*}{\textbf{Avg}} \\
     & Math & Science & Doc & Code & General & \\
    \midrule
    Math    & \cellcolor{gray!20}0.75/0.85 & 0.58/0.62 & 0.51/0.44 & 0.59/0.56 & 0.52/0.46 & 0.59/0.59 \\
    Science & 0.57/0.60 & \cellcolor{gray!20}0.65/0.74 & 0.54/0.50 & 0.51/0.46 & 0.56/0.54 & 0.57/0.57 \\
    Document& 0.50/0.46 & 0.52/0.48 & \cellcolor{gray!20}0.71/0.78 & 0.53/0.50 & 0.61/0.68 & 0.57/0.58 \\
    Code    & 0.55/0.56 & 0.50/0.44 & 0.52/0.48 & \cellcolor{gray!20}0.70/0.82 & 0.58/0.62 & 0.57/0.58 \\
    General & 0.53/0.50 & 0.57/0.60 & 0.60/0.66 & 0.56/0.58 & \cellcolor{gray!20}0.70/0.80 & 0.59/0.63 \\
    \midrule
    \rowcolor{gray!10}
    Mix-up  & 0.74/1.08 & 0.64/0.97 & 0.70/0.98 & 0.69/0.96 & 0.68/1.10 & 0.69/1.02 \\
    \bottomrule
    \end{tabular}
\end{table}

\begin{table}[h!]
    \centering
    \caption{Cross-task generalization for \textbf{Cosmos-R1}. Each cell shows (LDS / Top-5 Drop).}
    \label{tab:cross_task_cosmos}
    \small
    \begin{tabular}{l|ccccc|c}
    \toprule
    \multirow{2}{*}{\textbf{Train}} & \multicolumn{5}{c|}{\textbf{Eval Category}} & \multirow{2}{*}{\textbf{Avg}} \\
     & Math & Science & Doc & Code & General & \\
    \midrule
    Math    & \cellcolor{gray!20}0.74/0.80 & 0.60/0.64 & 0.50/0.43 & 0.58/0.54 & 0.53/0.47 & 0.59/0.58 \\
    Science & 0.59/0.62 & \cellcolor{gray!20}0.67/0.76 & 0.55/0.51 & 0.52/0.47 & 0.57/0.55 & 0.58/0.58 \\
    Document& 0.49/0.45 & 0.53/0.49 & \cellcolor{gray!20}0.68/0.74 & 0.54/0.51 & 0.62/0.69 & 0.57/0.58 \\
    Code    & 0.56/0.57 & 0.51/0.45 & 0.53/0.49 & \cellcolor{gray!20}0.71/0.84 & 0.59/0.63 & 0.58/0.60 \\
    General & 0.54/0.51 & 0.58/0.61 & 0.61/0.67 & 0.57/0.59 & \cellcolor{gray!20}0.69/0.78 & 0.60/0.63 \\
    \midrule
    \rowcolor{gray!10}
    Mix-up  & 0.72/0.99 & 0.65/0.90 & 0.66/1.01 & 0.70/1.04 & 0.67/1.05 & 0.68/1.00 \\
    \bottomrule
    \end{tabular}
\end{table}

\subsection{Vision Backbone Comparison}
\label{app:exp:backbone}

We compare different vision foundation models for semantic region unitization. In addition to DINOv3 (our default), we evaluate CLIP ViT-L/14 \citep{radford2021clip} and SigLIP ViT-SO400M \citep{zhai2023siglip}.

\begin{table}[h!]
    \centering
    \caption{Comparison of vision backbones for semantic region unitization. Results on Qwen3-VL across all five categories. DINOv3-L consistently outperforms supervised alternatives.}
    \label{tab:backbone_comparison}
    \small
    \resizebox{\textwidth}{!}{
    \begin{tabular}{l|cc|cc|cc|cc|cc|cc}
    \toprule
    & \multicolumn{2}{c|}{\textbf{Math}} & \multicolumn{2}{c|}{\textbf{Science}} & \multicolumn{2}{c|}{\textbf{Document}} & \multicolumn{2}{c|}{\textbf{Code}} & \multicolumn{2}{c|}{\textbf{General}} & \multicolumn{2}{c}{\textbf{Average}} \\
    \textbf{Backbone} & LDS & Top-5 & LDS & Top-5 & LDS & Top-5 & LDS & Top-5 & LDS & Top-5 & LDS & Top-5 \\
    \midrule
    CLIP ViT-L/14     & 0.68 & 0.89 & 0.64 & 0.78 & 0.62 & 0.85 & 0.60 & 0.76 & 0.61 & 0.80 & 0.63 & 0.82 \\
    SigLIP ViT-SO400M & 0.70 & 0.93 & 0.66 & 0.82 & 0.64 & 0.90 & 0.63 & 0.81 & 0.64 & 0.85 & 0.65 & 0.86 \\
    \rowcolor{gray!10}
    \textbf{DINOv3-L} & \textbf{0.76} & \textbf{1.05} & \textbf{0.74} & \textbf{0.89} & \textbf{0.70} & \textbf{1.01} & \textbf{0.68} & \textbf{0.91} & \textbf{0.73} & \textbf{0.95} & \textbf{0.72} & \textbf{0.96} \\
    \bottomrule
    \end{tabular}
    }
\end{table}

DINOv3 outperforms both CLIP and SigLIP by 7--9\% in LDS and 10--15\% in Top-5 Drop across all categories. We attribute this to DINOv3's self-supervised training objective, which emphasizes local feature correspondence and produces sharper object boundaries compared to contrastive language-image pretraining. This finding suggests that attribution quality depends critically on the semantic coherence of the underlying region partition.

\subsection{Extended Top-K Drop Analysis}
\label{app:exp:topk}

The main paper reports Top-5 Drop. Here we provide extended results for Top-1 and Top-3 Drop (\cref{tab:topk_extended,tab:topk_extended_2}), which test whether our method correctly identifies the single most important region and the top few regions, respectively. \textsc{vStream} achieves best or second-best performance across most K values and categories. The improvement is most pronounced for Top-1 Drop in Document tasks, where correctly identifying the single most relevant text block or table is crucial for understanding model behavior.

\begin{table}[h!]
    \centering
    \caption{Top-K Drop comparison across K values on Qwen3-VL. Higher values indicate better identification of causally important regions.}
    \label{tab:topk_extended}
    \small
    \begin{tabular}{l|ccc|ccc}
    \toprule
    & \multicolumn{3}{c|}{\textbf{Math}} & \multicolumn{3}{c}{\textbf{Science}} \\
    \textbf{Method} & Top-1 & Top-3 & Top-5 & Top-1 & Top-3 & Top-5 \\
    \midrule
    Random      & 0.02 & 0.05 & 0.08 & 0.03 & 0.07 & 0.11 \\
    Attention   & 0.08 & 0.19 & 0.31 & 0.07 & 0.17 & 0.28 \\
    InputGrad   & 0.22 & 0.48 & 0.72 & 0.21 & 0.45 & 0.68 \\
    AttnLRP     & 0.26 & 0.54 & 0.81 & 0.32 & 0.71 & 1.08 \\
    TAM         & 0.31 & 0.68 & 1.02 & 0.29 & 0.63 & 0.95 \\
    \textbf{vStream} & \textbf{0.33} & \textbf{0.72} & \textbf{1.05} & \textbf{0.28} & \textbf{0.61} & 0.92 \\
    \bottomrule
    \end{tabular}
\end{table}

\begin{table}[h!]
    \centering
    \caption{Top-K Drop comparison (continued) for Document, Code, and General categories.}
    \label{tab:topk_extended_2}
    \small
    \begin{tabular}{l|ccc|ccc|ccc}
    \toprule
    & \multicolumn{3}{c|}{\textbf{Document}} & \multicolumn{3}{c|}{\textbf{Code}} & \multicolumn{3}{c}{\textbf{General}} \\
    \textbf{Method} & Top-1 & Top-3 & Top-5 & Top-1 & Top-3 & Top-5 & Top-1 & Top-3 & Top-5 \\
    \midrule
    Random      & 0.04 & 0.09 & 0.15 & 0.02 & 0.06 & 0.09 & 0.02 & 0.05 & 0.07 \\
    Attention   & 0.09 & 0.21 & 0.33 & 0.08 & 0.18 & 0.29 & 0.07 & 0.16 & 0.26 \\
    InputGrad   & 0.21 & 0.46 & 0.71 & 0.28 & 0.62 & 0.96 & 0.23 & 0.50 & 0.75 \\
    AttnLRP     & 0.24 & 0.51 & 0.78 & 0.25 & 0.54 & 0.82 & 0.30 & 0.65 & 0.98 \\
    TAM         & 0.31 & 0.68 & 1.04 & \textbf{0.33} & \textbf{0.72} & \textbf{1.06} & 0.27 & 0.59 & 0.89 \\
    \textbf{vStream} & \textbf{0.34} & \textbf{0.73} & \textbf{1.09} & 0.28 & 0.61 & 0.91 & \textbf{0.31} & \textbf{0.68} & \textbf{1.02} \\
    \bottomrule
    \end{tabular}
\end{table}

\subsection{Reasoning Trajectory Analysis: Extended Results}
\label{app:exp:trajectory}

We provide extended statistics and visualizations for the reasoning trajectory analysis introduced in \cref{sec:trajectory}. 

\subsubsection{Quantitative Metrics}

We compute two geometric metrics for each reasoning trajectory:
\begin{itemize}
    \item \textbf{Path Length}: Total Euclidean distance traveled in the PCA-projected attribution space across all reasoning steps.
    \item \textbf{Tortuosity}: Ratio of path length to net displacement (start-to-end distance). A value of 1.0 indicates a straight path; higher values indicate more wandering.
\end{itemize}

\begin{table}[h!]
    \centering
    \caption{Trajectory statistics across models and outcome types. Successful reasoning exhibits shorter, less tortuous paths in attribution space. Values show mean $\pm$ std across $n=1500$ samples per group.}
    \label{tab:trajectory_stats}
    \small
    \begin{tabular}{l|cc|cc}
    \toprule
    & \multicolumn{2}{c|}{\textbf{Path Length} $\downarrow$} & \multicolumn{2}{c}{\textbf{Tortuosity} $\downarrow$} \\
    \textbf{Model} & Success & Failure & Success & Failure \\
    \midrule
    Qwen3-VL    & $0.003 \pm 0.002$ & $0.006 \pm 0.005$ & $13.7 \pm 12.0$ & $25.4 \pm 26.1$ \\
    GLM-4.1V    & $0.004 \pm 0.003$ & $0.007 \pm 0.006$ & $14.8 \pm 13.5$ & $28.3 \pm 29.7$ \\
    MiMo-VL     & $0.003 \pm 0.002$ & $0.005 \pm 0.004$ & $11.9 \pm 10.8$ & $22.7 \pm 24.2$ \\
    Cosmos-R1   & $0.004 \pm 0.003$ & $0.006 \pm 0.005$ & $15.2 \pm 14.1$ & $27.9 \pm 28.5$ \\
    \bottomrule
    \end{tabular}
\end{table}

Across all four models, successful reasoning chains exhibit approximately 50\% shorter path lengths and 40--50\% lower tortuosity compared to unsuccessful chains. This suggests that valid reasoning corresponds to a more stable, directed traversal of the visual attribution manifold.

\subsubsection{Interpretation}

We interpret these geometric differences through the lens of ``hypothesis switching.'' Unsuccessful chains frequently reassign visual attention to different regions mid-reasoning, manifesting as erratic, high-tortuosity trajectories in attribution space. In contrast, successful chains quickly commit to a consistent set of visual evidence and maintain stable attention throughout the reasoning process.

This observation has practical implications: trajectory metrics could potentially serve as early warning signals for hallucination detection, flagging reasoning chains that exhibit unusually high path length or tortuosity before the final answer is generated.

\subsubsection{3-Way Error Analysis: Wandering vs.\ Fixation}

We decompose unsuccessful POPE examples ($n{=}3{,}000$) into \emph{Reasoning Failures} (RF; incorrect answer, plausible object) and \emph{Hallucinations} (H; object not in image).
\cref{tab:3way_trajectory} reports the full metric breakdown.

\begin{table}[h!]
\centering
\caption{3-way trajectory metrics on POPE ($n{=}3{,}000$). S~=~Success, RF~=~Reasoning Failure, H~=~Hallucination. Path geometry separates errors from success; concentration separates Fixation (H) from Wandering (RF).}
\label{tab:3way_trajectory}
\small
\begin{tabular}{l|ccc|cc}
\toprule
\textbf{Metric} & \textbf{S} & \textbf{RF} & \textbf{H} & \textbf{KW $p$} & \textbf{$d$ (S--H)} \\
\midrule
Path Length       & $0.004 \pm 0.002$ & $0.006 \pm 0.005$ & $0.005 \pm 0.003$ & $<.001$ & $-0.73$ \\
Tortuosity        & $5.6 \pm 3.9$     & $9.4 \pm 10.7$    & $8.0 \pm 5.0$     & $<.001$ & $-0.60$ \\
Concentration     & $0.209 \pm 0.054$ & $0.196 \pm 0.058$ & $0.228 \pm 0.059$ & $<.001$ & $-0.34$ \\
\bottomrule
\end{tabular}
\end{table}

To test whether the Fixation pattern extends beyond binary POPE judgments to open-ended captioning, we compute per-caption concentration on $2{,}000$ COCO captions and regress against CHAIR$_i$ \citep{rohrbach2018object}, a standard object-hallucination score for captions.

\begin{figure}[h]
\centering
\includegraphics[width=0.7\linewidth]{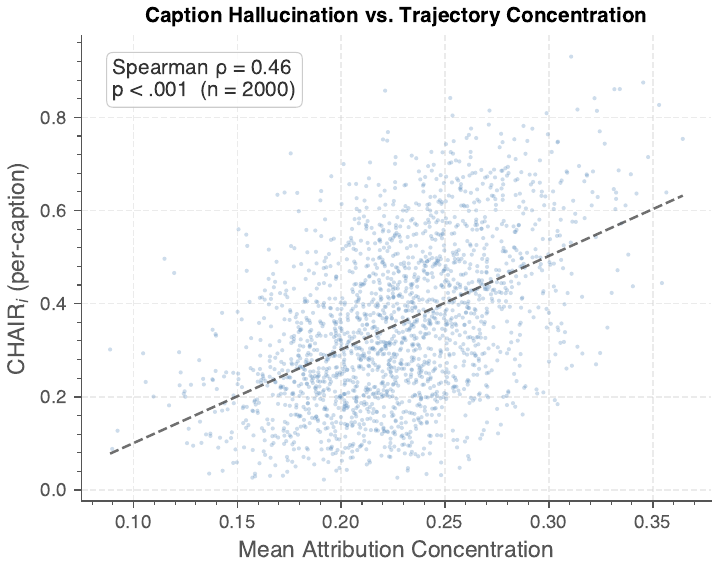}
\caption{\textbf{Per-caption attribution concentration vs.\ CHAIR$_i$ ($n{=}2{,}000$ COCO captions, Spearman $\rho{=}0.46$, $p{<}.001$).}
Higher concentration during caption generation predicts more severe object hallucination, extending the Fixation finding to open-ended generation.}
\label{fig:app_chair_scatter}
\end{figure}

\subsection{Per-Step Attribution Fidelity}
\label{app:exp:fidelity_decay}

Per-step attribution fidelity results are reported in the main text (\cref{fig:fidelity_decay}).
Correct reasoning chains maintain stable $R^2$ throughout generation, while incorrect chains degrade at roughly $20\%$ of reasoning elapsed, corroborating the early tortuosity signal in \cref{fig:conc_and_auc}.

\subsection{Estimator Architecture: Linear vs.\ MLP}
\label{app:mlp_ablation}

We compare the linear estimator used by \textsc{vStream} against MLP variants of increasing width on Qwen3-VL (3 seeds each).

\begin{table}[h!]
\centering
\caption{Linear vs.\ MLP estimator (Qwen3-VL, 3 seeds). MLP variants yield diminishing returns despite $>100\times$ more parameters.}
\label{tab:mlp_ablation}
\small
\begin{tabular}{lcccc}
\toprule
\textbf{Estimator} & \textbf{Params} & \textbf{$R^2$} & \textbf{Pearson $\rho$} & \textbf{Latency (ms)} \\
\midrule
Linear     & 1{,}152   & $0.65 \pm .01$ & $0.81 \pm .01$ & 0.03 \\
MLP-64     & 73{,}856  & $0.66 \pm .01$ & $0.82 \pm .01$ & 0.12 \\
MLP-128    & 147{,}712 & $0.67 \pm .01$ & $0.82 \pm .00$ & 0.15 \\
MLP-256    & 295{,}424 & $0.67 \pm .01$ & $0.82 \pm .01$ & 0.18 \\
\bottomrule
\end{tabular}
\end{table}

The linear estimator achieves $R^2{=}0.65$ with 1{,}152 parameters and 0.03\,ms inference latency.
MLP-256 gains $+0.02$ in $R^2$ (reaching $0.67$) but requires $256{\times}$ more parameters and $6{\times}$ higher latency.
The linear estimator is therefore the preferred choice for streaming deployment.

\subsection{Cross-Architecture Estimator Weight Patterns}
\label{app:exp:weight_patterns}

\begin{figure}[h]
\centering
\includegraphics[width=\linewidth]{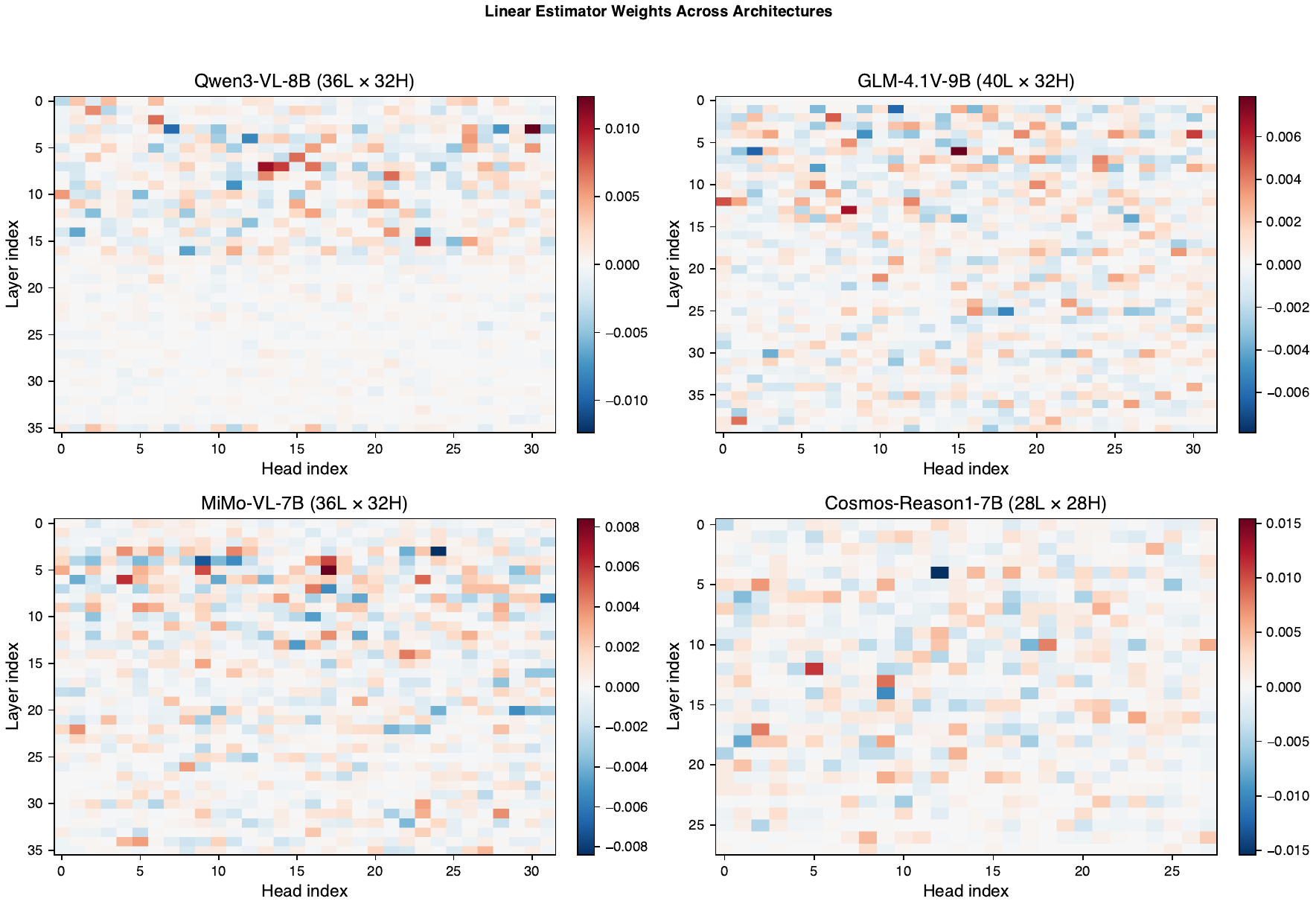}
\caption{\textbf{Estimator weight heatmaps $\mathbf{w} \in \mathbb{R}^{L \times H}$ across four architectures.}
All models concentrate weight in early-to-mid layers, suggesting a consistent architectural prior: early layers encode coarse visual-semantic alignment that is most predictive of ablation effects.}
\label{fig:app_weight_heatmap}
\end{figure}

\subsection{Context-Length Robustness}
\label{app:exp:context_length}

\begin{figure}[h]
\centering
\includegraphics[width=\linewidth]{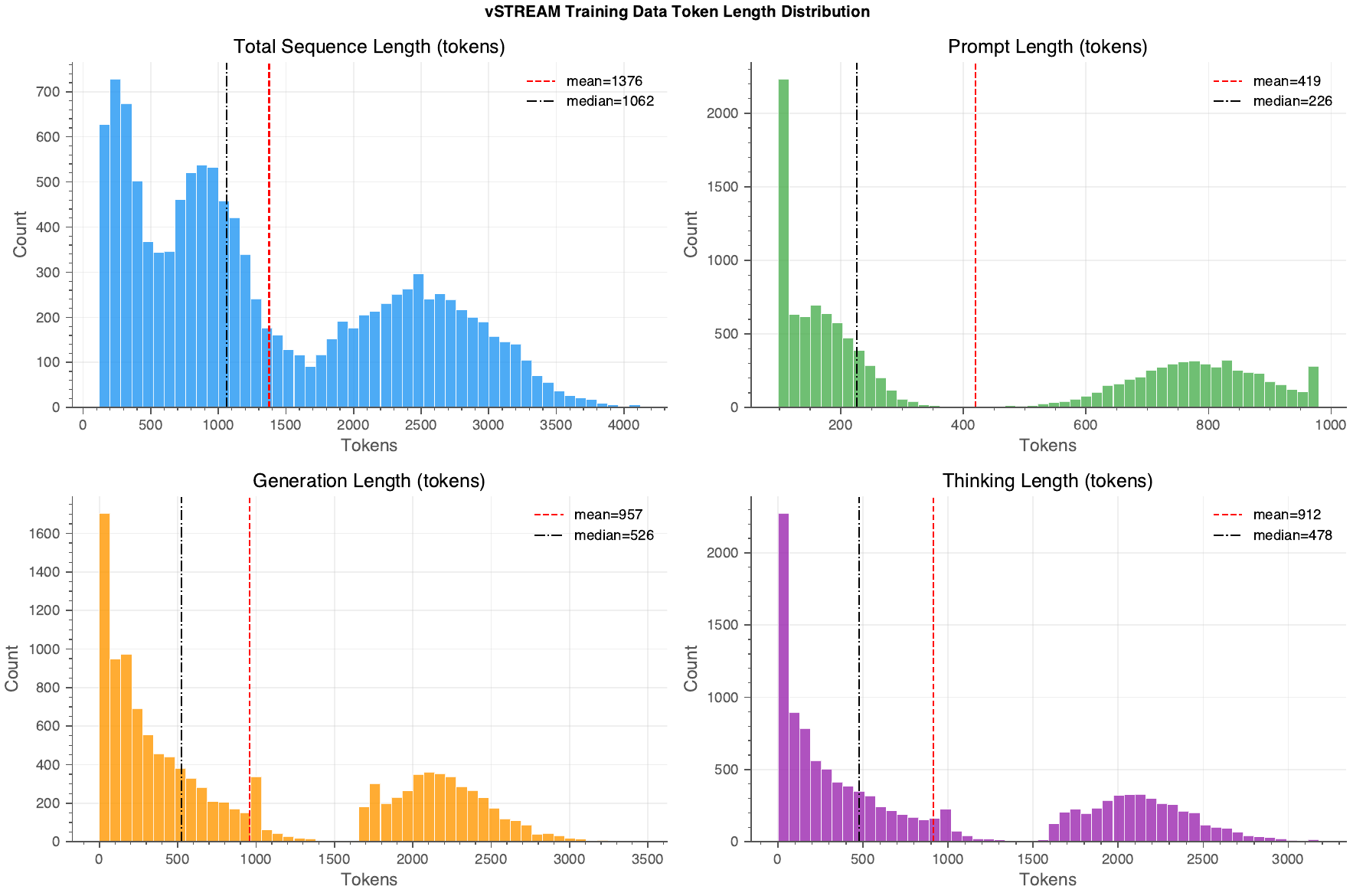}
\caption{\textbf{Generation length distribution across training examples.}
The distribution is bimodal: short VQA responses ($\sim$400 tokens) and long reasoning traces ($\sim$3{,}200 tokens).}
\label{fig:app_context_length}
\end{figure}

\begin{figure}[h]
\centering
\includegraphics[width=0.7\linewidth]{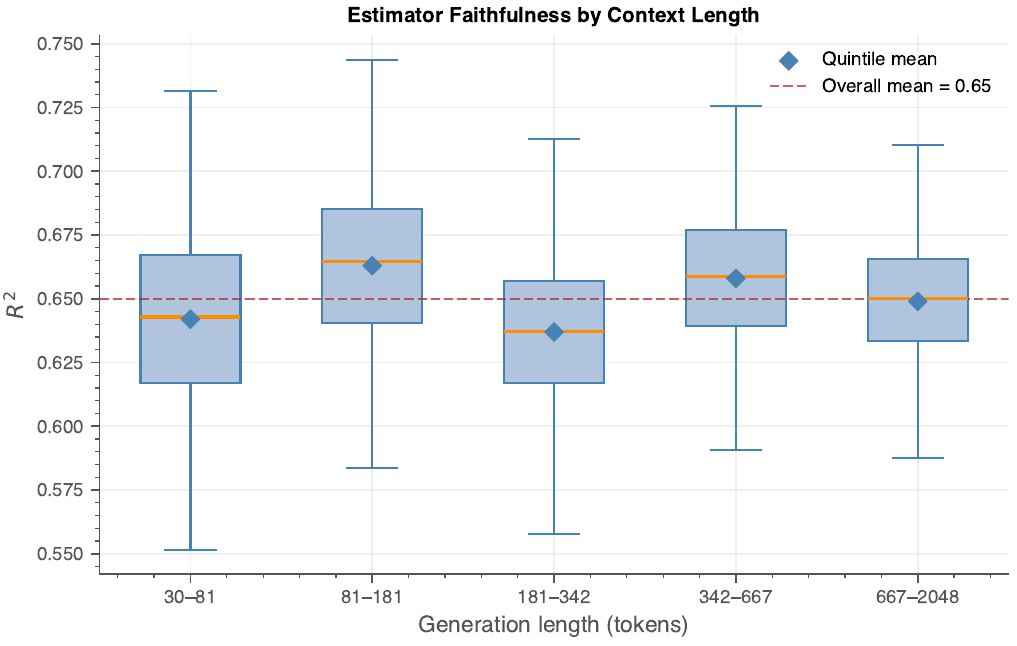}
\caption{\textbf{Faithfulness ($R^2$) by generation-length quintile.}
No systematic degradation as context length increases, indicating that the estimator generalizes across short VQA and long reasoning traces.}
\label{fig:app_faithfulness_length}
\end{figure}

\subsection{Out-of-Distribution Generalization: VQA-RAD}
\label{app:exp:ood}

We test whether the estimator trained exclusively on natural images generalizes to medical radiology images (VQA-RAD \citep{lau2018dataset}) without retraining.

\begin{figure}[h]
\centering
\includegraphics[width=0.7\linewidth]{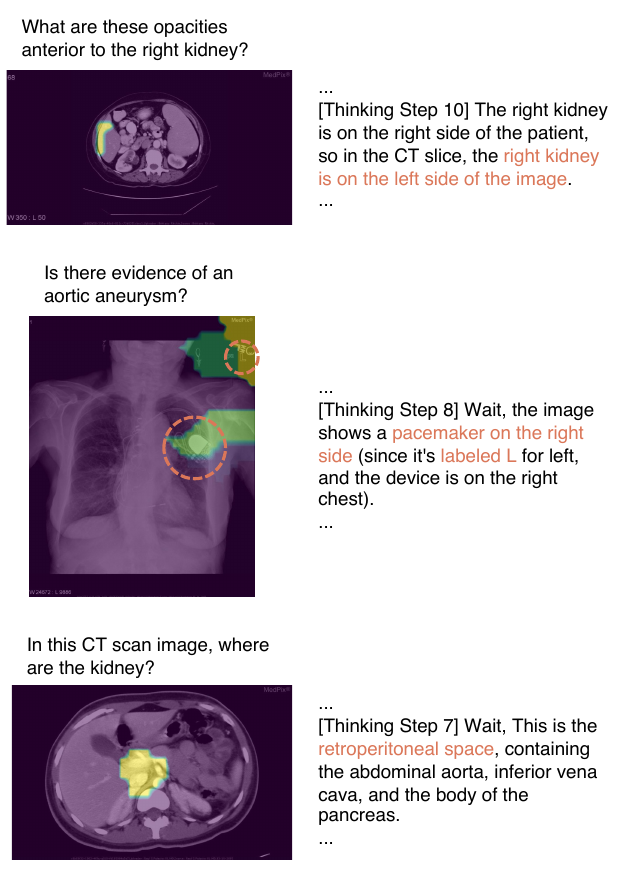}
\caption{\textbf{OOD attribution on VQA-RAD (medical radiology).}
The estimator, trained on natural images only and applied without retraining, correctly localizes kidneys in abdominal CT (top, bottom) and a pacemaker in chest X-ray (middle).}
\label{fig:app_ood_vqarad}
\end{figure}

\begin{table}[h!]
\centering
\caption{Attribution fidelity on VQA-RAD (OOD). \textsc{vStream} estimator trained on natural images only, applied without retraining. Format: LDS\,/\,Top-5 Drop; higher is better.}
\label{tab:ood_vqarad}
\small
\begin{tabular}{lcc}
\toprule
\textbf{Method} & \textbf{LDS} & \textbf{Top-5 Drop} \\
\midrule
Random     & 0.29 & 0.14 \\
Attention  & 0.41 & 0.32 \\
InputGrad  & 0.58 & 0.71 \\
AttnLRP    & 0.66 & 0.79 \\
TAM        & 0.63 & 0.88 \\
\rowcolor{gray!10}
vStream (Ours) & 0.67 & 0.84 \\
\bottomrule
\end{tabular}
\end{table}

Despite the large domain shift, \textsc{vStream} achieves competitive fidelity with no domain adaptation, demonstrating that the learned attention-to-ablation mapping transfers across image modalities.

\subsection{Segmentation Evaluation}
\label{app:exp:segmentation}

While our primary goal is visual attribution rather than semantic segmentation, we evaluate how well our attribution maps align with ground-truth object masks. We use three benchmarks: ImageNet-Segmentation \citep{guillaumin2014imagenet}, which provides pixel-level annotations for a subset of ImageNet validation images; COCO \citep{lin2014coco}, with instance segmentation masks; and the RefCOCO family of datasets \citep{kazemzadeh2014refcoco} for referring expression grounding. For RefCOCO evaluation, we report results on the \textit{val} split, which is the standard evaluation protocol. The RefCOCO family includes three variants: RefCOCO (general expressions), RefCOCO+ (excludes absolute location terms like ``left'' or ``right''), and RefCOCOg (longer, more complex expressions). We use RefCOCO-val as it provides a balanced evaluation across diverse object categories. Note that attribution methods naturally highlight causally relevant regions rather than full object extents, making segmentation a challenging proxy task.

\begin{table}[h!]
    \centering
    \caption{Segmentation performance (mIoU) on Qwen3-VL across three benchmarks. For RefCOCO, we report results on the val split. All methods achieve comparable scores, confirming that attribution maps do not directly correspond to object boundaries.}
    \label{tab:segmentation}
    \small
    \begin{tabular}{lccc}
        \toprule
        \textbf{Method} & \textbf{ImageNet-Seg} & \textbf{COCO} & \textbf{RefCOCO-val} \\
        \midrule
        Attention   & 0.21 & 0.19 & 0.23 \\
        InputGrad   & 0.23 & 0.21 & 0.25 \\
        AttnLRP     & 0.27 & 0.24 & 0.27 \\
        TAM         & 0.27 & 0.25 & 0.27 \\
        \rowcolor{gray!10}
        vStream (Ours) & 0.27 & 0.24 & 0.27 \\
        \bottomrule
    \end{tabular}
\end{table}

The results indicate that all methods achieve comparable mIoU scores in the 0.20--0.30 range, which is typical for attribution-to-segmentation evaluation. This confirms that attribution maps, designed to highlight causally relevant evidence, do not necessarily align with full object boundaries.

\section{Qualitative Analysis}
\label{app:qualitative}

We provide additional qualitative results in \cref{fig:qual_qwen_1_ms}--\cref{fig:qual_mimo_2_gdc}. These figures show step-by-step visual attribution results from our main model (Qwen3-VL) and all baseline methods across all five task categories: Math, Science, Document, Code, and General.

Each figure displays the attribution heatmap for individual reasoning steps, enabling direct comparison between \textsc{vStream} and baseline attribution methods. The visualizations demonstrate that \textsc{vStream} produces focused, semantically coherent attribution maps that track the model's reasoning process in real-time.

\subsection{Failure Analysis}
\label{app:qualitative:failure}

Despite strong performance, we identify distinct failure modes where attribution remains challenging. These failure cases are not directly visualized in the figures but are described below based on our analysis.

\paragraph{Ambiguity in Visual References.}
When the visual input contains multiple identical objects (e.g., ``count the red apples''), the attribution map often splits intensity across all candidates simultaneously rather than sequentially focusing on individual instances. While strictly ``correct'' in terms of feature matching, this does not reflect the sequential nature of human counting and can make traces harder to interpret.

\paragraph{Hallucination and Disconnected Reasoning.}
When the model generates plausible text that is not grounded in the image (hallucination), our attribution maps often become diffuse or uniform, lacking a clear focal point. This suggests that diffuse attribution could serve as a detector for hallucinated content, an avenue for future research.

\paragraph{Dense Text and Small Objects.}
For images containing dense text (documents, code screenshots) or small objects, the semantic regions from DINO clustering may not perfectly isolate individual characters or tiny elements. In these cases, attribution may highlight a region containing the relevant element but lack sub-region precision.

\paragraph{Abstract Reasoning Steps.}
During purely symbolic computation steps (e.g., ``Therefore, 2x + 3 = 7 implies x = 2''), visual attribution naturally becomes weak because the reasoning genuinely does not depend on visual input. Our method correctly produces low attribution scores in these cases. Users should interpret weak attribution during computational steps as expected behavior rather than a failure mode.

\section{Broader Impact and Limitations}
\label{app:impact}

\subsection{Broader Impact}

\textbf{Advancing Trustworthy AI.} 
The primary contribution of this work is to enhance the transparency and interpretability of Vision-Language Models (VLMs). As these models are increasingly deployed in high-stakes domains—such as medical diagnosis, autonomous navigation, and legal analysis—the ability to verify \textit{why} a model made a decision is paramount. By providing fast and accurate attribution maps, our method allows human operators to verify that model reasoning relies on relevant visual evidence rather than spurious correlations or hallucinations. This is a critical step towards safe deployment of reasoning models.

\textbf{Potential Risks.} 
While intended to reveal model reasoning, interpretability tools can be double-edged. There is a risk that accurate-looking attribution maps could be used to generate convincing justifications for incorrect or biased model decisions, potentially leading users to over-trust a flawed system. Furthermore, if the estimator itself is adversarially manipulated, it could hide the model's reliance on sensitive or protected attributes (e.g., race or gender) in decision-making processes. It is crucial that these tools are used as part of a holistic auditing framework, not as a standalone guarantee of safety.

\subsection{Limitations}

We identify several limitations of our approach that should be considered when applying \textsc{vStream} in practice.

\textbf{Linearity Assumption.} 
Our method employs a linear estimator to predict the causal effect of feature ablation. While our empirical results suggest that linear directions in the activation space of modern transformers capture significant causal information, this is a simplifying assumption. Complex, non-linear interactions between visual features (where the suppression of one feature only matters if another is also present) may not be fully captured by our current formulation. Future work could explore non-linear estimators (e.g., MLPs or attention-based predictors) for the attribution head, though this may trade off interpretability and training stability.

\textbf{Dependency on Visual Backbone (DINO).} 
Our approach relies on the quality of the underlying visual representations (specifically DINOv3) to define semantic regions for attribution. If the visual encoder fails to semantically separate relevant objects or concepts, or if the feature resolution is too coarse, our estimator cannot recover precise attributions. This dependency means:
\begin{itemize}
    \item For images with unusual visual content not well-represented in DINO's training data, region quality may degrade.
    \item Very small objects or fine-grained text may not be isolated into separate regions.
    \item The computational cost of running DINOv3 adds approximately 45ms per image.
\end{itemize}

\textbf{Training Data Quality.}
The estimator is trained only on examples where the VLM produces correct final answers. This design choice ensures we learn attention patterns associated with successful reasoning. However, it also means:
\begin{itemize}
    \item The estimator may be less calibrated for incorrect reasoning traces.
    \item Attribution quality depends on having sufficient correctly-answered examples in each domain.
    \item For very challenging tasks where models rarely succeed, collecting training data becomes difficult.
\end{itemize}

\textbf{Generalization Across Domains.} 
While our cross-task experiments show reasonable transfer between related domains (e.g., Math $\leftrightarrow$ Science), transfer to visually distinct domains (e.g., Math $\rightarrow$ Document) is weaker. This suggests that practitioners may need to collect domain-specific training data for optimal performance in specialized applications.

\textbf{Attention as a Proxy for Information Flow.}
Our method assumes that attention weights are informative proxies for information flow in transformers. However, recent work has shown that attention can be manipulated without affecting outputs \citep{jain-wallace-2019-attention}, and that alternative pathways (e.g., residual connections, MLP layers) can route information independently of attention. Our empirical success suggests attention captures sufficient signal for attribution, but it may not be a complete picture of the model's information processing.

\textbf{Computational Requirements for Ground-Truth Collection.}
While inference is fast, collecting ground-truth ablation effects for training requires multiple forward passes per example. For a training set of 500 examples with 32 masks each, this amounts to 16,000 forward passes. On an H100 GPU, this takes approximately 2--4 hours per model-task pair.

\textbf{Model-Specific Training Requirement.}
Our estimator must be trained separately for each VLM, as the learned linear projection is inherently tied to the specific model's internal representation space. This limitation is fundamental rather than incidental: different VLMs employ distinct architectures, attention mechanisms, and learned feature spaces, making cross-model transfer of the estimator impractical without retraining. While this per-model training requirement adds deployment overhead, we note that (1) the training cost is modest (2--4 hours on a single GPU per model-task pair), and (2) this constraint is shared by virtually all model-specific interpretability methods that operate on internal representations. We view the development of architecture-agnostic attribution features as an important direction for future work, as discussed in \cref{app:impact:future}.

\subsection{Future Directions}
\label{app:impact:future}

Several promising directions could extend and improve upon our work. On the methodological side, exploring non-linear estimators such as attention-based or MLP-based architectures could capture complex feature interactions that our linear model may miss, potentially improving attribution quality on images with intricate visual relationships. Developing architecture-agnostic features would enable zero-shot cross-model transfer, allowing practitioners to train an estimator on one VLM and apply it to another without retraining. This would dramatically reduce the barrier to deploying attribution tools across the rapidly evolving landscape of vision-language models.

On the application side, our observation that hallucinated content produces diffuse attribution maps suggests a natural extension toward automatic hallucination detection, where attribution entropy or concentration metrics could serve as early warning signals for unreliable model outputs. Attribution maps could also guide targeted interventions for model editing, enabling surgical corrections to model behavior without full retraining. Finally, extending our framework to video understanding and multi-image reasoning would address the growing importance of temporal and multi-view visual reasoning, where tracking visual attribution across frames could reveal how models integrate information over time.

\newpage
\begin{figure*}[!t]
\centering
\rotatebox{90}{%
\begin{minipage}{\textheight}
\centering
\includegraphics[width=\textwidth]{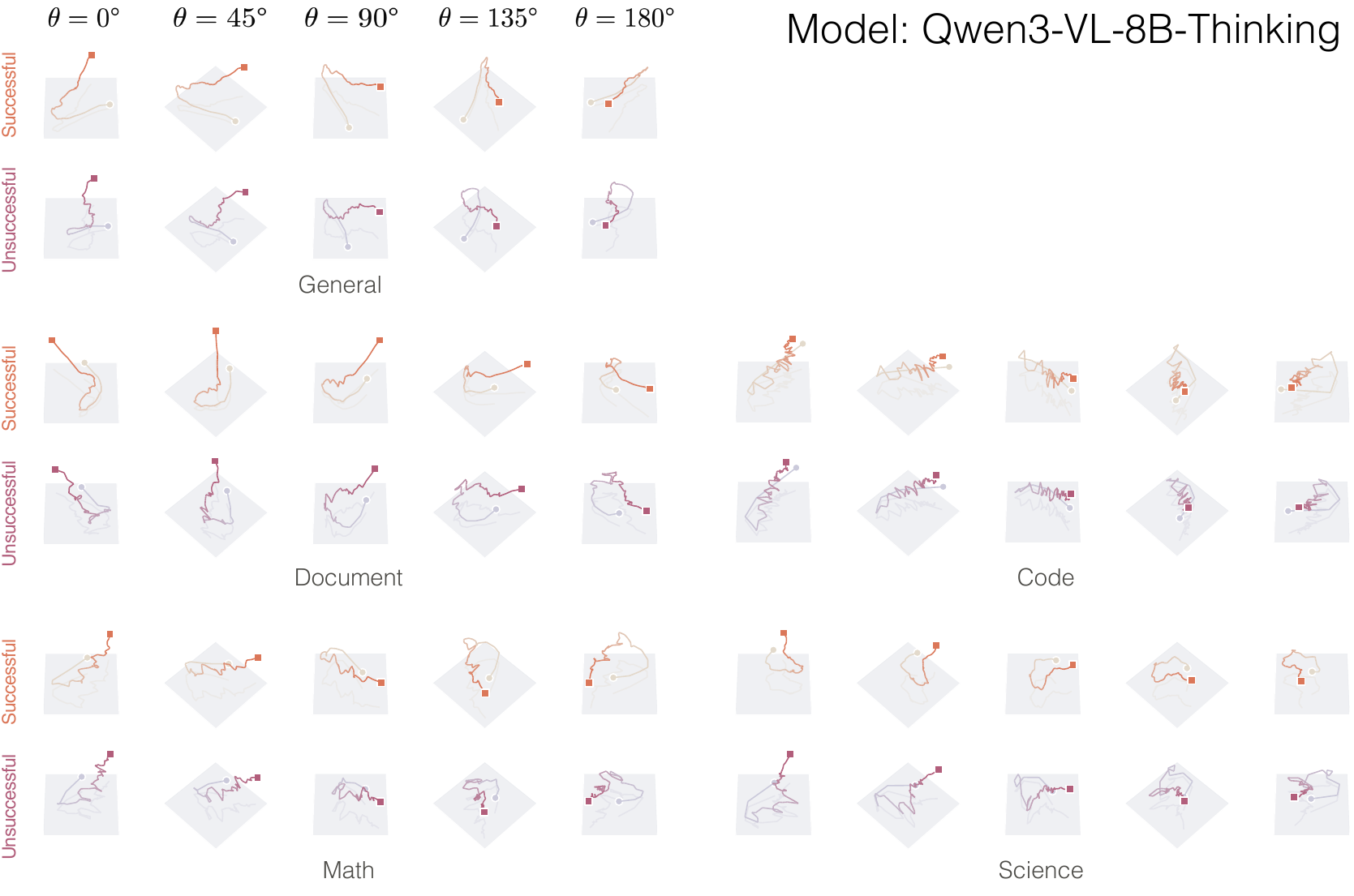}
\captionof{figure}{
    \textbf{Reasoning trajectory dynamics for Qwen3-VL-8B-Thinking.}
    Visual attribution trajectories projected into PCA space for successful (\textit{left}) and unsuccessful (\textit{right}) reasoning chains.
}
\label{fig:traj_qwen}
\end{minipage}%
}
\end{figure*}

\begin{figure*}[!t]
\centering
\rotatebox{90}{%
\begin{minipage}{\textheight}
\centering
\includegraphics[width=\textwidth]{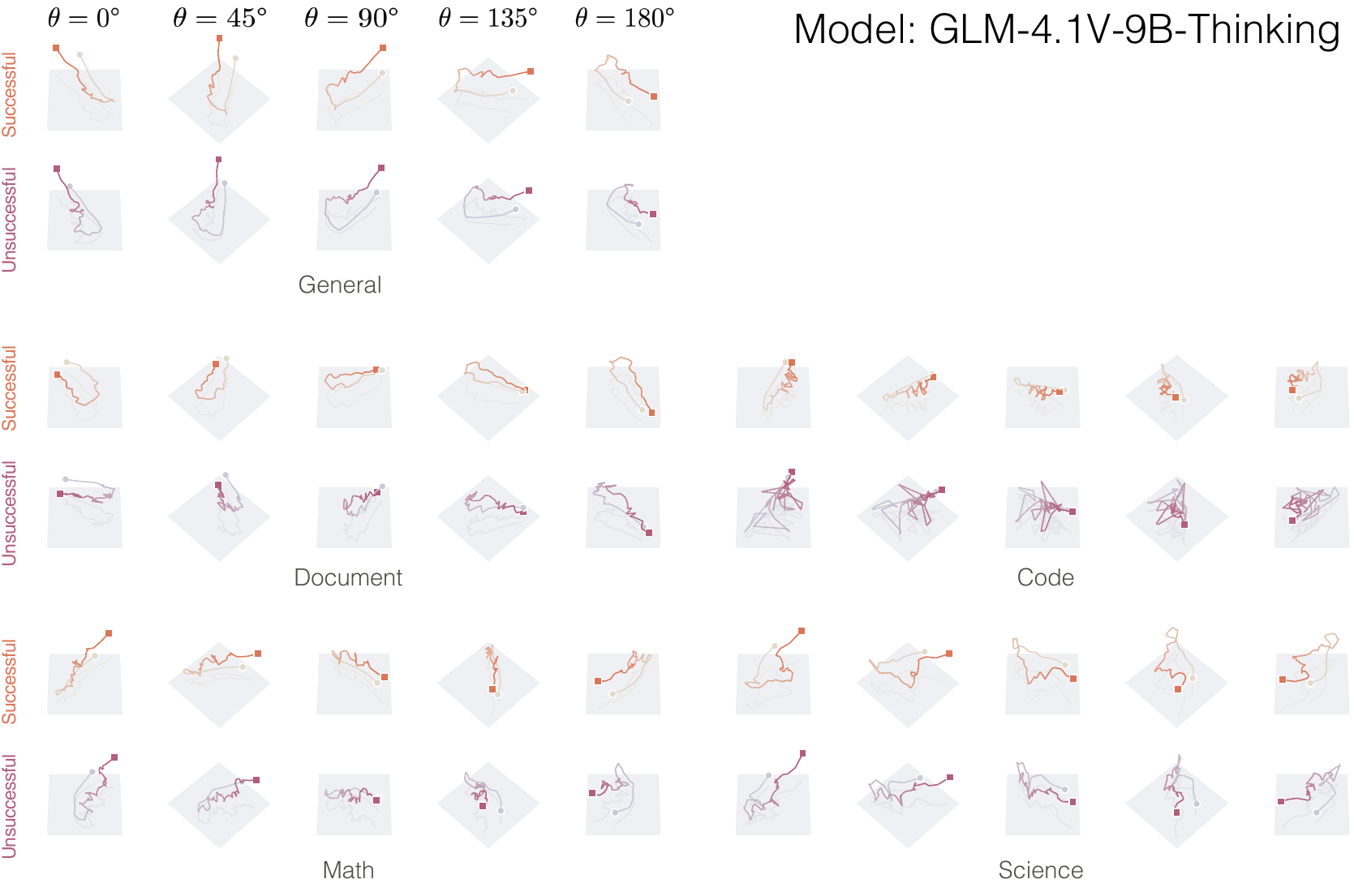}
\captionof{figure}{
    \textbf{Reasoning trajectory dynamics for GLM-4.1V-9B-Thinking.}
    Visual attribution trajectories projected into PCA space for successful (\textit{left}) and unsuccessful (\textit{right}) reasoning chains.
}
\label{fig:traj_glm}
\end{minipage}%
}
\end{figure*}

\begin{figure*}[!t]
\centering
\rotatebox{90}{%
\begin{minipage}{\textheight}
\centering
\includegraphics[width=\textwidth]{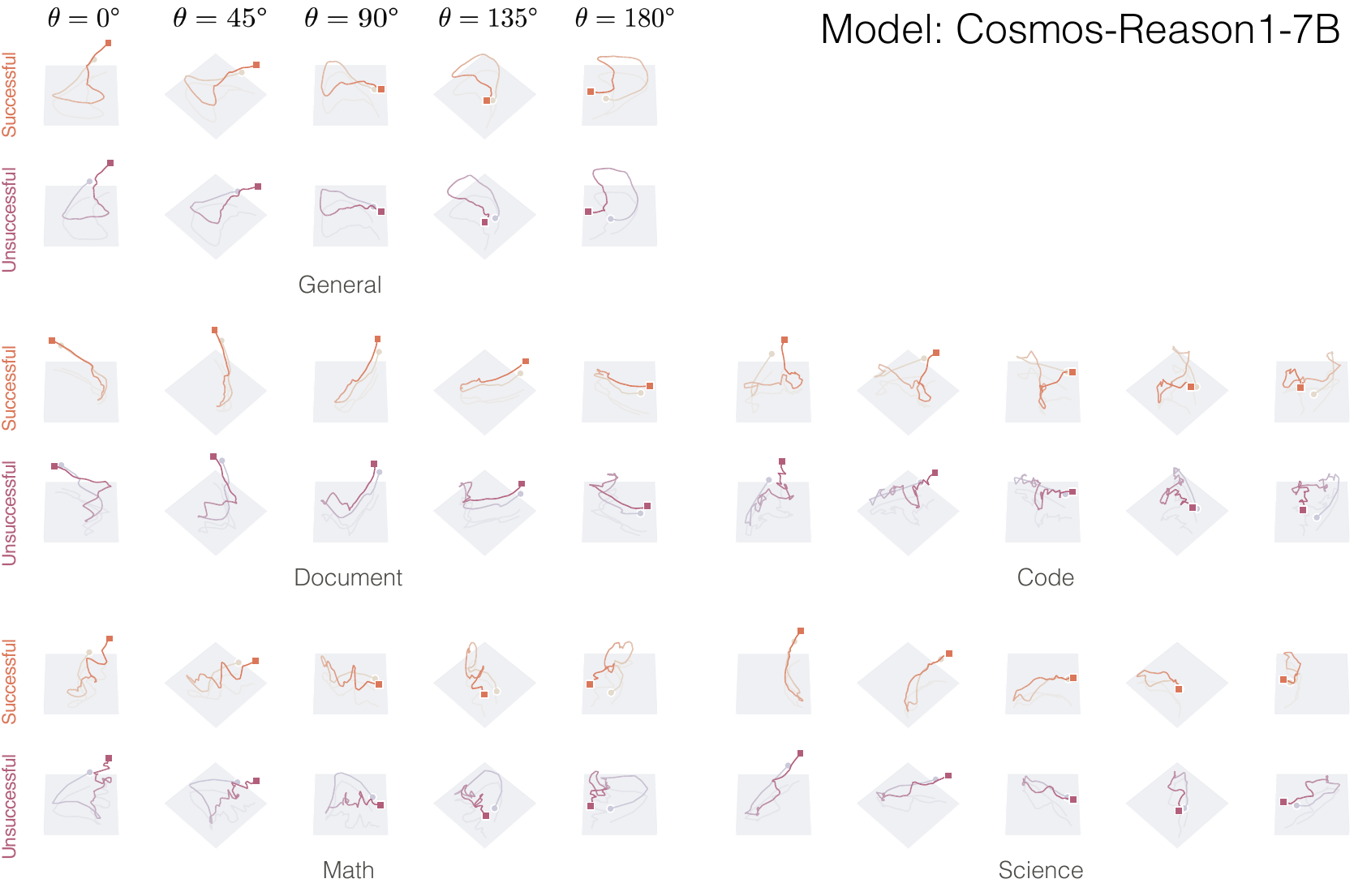}
\captionof{figure}{
    \textbf{Reasoning trajectory dynamics for Cosmos-R1.}
    Visual attribution trajectories projected into PCA space for successful (\textit{left}) and unsuccessful (\textit{right}) reasoning chains.
}
\label{fig:traj_cosmos}
\end{minipage}%
}
\end{figure*}

\begin{figure*}[!t]
\centering
\rotatebox{90}{%
\begin{minipage}{\textheight}
\centering
\includegraphics[width=\textwidth]{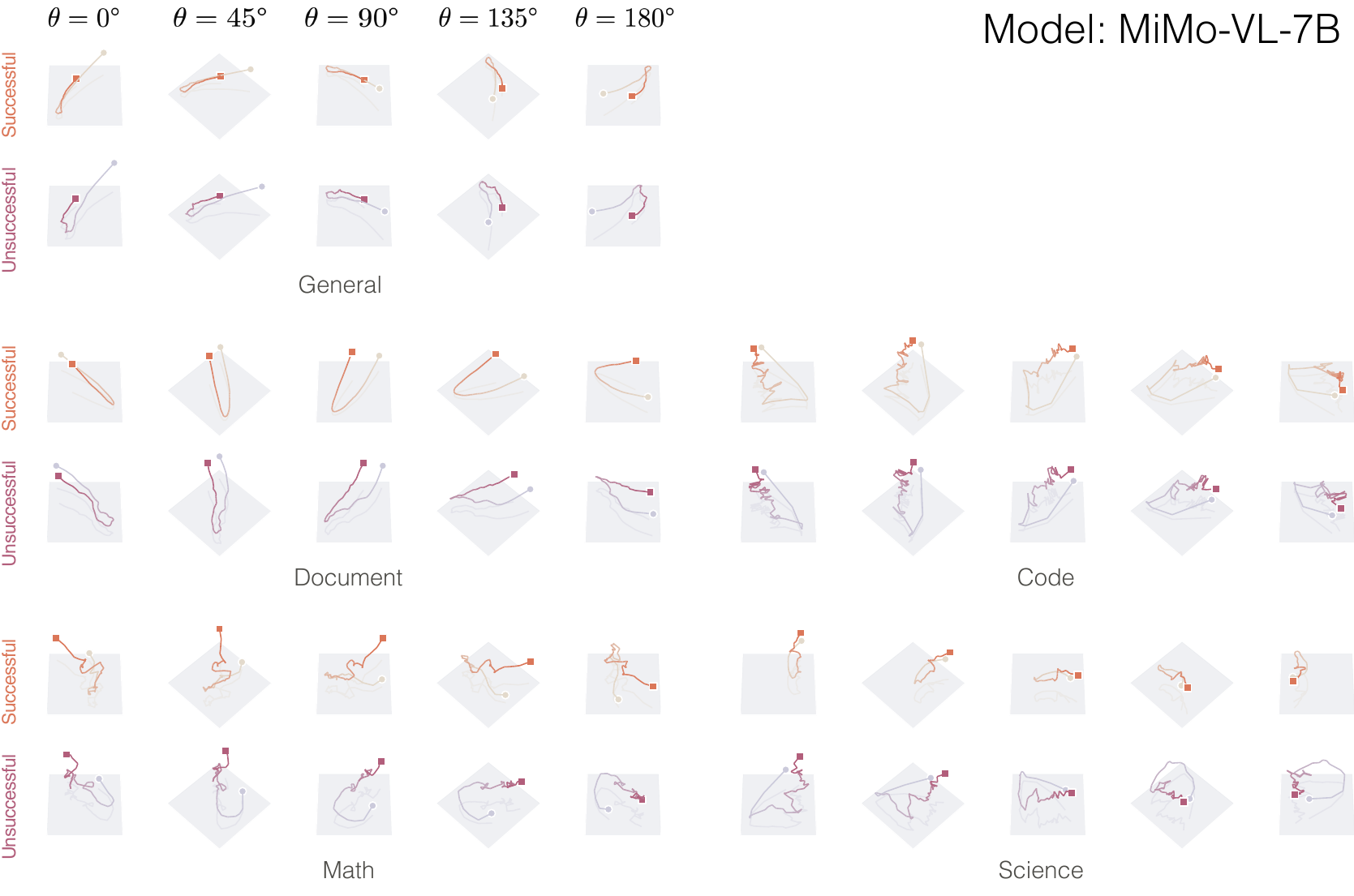}
\captionof{figure}{
    \textbf{Reasoning trajectory dynamics for MiMo-VL-7B.}
    Visual attribution trajectories projected into PCA space for successful (\textit{left}) and unsuccessful (\textit{right}) reasoning chains.
}
\label{fig:traj_mimo}
\end{minipage}%
}
\end{figure*}

\begin{figure*}[t]
    \centering
    \includegraphics[width=\linewidth]{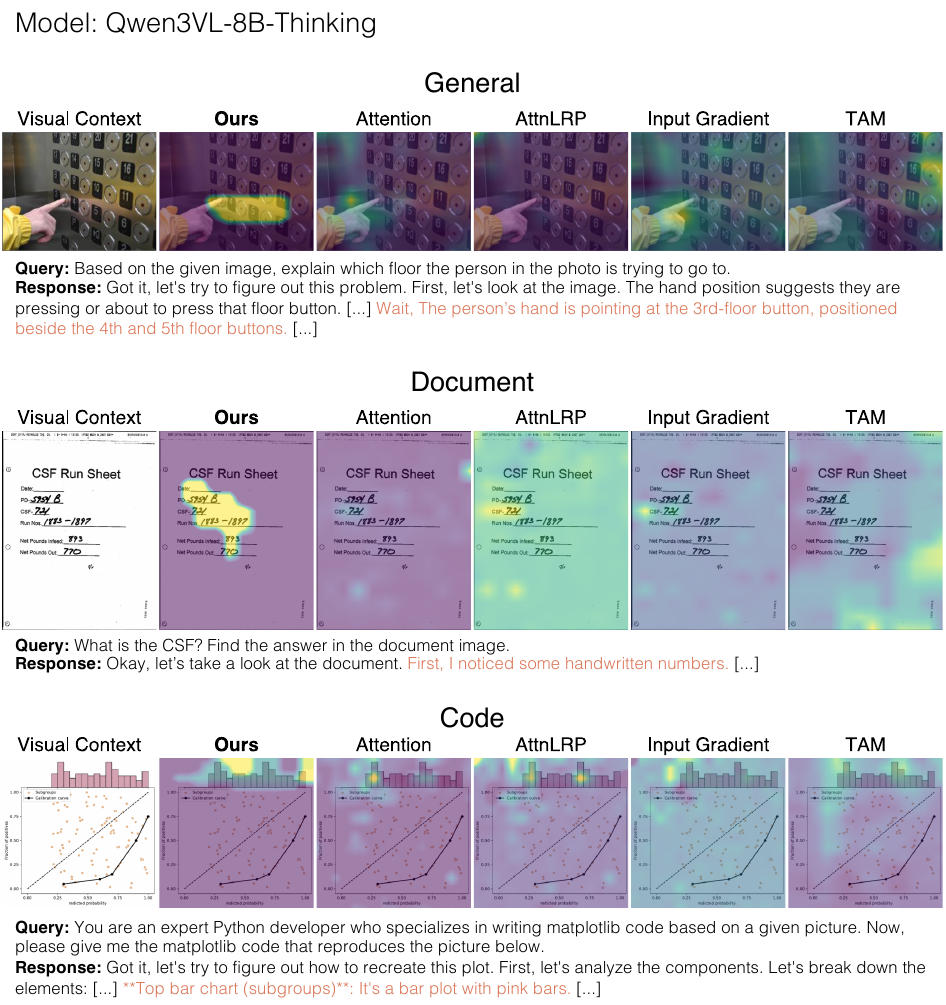}
    \caption{
        \textbf{Additional qualitative results: Qwen3-VL on general/document/code reasoning.}
    }
    \label{fig:qual_qwen_1_gdc}
\end{figure*}

\begin{figure*}[t]
    \centering
    \includegraphics[width=\linewidth]{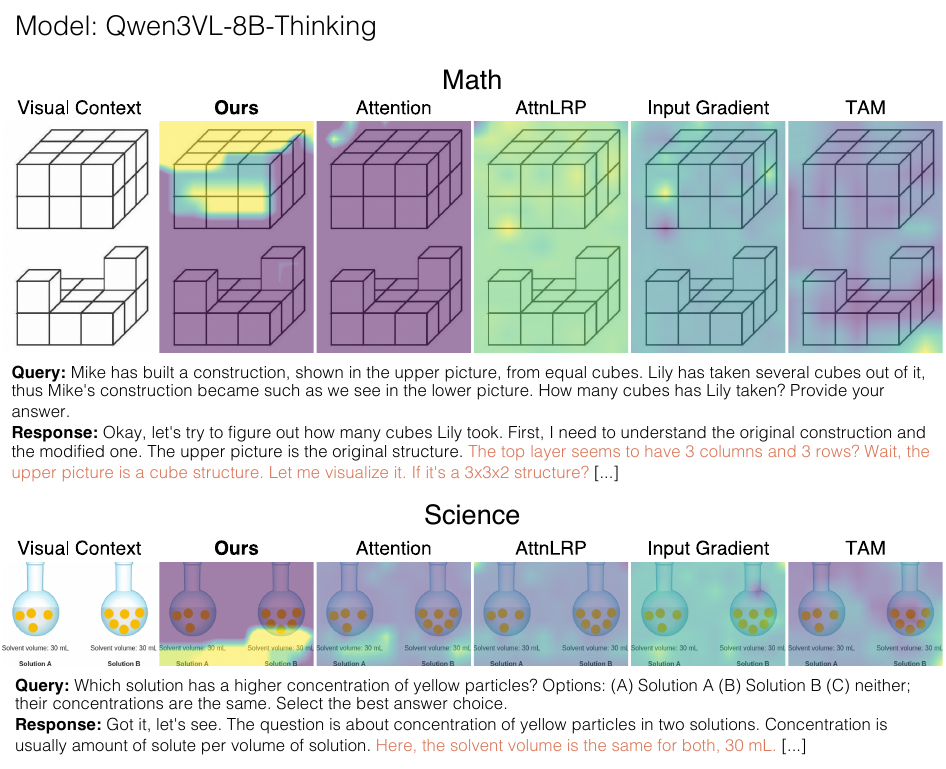}
    \caption{
        \textbf{Additional qualitative results: Qwen3-VL on math/science reasoning.}
    }
    \label{fig:qual_qwen_1_ms}
\end{figure*}

\begin{figure*}[t]
    \centering
    \includegraphics[width=\linewidth]{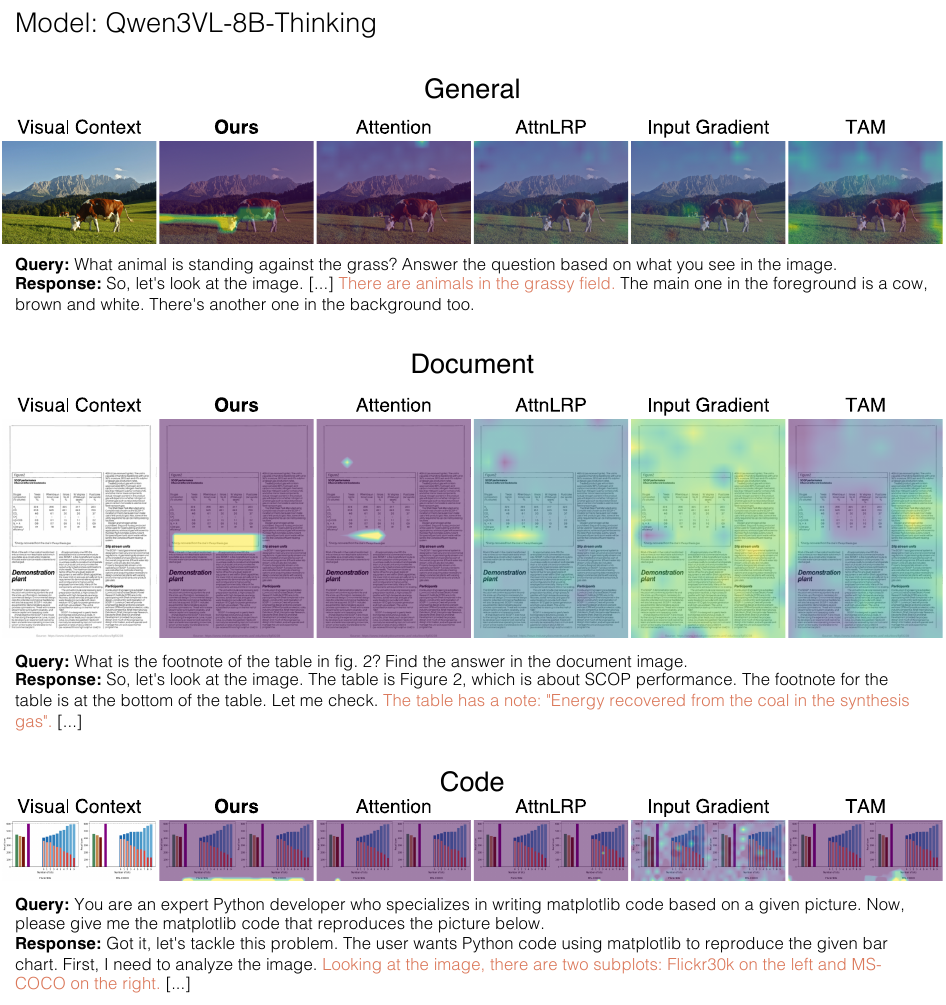}
    \caption{
        \textbf{Additional qualitative results: Qwen3-VL on general/document/code reasoning.}
    }
    \label{fig:qual_qwen_2_gdc}
\end{figure*}

\begin{figure*}[t]
    \centering
    \includegraphics[width=\linewidth]{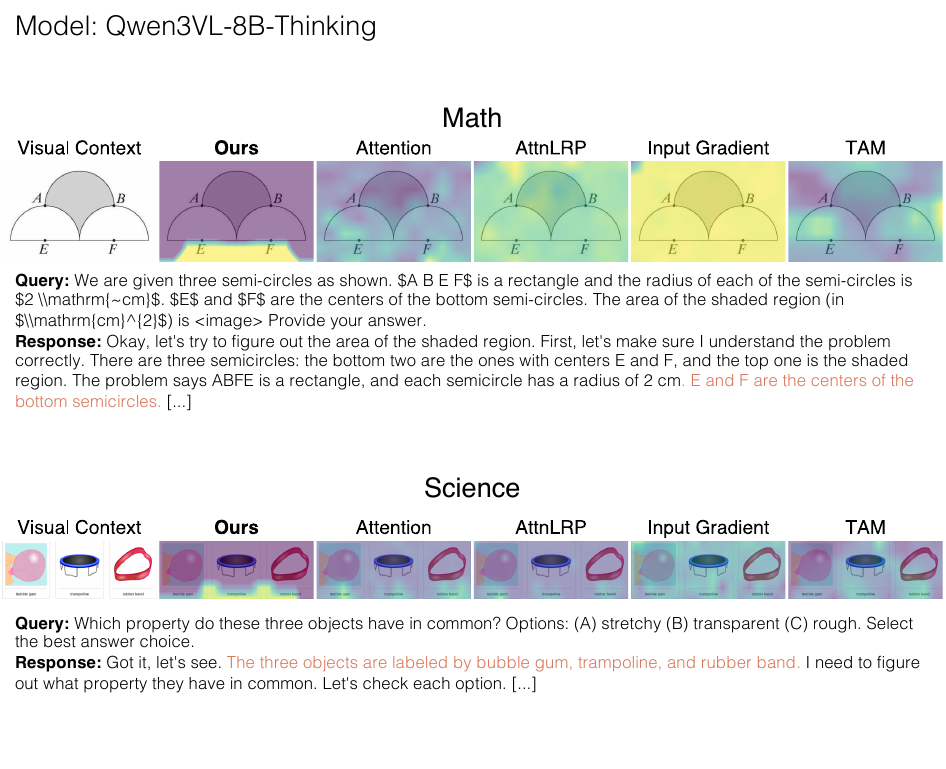}
    \caption{
        \textbf{Additional qualitative results: Qwen3-VL on math/science reasoning.}
    }
    \label{fig:qual_qwen_2_ms}
\end{figure*}

\begin{figure*}[t]
    \centering
    \includegraphics[width=\linewidth]{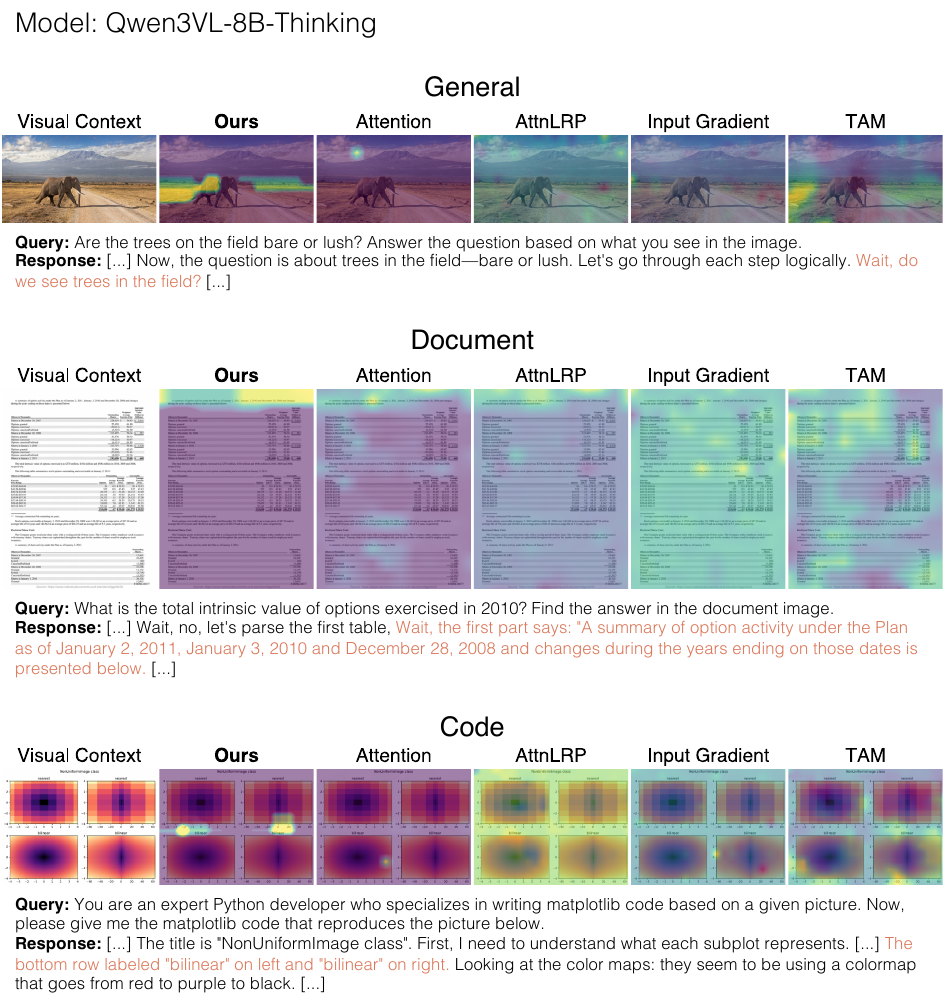}
    \caption{
        \textbf{Additional qualitative results: Qwen3-VL on general/document/code reasoning.}
    }
    \label{fig:qual_qwen_3_gdc}
\end{figure*}

\begin{figure*}[t]
    \centering
    \includegraphics[width=\linewidth]{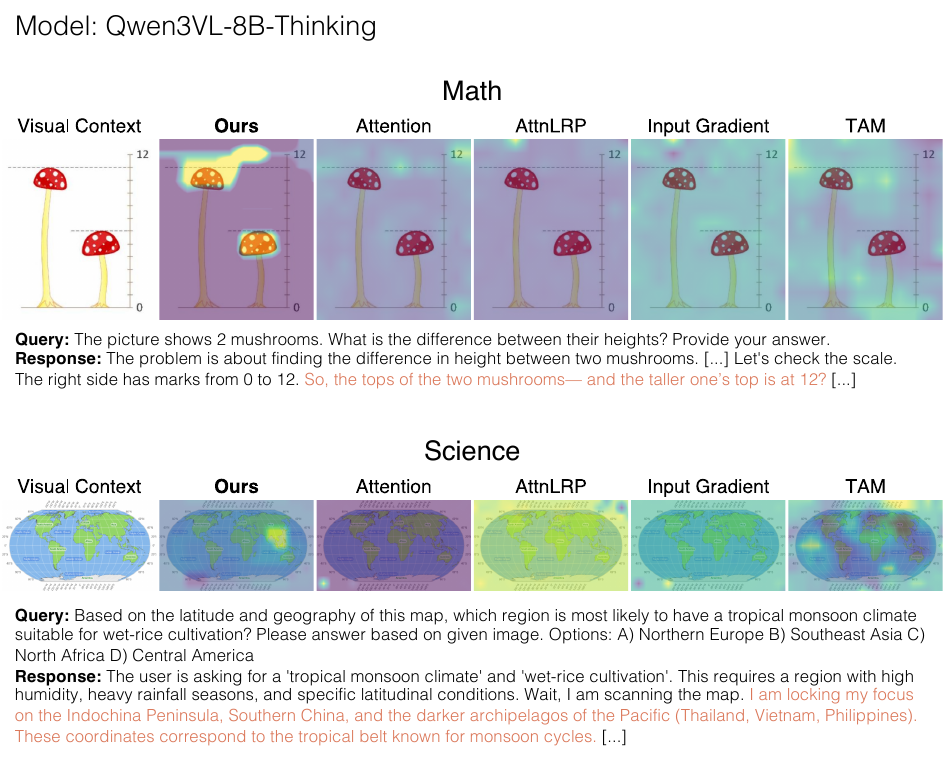}
    \caption{
        \textbf{Additional qualitative results: Qwen3-VL on math/science reasoning.}
    }
    \label{fig:qual_qwen_3_ms}
\end{figure*}

\begin{figure*}[t]
    \centering
    \includegraphics[width=\linewidth]{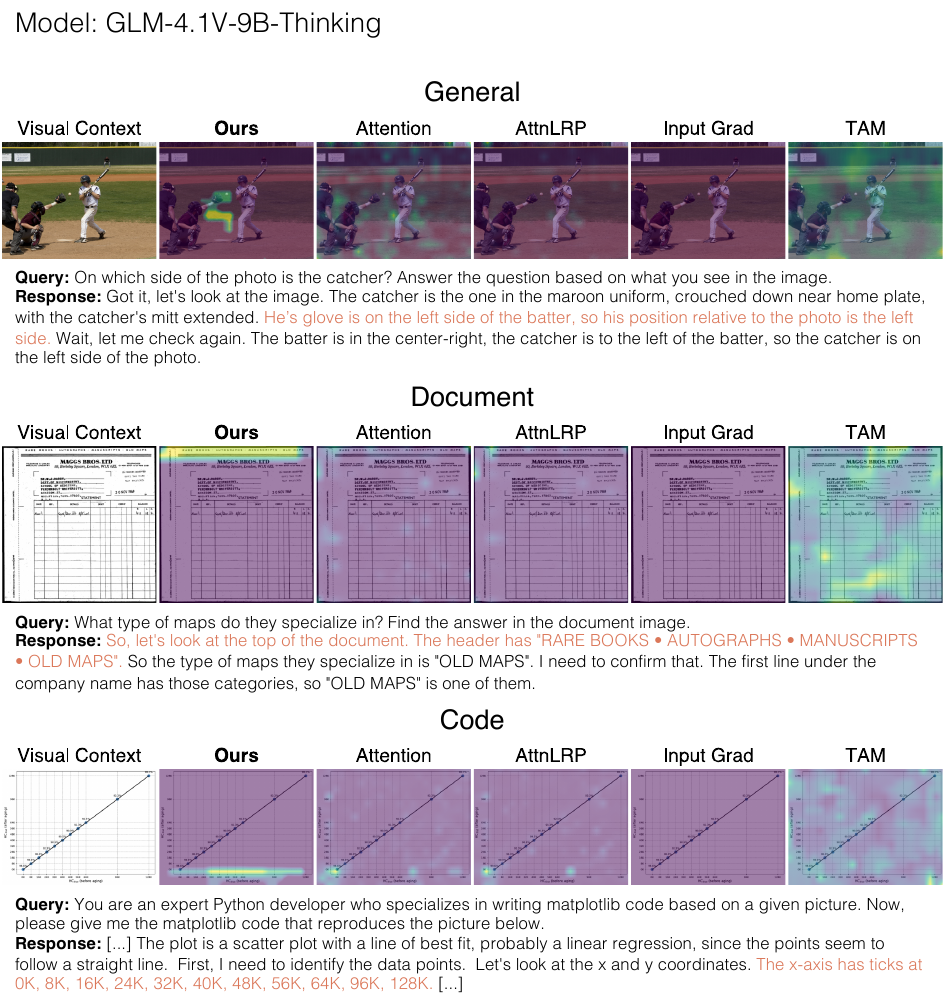}
    \caption{
        \textbf{Additional qualitative results: GLM-4.1V on general/document/code reasoning.}
    }
    \label{fig:qual_glm_1_gdc}
\end{figure*}

\begin{figure*}[t]
    \centering
    \includegraphics[width=\linewidth]{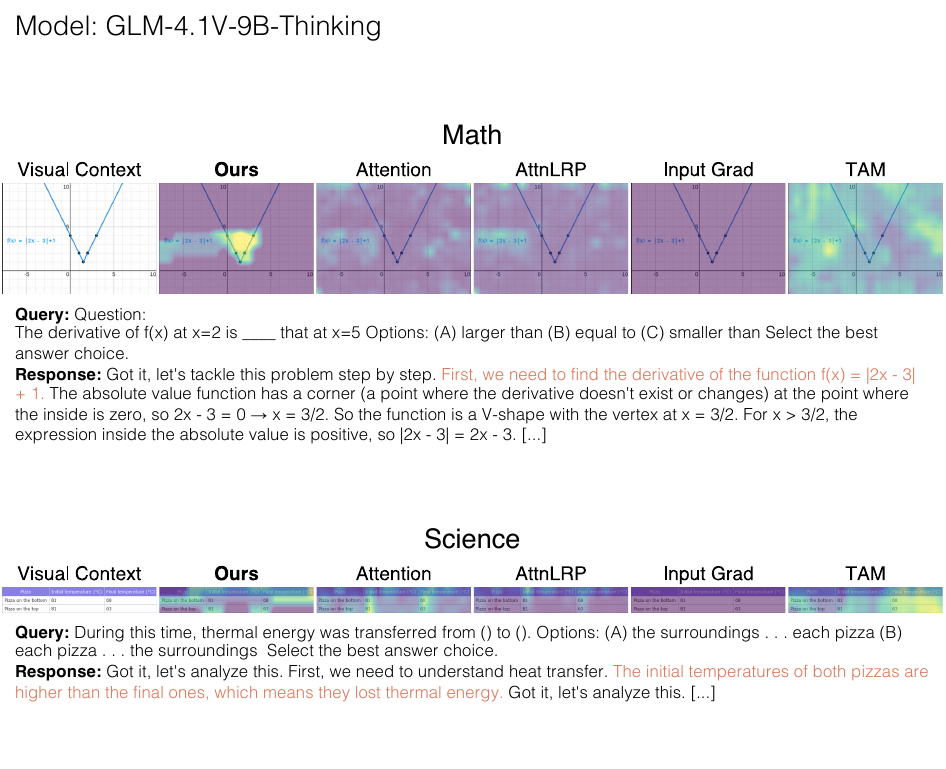}
    \caption{
        \textbf{Additional qualitative results: GLM-4.1V on math/science reasoning.}
    }
    \label{fig:qual_glm_1_ms}
\end{figure*}

\begin{figure*}[t]
    \centering
    \includegraphics[width=\linewidth]{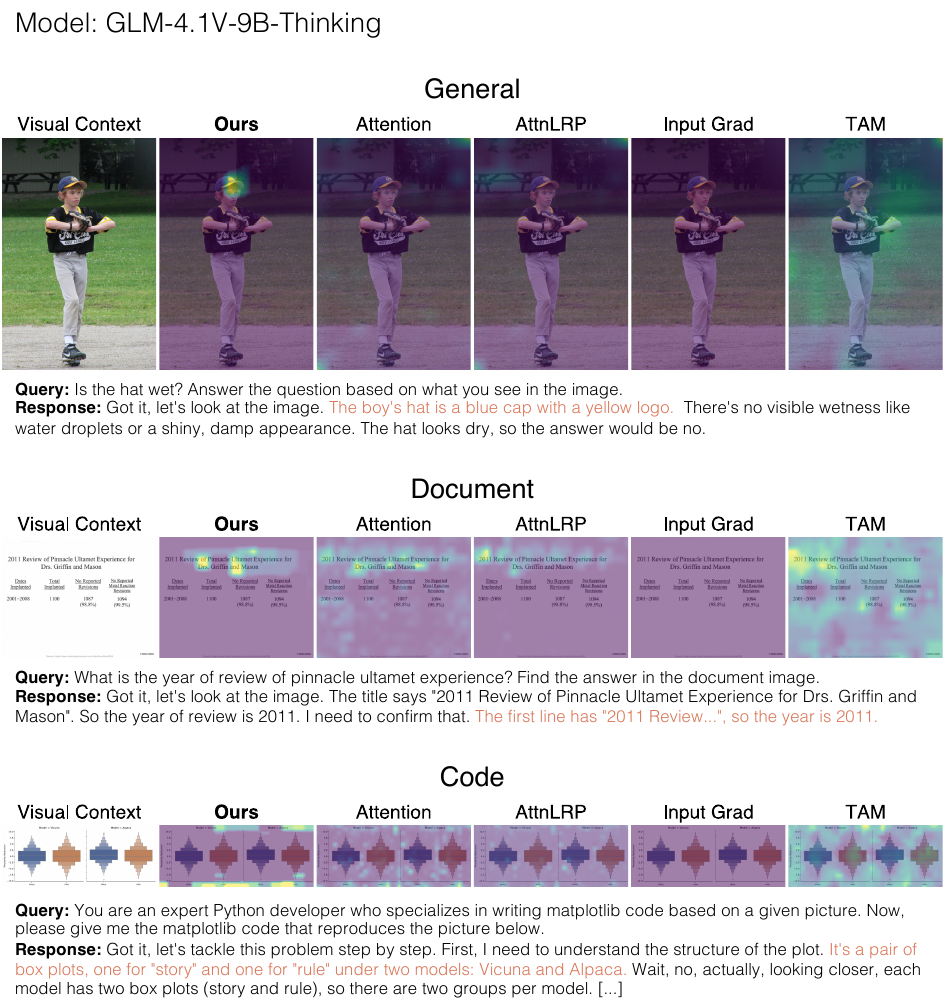}
    \caption{
        \textbf{Additional qualitative results: GLM-4.1V on general/document/code reasoning.}
    }
    \label{fig:qual_glm_2_gdc}
\end{figure*}

\begin{figure*}[t]
    \centering
    \includegraphics[width=\linewidth]{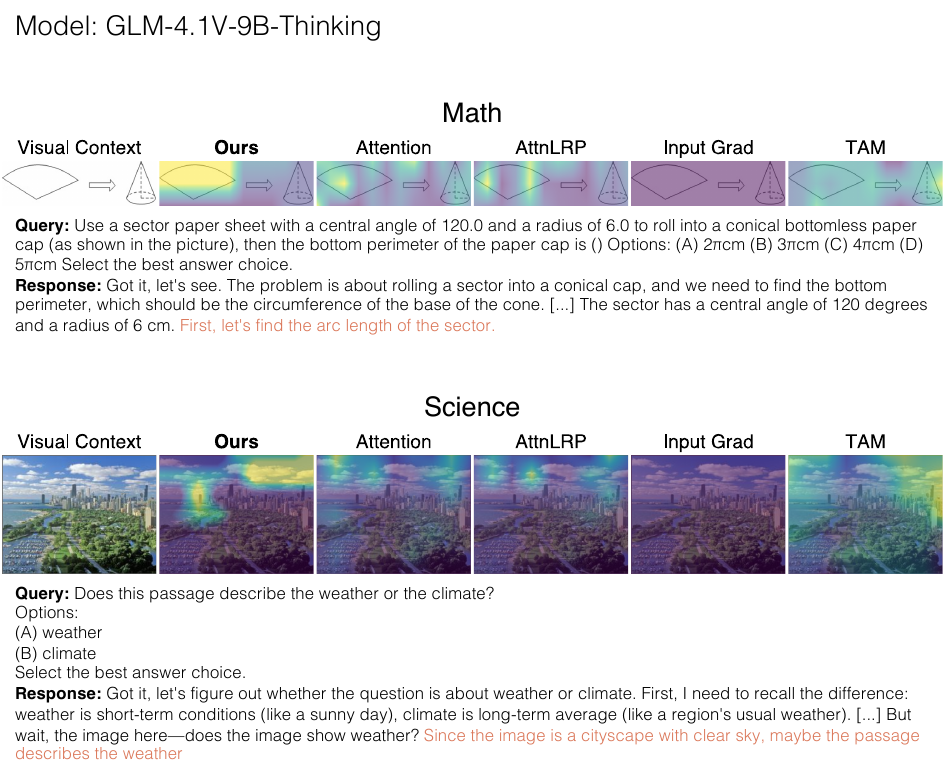}
    \caption{
        \textbf{Additional qualitative results: GLM-4.1V on math/science reasoning.}
    }
    \label{fig:qual_glm_2_ms}
\end{figure*}

\begin{figure*}[t]
    \centering
    \includegraphics[width=\linewidth]{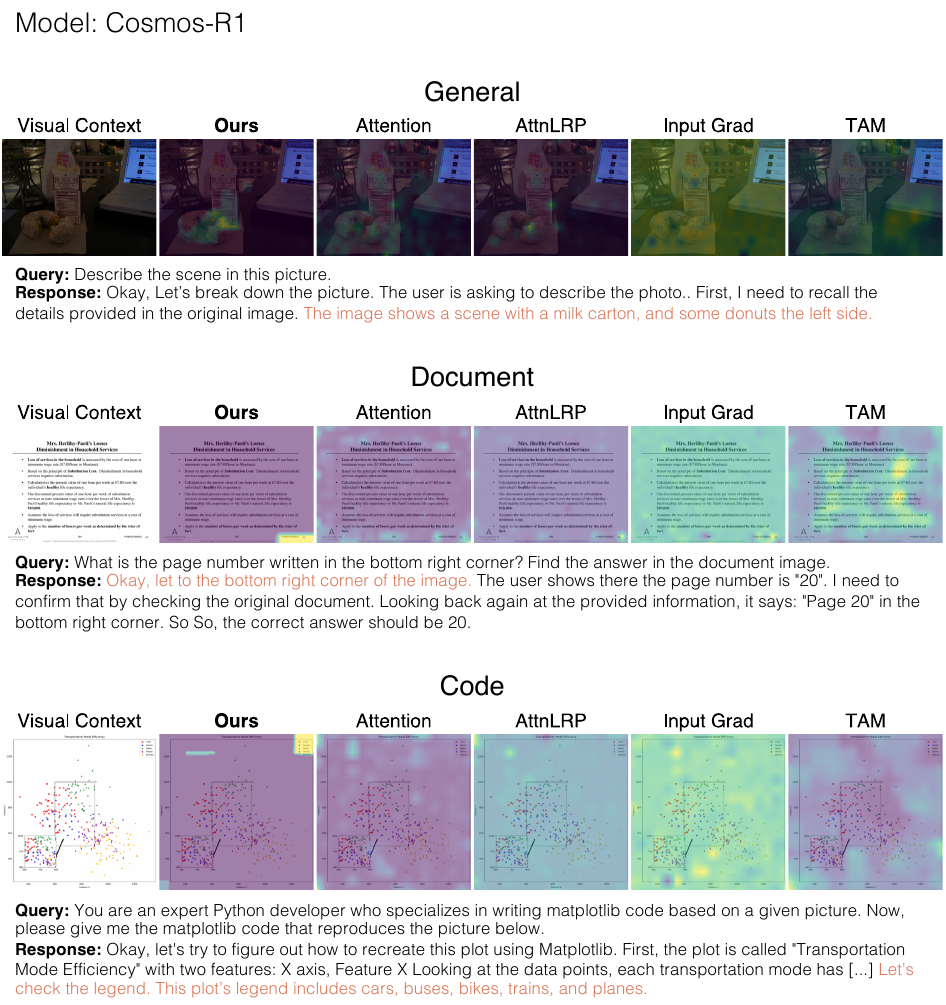}
    \caption{
        \textbf{Additional qualitative results: Cosmos-R1 on general/document/code reasoning.}
    }
    \label{fig:qual_cosmos_1_gdc}
\end{figure*}

\begin{figure*}[t]
    \centering
    \includegraphics[width=\linewidth]{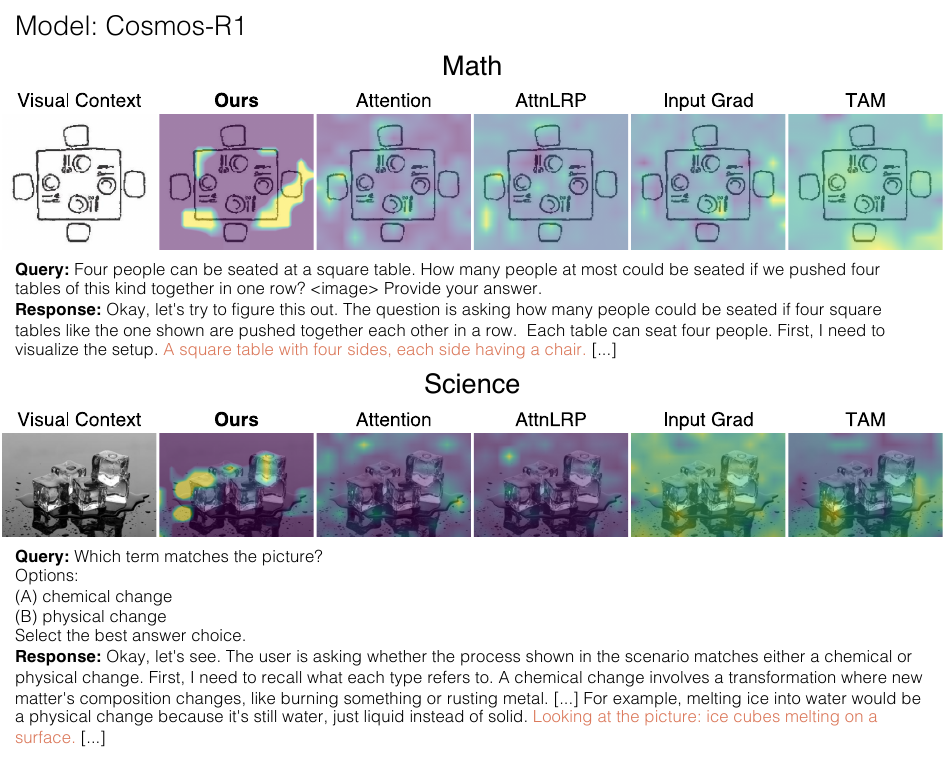}
    \caption{
        \textbf{Additional qualitative results: Cosmos-R1 on math/science reasoning.}
    }
    \label{fig:qual_cosmos_1_ms}
\end{figure*}

\begin{figure*}[t]
    \centering
    \includegraphics[width=\linewidth]{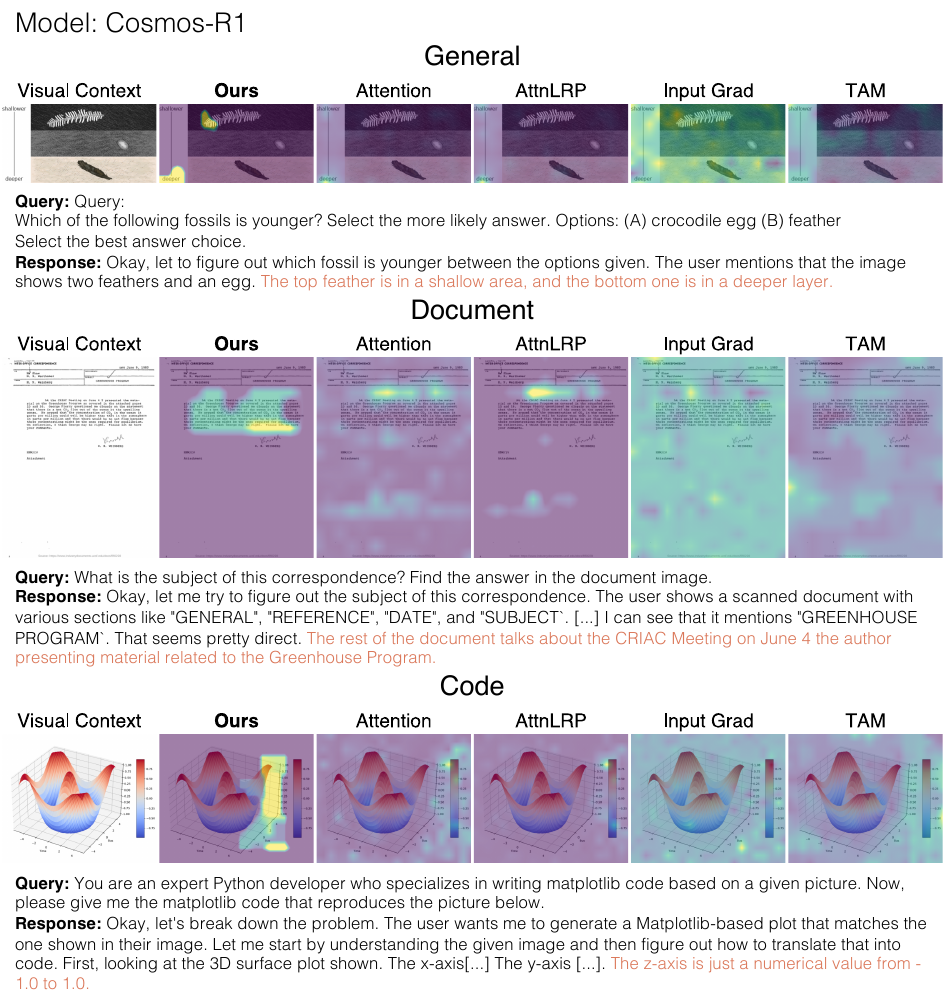}
    \caption{
        \textbf{Additional qualitative results: Cosmos-R1 on general/document/code reasoning.}
    }
    \label{fig:qual_cosmos_2_gdc}
\end{figure*}

\begin{figure*}[t]
    \centering
    \includegraphics[width=\linewidth]{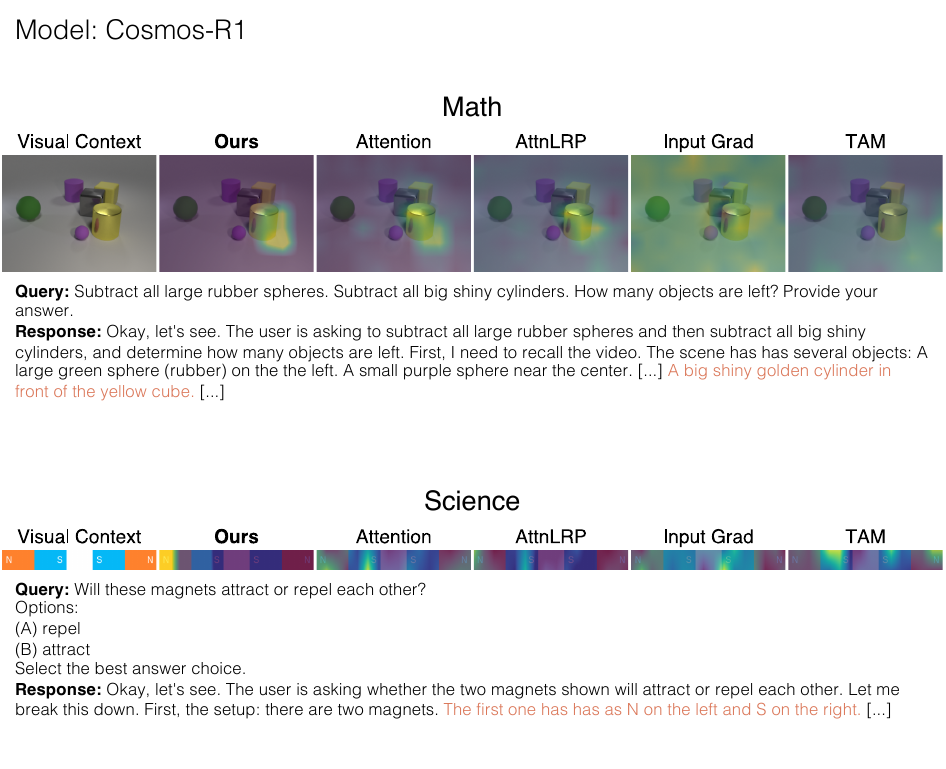}
    \caption{
        \textbf{Additional qualitative results: Cosmos-R1 on math/science reasoning.}
    }
    \label{fig:qual_cosmos_2_ms}
\end{figure*}

\begin{figure*}[t]
    \centering
    \includegraphics[width=\linewidth]{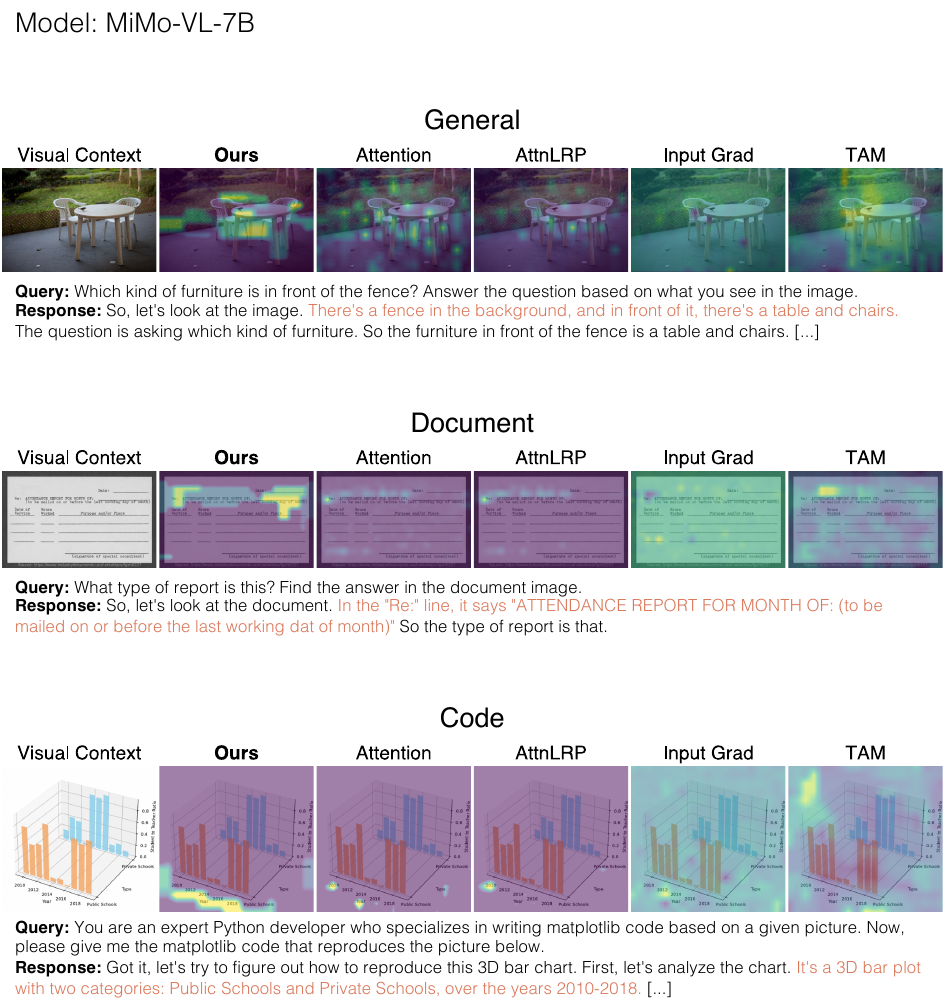}
    \caption{
        \textbf{Additional qualitative results: MiMo-VL on general/document/code reasoning.}
    }
    \label{fig:qual_mimo_1_gdc}
\end{figure*}

\begin{figure*}[t]
    \centering
    \includegraphics[width=\linewidth]{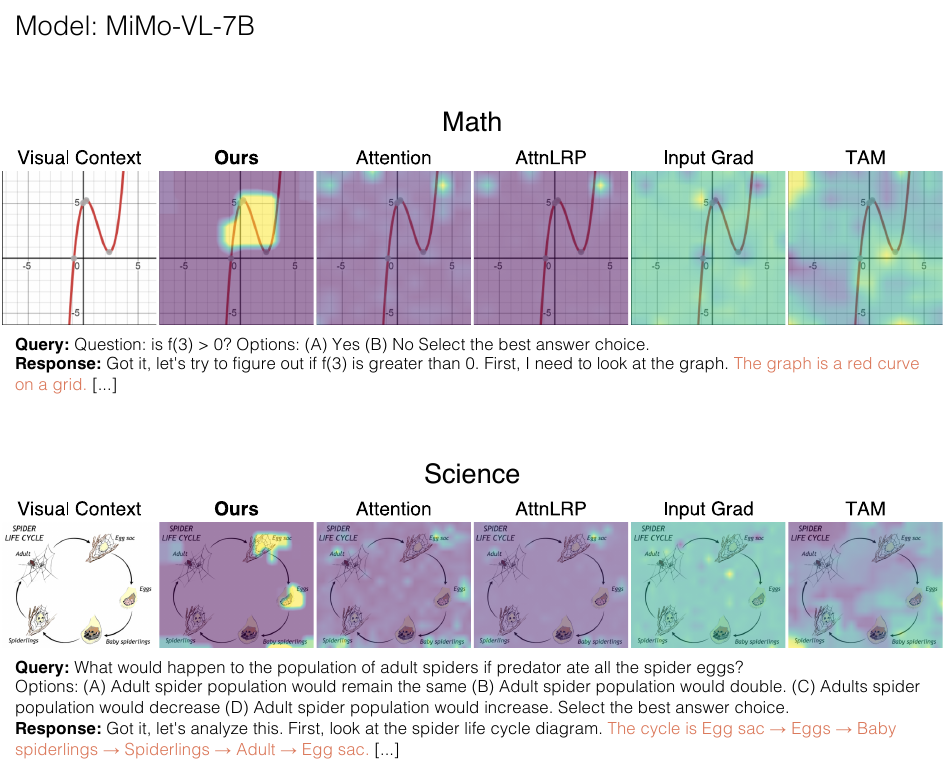}
    \caption{
        \textbf{Additional qualitative results: MiMo-VL on math/science reasoning.}
    }
    \label{fig:qual_mimo_1_ms}
\end{figure*}

\begin{figure*}[t]
    \centering
    \includegraphics[width=\linewidth]{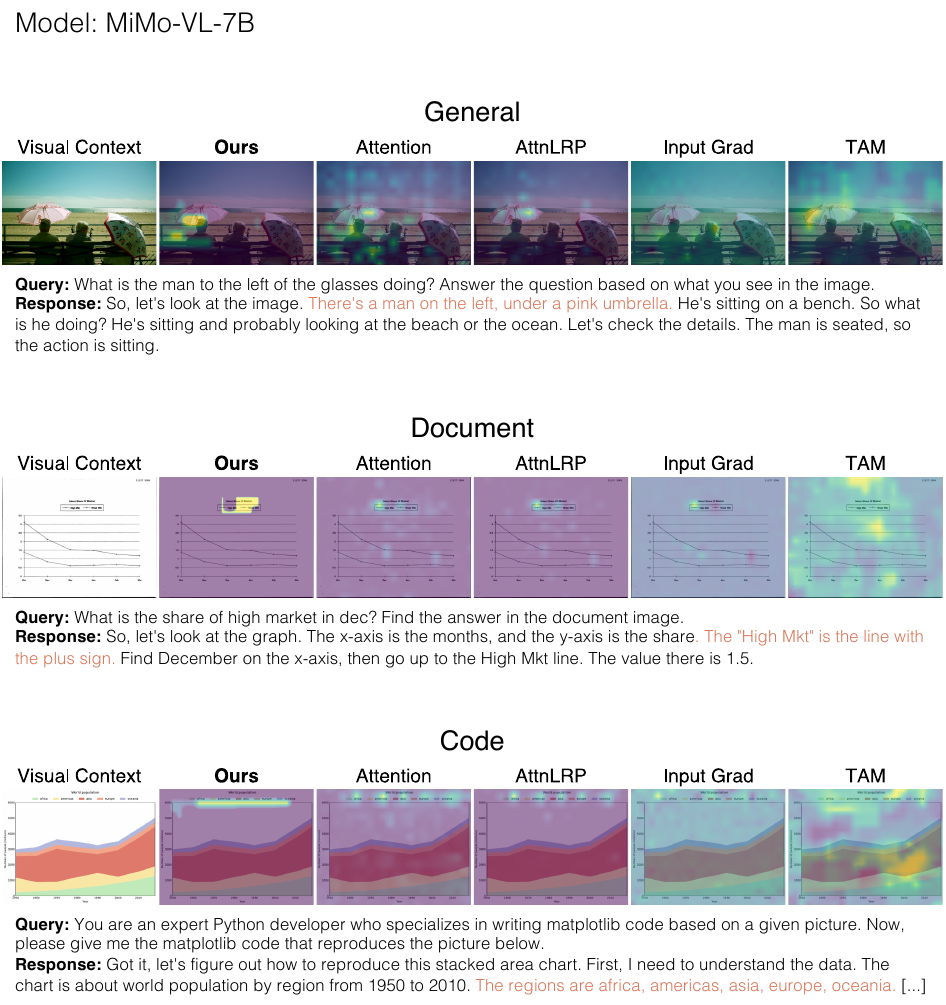}
    \caption{
        \textbf{Additional qualitative results: MiMo-VL on general/document/code reasoning.}
    }
    \label{fig:qual_mimo_2_gdc}
\end{figure*}

\begin{figure*}[t]
    \centering
    \includegraphics[width=\linewidth]{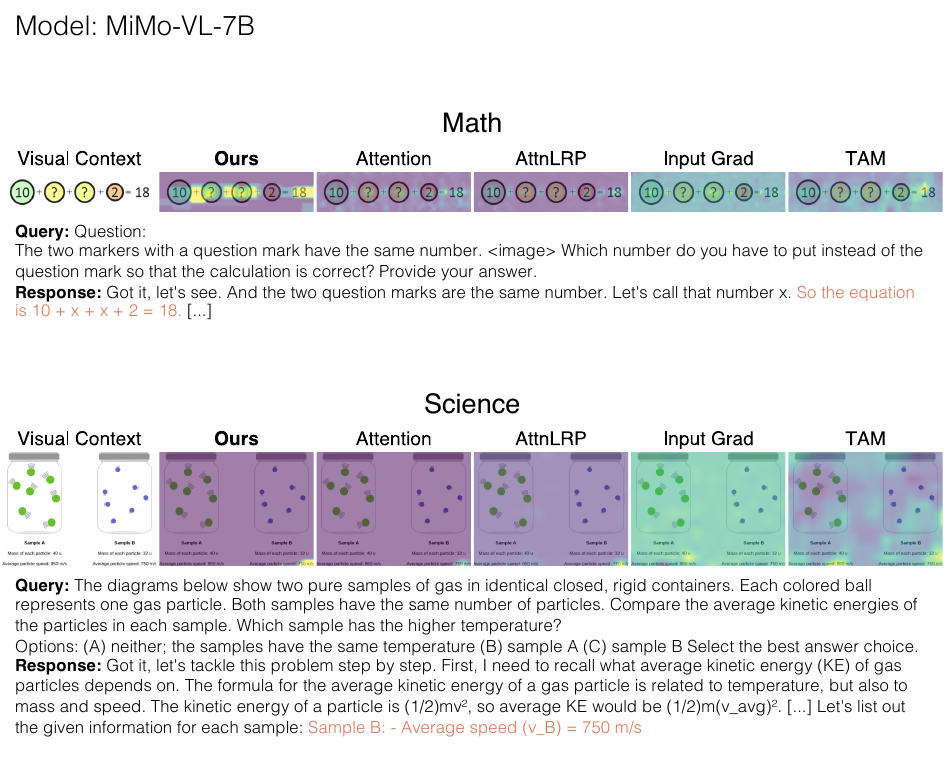}
    \caption{
        \textbf{Additional qualitative results: MiMo-VL on math/science reasoning.}
    }
    \label{fig:qual_mimo_2_ms}
\end{figure*}

\end{document}